\documentclass{article}

\usepackage{cite}

\usepackage[nonatbib, final]{neurips_2022}

\usepackage{url}
\usepackage[hidelinks]{hyperref}
\usepackage{amsmath,amsthm,amssymb,amsbsy,amsfonts,amscd,bm}
\usepackage{paralist}
\usepackage{xcolor}
\usepackage{color}
\usepackage{graphicx}
\graphicspath{{./figs/}}
\usepackage{algorithm}
\usepackage{algorithmic}
\usepackage{comment}
\usepackage{multirow}
\usepackage{enumitem}
\usepackage{fancyhdr}
\usepackage{cleveref}
\usepackage{subcaption}
\usepackage{wrapfig}

\newtheorem{lemma}{Lemma}
\newtheorem{cor}{Corollary}
\newtheorem{definition}{Definition}

\newtheorem{theorem}{Theorem}

\theoremstyle{remark}

\newcommand{\R}{\mathbb{R}}

\newcommand{\e}{\begin{equation}}
\newcommand{\ee}{\end{equation}}
\newcommand{\en}{\begin{equation*}}
\newcommand{\een}{\end{equation*}}
\newcommand{\eqn}{\begin{eqnarray}}
\newcommand{\eeqn}{\end{eqnarray}}
\newcommand{\bmat}{\begin{bmatrix}}
\newcommand{\emat}{\end{bmatrix}}

\DeclareMathAlphabet\mathbfcal{OMS}{cmsy}{b}{n}

\newcommand{\mb}{\bm}
\newcommand{\mc}{\mathcal}

\newcommand{\bb}{\mathbb}

\newcommand{\vct}[1]{\boldsymbol{#1}}
\newcommand{\mtx}[1]{\boldsymbol{#1}}

\newcommand{\trace}{\operatorname{trace}}

\newcommand{\diag}{\operatorname{diag}}

\newcommand{\wh}{\widehat}

\newcommand{\wt}{\widetilde}
\newcommand{\ol}{\overline}

\newcommand{\norm}[2]{\left\| #1 \right\|_{#2}}

\newcommand{\abs}[1]{\left| #1 \right|}

\newcommand{\parans}[1]{\left(#1\right)}

\newcommand{\calL}{\mathcal{L}}

\hypersetup{
    colorlinks=true,%
    citecolor=blue,%
    filecolor=blue,%
    linkcolor=blue,%
    urlcolor=blue
}

\newcommand{\Lce}{\mc L_{\mathrm{CE}}}
\newcommand{\Lls}{\mc L_{\mathrm{LS}}}
\newcommand{\Lfl}{\mc L_{\mathrm{FL}}}
\newcommand{\Lmse}{\mc L_{\mathrm{MSE}}}
\newcommand{\Lgl}{\mc L_{\mathrm{GL}}}

\newcommand{ \Brac }[1]{\left\lbrace #1 \right\rbrace}

\newcommand{ \paren }[1]{ \left( #1 \right) }

\newcommand{\NC}{$\mc {NC}$}

\newcommand{\vb}{\vct{b}}

\newcommand{\ve}{\vct{e}}

\newcommand{\vh}{\vct{h}}

\newcommand{\vp}{\vct{p}}

\newcommand{\vw}{\vct{w}}
\newcommand{\vx}{\vct{x}}
\newcommand{\vy}{\vct{y}}
\newcommand{\vz}{\vct{z}}

\newcommand{\vtheta}{\vct{\theta}}

\newcommand{\vzero}{\vct{0}}
\newcommand{\vone}{\vct{1}}

\newcommand{\mA}{\mtx{A}}
\newcommand{\mB}{\mtx{B}}

\newcommand{\mD}{\mtx{D}}

\newcommand{\mG}{\mtx{G}}
\newcommand{\mH}{\mtx{H}}

\newcommand{\mP}{\mtx{P}}

\newcommand{\mT}{\mtx{T}}
\newcommand{\mU}{\mtx{U}}
\newcommand{\mV}{\mtx{V}}
\newcommand{\mW}{\mtx{W}}

\newcommand{\mZ}{\mtx{Z}}

\newcommand{\mSigma}{\mtx{\Sigma}}

\graphicspath{{./figs/}}

\newlength{\imgwidth}
\newcommand{\note}[1]{\textcolor{blue}{\bf [{\em Note:} #1]}}
\newcommand{\revise}[1]{\textcolor{black}{#1}}

\newboolean{twoColVersion}
\setboolean{twoColVersion}{false}
\newcommand{\twoCol}[2]{\ifthenelse{\boolean{twoColVersion}} {#1} {#2} }

\newcommand{\zz}[1]{\textcolor{yellow}{\bf [{\em Zhihui:} #1]}}

\usepackage[toc, page]{appendix}
\newcommand{\lambdaW}{\lambda_{\mW}}
\newcommand{\lambdaH}{\lambda_{\mH}}
\newcommand{\lambdab}{\lambda_{\vb}}
\usepackage[font=footnotesize]{caption}

\title{Are All Losses Created Equal:\\ A Neural Collapse Perspective}

\author{%
  Jinxin Zhou \\
  Ohio State University\\
  \texttt{zhou.3820@osu.edu}\\
  \And
  Chong You\\
  Google Research \\
  \texttt{cyou@google.com} \\
  \And
  Xiao Li\\
  University of Michigan \\
  \texttt{xlxiao@umich.edu} \\
  \And
  Kangning Liu\\
  New York University \\
  \texttt{kl3141@nyu.edu} \\
  \And
  Sheng Liu\\
  New York University \\
  \texttt{shengliu@nyu.edu} \\
  \And
  Qing Qu\\
  University of Michigan \\
  \texttt{qingqu@umich.edu} \\
  \And
  Zhihui Zhu\thanks{Corresponding author.} \\
  Ohio State University\\
  \texttt{zhu.3440@osu.edu}
}

\begin{document}

\maketitle

\begin{abstract}
While cross entropy (CE) is the most commonly used loss function to train deep neural networks for classification tasks, many alternative losses have been developed to obtain better empirical performance.  Among them, which one is the best to use is still a mystery, because there seem to be multiple factors affecting the answer, such as properties of the dataset, the choice of network architecture, and so on.  This paper studies the choice of loss function by examining the last-layer features of deep networks, drawing inspiration from a recent line work showing that the global optimal solution of CE and mean-square-error (MSE) losses exhibits a \emph{Neural Collapse} (\NC) phenomenon.  That is, for sufficiently large networks trained until convergence, (i) all features of the same class collapse to the corresponding class mean and (ii) the means associated with different classes are in a configuration where their pairwise distances are all equal and maximized. 
We extend such results and show through global solution and landscape analyses that a broad family of loss functions including commonly used label smoothing (LS) and focal loss (FL) exhibits \NC. Hence, all relevant losses (i.e., CE, LS, FL, MSE) produce equivalent features on training data.  In particular, based on the \emph{unconstrained feature model} assumption, we provide either the global landscape analysis for LS loss or the local landscape analysis for FL loss and show that  the (only!) global minimizers are \NC\ solutions, while all other critical points are strict saddles whose Hessian exhibit negative curvature directions either in the global scope for LS loss or in the local scope for FL loss near the optimal solution. 
The experiments further show that \NC\ features obtained from all relevant losses (i.e., CE, LS, FL, MSE) lead to largely identical performance on test data as well, provided that the network is sufficiently large and trained until convergence. \revise{The source code is available at}~\url{https://github.com/jinxinzhou/nc_loss}.
\end{abstract}

\section{Introduction}\label{sec:intro}

Loss function is an indispensable component in the training of deep neural networks (DNNs).
While cross-entropy (CE) loss is one of the most popular choices for classification tasks, studies over the past few years have suggested many improved versions of CE that bring better empirical performance. 
Some notable examples include label smoothing (LS) \cite{szegedy2016rethinking} where one-hot label is replaced by a smoothed label, focal loss (FL) \cite{lin2017focal} which puts more emphasis on hard misclassified samples and reduces the relative loss on the already well-classified samples, and so on. 
Aside from CE and its variants, the mean squared error (MSE) loss which was typically used for regression tasks is recently demonstrated to have a competitive  performance when compared to CE for classification tasks \cite{hui2020evaluation}.

Despite the existence of many loss functions there is however a lack of consensus as to which one is the best to use, and the answer seems to depend on multiple factors such as properties of the dataset, choice of network architecture, and so on \cite{janocha2016loss}. 
In this work, we aim to understand the effect of loss function in classification tasks from the perspective of characterizing the last-layer features and classifier of a DNN trained under different losses. 
Our study is motivated by a sequence of recent work that identify an intriguing \emph{Neural Collapse} (\NC) phenomenon in trained networks, which refers to the following properties of the last-layer features and classifier:


\begin{itemize}[leftmargin=0.3in,topsep=0.2em,itemsep=0.11em]
    \item[(i)] \textbf{ \emph{Variability Collapse}:} all features of the same class collapse to the corresponding class mean. 
    \item[(ii)] \textbf{ \emph{Convergence to Simplex ETF}:} the means associated with different classes are in a Simplex Equiangular Tight Frame (ETF) configuration where their pairwise distances are all equal and maximized.
    \item[(iii)]  \textbf{\emph{Convergence to Self-duality}:} the class means are ideally aligned with the last-layer linear classifiers.
    \item[(iv)] \textbf{ \emph{Simple Decision Rule}:} the last-layer classifier is equivalent to a Nearest Class-Center decision rule.
\end{itemize}
This \NC\ phenomena is first discovered by Papyan et al. \cite{papyan2020prevalence,papyan2020traces} under canonical classification problems trained with the CE loss. Following with the CE loss, Han et al. \cite{han2021neural} recently reported that DNNs trained with MSE loss for classification problems also exhibit similar \NC\ phenomena. These results imply that deep networks are essentially learning maximally separable features between classes, and a max-margin classifier in the last layer upon these learned features. 
The intriguing empirical observation motivated a surge of theoretical investigation
\cite{han2021neural, lu2020neural,weinan2020emergence,mixon2020neural,graf2021dissecting,fang2021layer,ji2021unconstrained,tirer2022extended,ergen2021revealing,zhu2021geometric,rangamani2022neural,poggio2020explicit,poggio2020implicit,rangamani2021dynamics,rangamani2022deep,zhou2022optimization}, mostly under a simplified \emph{unconstrained feature model} \cite{mixon2020neural} or \emph{layer-peeled model} \cite{fang2021layer} that treats the last-layer features of each samples before the final classifier as free optimization variables. Under the simplified unconstrained feature model, it has been proved that the \NC\ solution is the only global optimal solution for the CE and MSE losses which are also proved to have benign global landscape, explaining why the global \NC\ solution can be obtained. 

\paragraph{Contributions.}
While previous work provide thorough analysis for \NC\ under CE and MSE losses, the theoretical analysis beyond CE and MSE losses is still limited, and their work only focus on one specific loss without a general format. 
In this paper, we consider a broad family of loss functions that includes CE and some other popular loss functions such as LS and FL as special cases. 
Under the \emph{unconstrained feature model}, we theoretically demonstrate in Section~\ref{sec:main-results} that the \NC\ solution is the only global optimal solution to the family of loss functions. 
Moreover, we provide a global landscape analysis, showing that the LS loss function is a strict saddle function and FL loss function is a local strict saddle function \cite{ge2015escaping,sun2015nonconvex,zhang2020symmetry}. 
A (local) strict saddle function is a function for which every critical point is either a global solution or a strict saddle point with negative curvature (locally).
Hence, our result suggests that any optimizer can escape strict saddle points and converge to the global solution responding to \NC\ for LS and FL.
As far as we know, this paper is the first work that conducts global optimal solution and benign optimization landscape analysis beyond the scope of CE and MSE losses. 

Our theoretical results explained above have important implications for understanding the role of loss function in training DNNs for classification tasks.
Because all losses lead to \NC\;solutions, their corresponding features are equivalent up to a rotation of the feature space.
In other words, our analysis provides a theoretical justification for the following claim:
\begin{itemize}[leftmargin=*] 
    \item[] \emph{All losses (i.e., CE, LS, FL, MSE) lead to largely identical features on \textbf{training} data by large DNNs and sufficiently many iterations. } 
\end{itemize}
We also provide an experimental verification of this claim through experiments in Section~\ref{exp:NC-across-losses}.

While \NC\ reveals that all losses are equivalent at training time, it does not have a direct implication for the features associated with test data as well as the generalization performance \cite{zhang2021understanding}. 
In particular, a recent work \cite{hui2022limitations} shows empirically that \NC\ does not occur for the features associated with test data. 
\revise{Nonetheless, we show through empirical evidence that for large DNNs, \NC\ on training data well predicts the test performance.}
In particular, our empirical study in Section~\ref{exp:same-performance-across-losses} shows the following:
\begin{itemize}[leftmargin=*]
    \item[] \emph{ All losses (CE, LS, FL, MSE) lead to largely identical performance on \textbf{test} data by large DNNs.}
\end{itemize}

Our conclusion that all losses are created equal appears to go against existing evidence on the advantages of some losses over the others. 
Here we emphasize that our conclusion has an important premise, namely the neural network has sufficient approximation power and the training is performed for sufficiently many iterations. 
Hence, our conclusion implies that the better performance with particular choices of loss functions comes as a result that the training does not produce a globally optimal (i.e., \NC) solution. 
In such cases different losses lead to different solutions on the training data, and correspondingly different performance on test data. 
Such an understanding may provide important practical guidance on what loss to choose in different cases (e.g., different model sizes and different training time budgets), as well as for the design of new and better losses in the future. \revise{We note that our conclusion is based on natural accuracy, rather than
model transferability or robustness, which is worth additional efforts to exploit and is left as future work.}

\section{The Problem Setup}\label{sec:problem}
A typical deep neural network $\Psi(\cdot): \R^{D} \mapsto \R^K$ consists of a multi-layer nonlinear compositional feature mapping $\Phi(\cdot): \R^{D} \mapsto \R^d$ and a linear classifier $(\mb W, \mb b)$, which can be generally expressed as 
\begin{align}\label{eq:func-NN}
    \Psi_{\mb \Theta}(\mb x) \;=\;  \mb W \Phi_{\vtheta}(\vx)  + \mb b,
\end{align}
where we use $\vtheta$ to represent the network parameters in the feature mapping and $\mb W\in\R^{K\times d}$ and $\vb\in\R^{K}$ to represent the linear classifier's weight and bias, respectively. Therefore, \emph{all} the network parameters are the set of $\mb \Theta = \Brac{\vtheta, \mb W,\mb b }$. For the input $\vx$, the output of the feature mapping $\Phi_{\vtheta}(\vx)$ is usually termed as the \emph{representation} or \emph{feature} learned from the network.

With an appropriate loss function, the parameters ${\mb \Theta}$ of the whole network are optimized to learn the underlying relation from the input sample $\vx$ to their corresponding target $\vy$ so that the output of the network $\Psi_{\mb \Theta}(\mb x)$ approximates the corresponding target, i.e. $\Psi_{\mb \Theta}(\mb x)\approx \mb y$ in term of the expectation over a distribution $\mc D$ of input-output data pairs $(\vx, \vy)$. While it is hard to get access to the ground-truth distribution $\mc D$ in most cases, one can approximate the distribution $\mc D$ through sampling enough data pairs i.i.d. from $\mc D$. In this paper, we study the multi-class balanced classification tasks with $K$ class and $n$ samples per class, where we use the one-hot vector $\vy_k \in \mathbb{R}^K$ with unity only in $k$-th entry $(1\leq k\leq K)$ to denote the label of the $i$-th sample $\vx_{k,i} \in \R^D$ in the $k$-th class. We then learn the parameters $\mb \Theta$ via minimizing the following empirical risk over the total $N = nK$ training samples
\begin{align}\label{eqn:dl-loss}
    \min_{\mb \Theta} \; \frac{1}{N}\sum_{k=1}^K \sum_{i=1}^{n} {\mc L} \paren{ \Psi_{\mb \Theta}(\mb x_{k,i}),\vy_k } \;+\; \frac{\lambda}{2} \norm{\mb \Theta}{F}^2,
\end{align}
where $\lambda>0$ is the regularization parameter (a.k.a., the weight decay parameter\footnote{\revise{Without weight decay, the features and classifiers will tend to blow up for CE and many other losses.} 
}) and ${\mc L} \paren{ \Psi_{\mb \Theta}(\mb x_{k,i}),\vy_k }$ is a predefined loss function that appropriately measures the difference between the output $\Psi_{\mb \Theta}(\mb x_{k,i})$ and the target $\vy_k$. Some common loss functions used for training deep neural networks will be specified in the next section.

\subsection{Commonly Used Training Losses}
In this subsection, we first present four common loss functions for classification task. To simplify the notation, let $\vz = \mb W \Phi_{\vtheta}(\vx)  + \mb b$ denote the network's output (``logit") vector for the input $\vx$. Assume $\vz$ belongs to the $k$-th class. Also let $\vy^{\text{smooth}}_{k} = (1-\alpha) \vy_{k}+\frac{\alpha}{K}\vone_K$ denote the smoothed targets of $k$-th class, where $0\leq\alpha<1$ and $\vone_K\in\R^K$ is a vector with all entries equal to one. We will use $z_{\ell}$, $y_{k,\ell}$ and $y^{\text{smooth}}_{\ell}$ to denote the $\ell$-th entry of $\vz$, $\vy_{k}$ and $\vy^{\text{smooth}}_{k}$, respectively, where $y^{\text{smooth}}_{k,k}=1 - \frac{K-1}{K}\alpha$ and $y^{\text{smooth}}_{k,\ell}=\frac{\alpha}{K}$ for $k\neq\ell$.

\textbf{Cross entropy (CE)} is perhaps the most common loss for multi-class classification in deep learning. It measures the distance between the target distribution $\vy_k$ and the network output distribution obtained by applying the softmax function on $\vz$, resulting in the following expression
\begin{align}
    \Lce (\vz,\vy_k) \;&=\;- \log\paren{\frac{\exp(z_{k})}{\sum_{j=1}^K\exp(z_{j})}}.\label{def:ce}
\end{align}

\textbf{Focal loss (FL) \cite{lin2017focal}} is first proposed to deal with the extreme foreground-background class imbalance in dense object detection, which adaptively focuses less on the well-classified samples. Recent work \cite{mukhoti2020calibrating, smith2022cyclical} reports that focal loss also improves calibration and automatically forms curriculum learning in multi-class classification setting.
Letting $\gamma\geq 0$ denote the tunable focusing parameter, the focal loss can be expressed as:
\begin{align}
    \Lfl (\vz,\vy_k) \;&=\; -\paren{1-\frac{\exp(z_{k})}{\sum_{j=1}^K\exp(z_{j})}}^\gamma\log\paren{\frac{\exp(z_{k})}{\sum_{j=1}^K\exp(z_{j})}}.\label{def:fl}
\end{align}

\textbf{Label smoothing (LS) \cite{szegedy2016rethinking}} replaces the hard targets in CE with smoothed targets $\vy^{\text{smooth}}_{k}$ obtained from mixing the original targets $\vy_{k}$ with a uniform distribution over all entries $\vone_K$. Experiments in\cite{muller2019does,chen2020investigation} find that classification models trained with label smoothing have better calibration and generalization.
Denoting by $0 \le \alpha \le  1$ the tunable smoothing parameter, the label smoothing loss function can be formulated as:
\begin{align}
    \Lls (\vz,\vy_k) \;&=\; -\sum_{\ell=1}^K y_{k,\ell}^{\text{smooth}}\log\paren{\frac{\exp(z_{\ell})}{\sum_{j=1}^K\exp(z_{j})}}.\label{def:ls}
\end{align}
When $\alpha = 0$, the above label smoothing loss reduces to the CE loss.

\textbf{Mean square error (MSE)} is often used for regression but not classification task. The recent work \cite{hui2020evaluation} shows that classification networks trained with MSE loss achieve on par performance compared to those trained with the CE loss. Throughout our paper, we use the rescaled MSE version  \cite{hui2020evaluation}:
\begin{align}
    \Lmse (\vz,\vy_k) \;&=\; \kappa(z_{k}-\beta)^2 + \sum_{\ell\neq k}^Kz_{\ell}^2,\label{def:mse}
\end{align}
where $\kappa>0$ and $\beta>0$ are hyperparameters.

\subsection{Problem Formulation Based on Unconstrained Feature Models}
Because of the interaction between a large number of nonlinear layers in the feature mapping $\Phi_{\vtheta}$, it is tremendously challenging to analyze the optimization of deep neural networks. To simplify the difficulty of deep neural network analysis, a series of recent works of theoretically studying \NC\ phenomenon use a so-called {\it unconstrained feature model} (or {\it layer-peeled model} in \cite{fang2021layer}) which treats the last-layer features as \emph{free} optimization variables $\vh =\Phi(\vx)\in\R^{d}$. The reason behind the {\it unconstrained feature model} is that modern highly overparameterized deep networks are able to approximate any continuous functions \cite{cybenko1989approximation, hornik1991approximation,lu2017expressive,shaham2018provable} and the characterization of \NC\ are only related with the last layer features. We adopt the same approach and study the effects of different training losses on the last-layer representations of the network under the unconstrained feature model. For convenient, let us denote 
\begin{align*}
    \mb W &\;:=\; \begin{bmatrix}
    \mb w^{1 } & \ \mb w^{2 } & \cdots & \mb w^{K }  
    \end{bmatrix}^\top \in \bb R^{K \times d}, \\
    \mb H &\;:=\; \begin{bmatrix}
    \mb H_1 & \mb H_2 & \cdots & \mb H_n
    \end{bmatrix}\in \bb R^{d \times N}  , ~~\text{and}\\
    \mb Y &\::=\; \begin{bmatrix}
    \mb Y_1 & \mb Y_2 &\cdots & \mb Y_K
    \end{bmatrix} \in \bb R^{K \times N},
\end{align*}
where $\mb w^{k}$ is the $k$-th row vector of $\mb W$, all the features in the $k$-th class are denoted as $\mb H_i:= \begin{bmatrix}\mb h_{1,i} & \cdots & \mb h_{K,i} \end{bmatrix} \in \bb R^{ d \times K}$ and $\mb h_{k,i}$ is the feature of the $i$-th sample in the $k$-th class, and $\mb Y_k := \begin{bmatrix}  \mb y_k  &\cdots & \mb y_k \end{bmatrix} \in \bb R^{K \times n}$ for all $k=1,2,\cdots,K$ and $i=1,2,\cdots,n$. Based on the unconstrained feature model, we consider a slight variant of \eqref{eqn:dl-loss}, given by
\begin{align}\label{eq:obj}
     \min_{\mb W , \mb H,\mb b  } & f(\mb W,\mb H,\mb b):= \frac{1}{N}\sum_{k=1}^K \sum_{i=1}^{n} {\mc L} \paren{\mW\vh_{k,i}+\vb,\vy_k } + \frac{\lambda_{\mb W} }{2} \norm{\mb W}{F}^2 
     + \frac{\lambda_{\mb H} }{2} \norm{\mb H}{F}^2 + \frac{\lambda_{\mb b} }{2} \norm{\mb b}{2}^2,
\end{align}
where $\lambda_{\mb W}$, $\lambda_{\mb H},\;\lambda_{\mb b}>0$ are the penalty parameters for $\mb W$, $\mb H$, and $\mb b$, respectively.

By viewing the last-layer feature $\mb H$ as a free optimization variable, the simplified objective function \eqref{eq:obj} consider the weight decay about $\mb W$ and $\mb H$, which is slightly different from practice that the weight decay is imposed on all the network parameters $\mb \Theta $ as shown in \eqref{eqn:dl-loss}. Nonetheless, the underlying rationale is that the weight decay on $\mb \Theta$ \emph{implicitly} penalizes the energy of the features (i.e., $\|\mb H\|_F$)~ \cite{zhu2021geometric}.  

As \NC\ phenomena for the learned features and classifiers is first discovered for neural networks trained with the CE loss \cite{papyan2020prevalence},  the CE loss has been mostly studied through the above simplified unconstrained feature model  \cite{ lu2020neural,weinan2020emergence,graf2021dissecting,fang2021layer,zhu2021geometric} to understand the \NC\ phenomena. The work \cite{mixon2020neural,han2021neural,tirer2022extended, zhou2022optimization} also studied the MSE loss, but the analysis there shows the solutions of the learned features and classifiers depend crucially on the bias term, while for CE loss with or without the bias term have no effect on the learned features and classifiers under the unconstrained feature model. The other losses such as focal loss and label smoothing have been less studied, though they are widely employed in practice to obtain better performance. This will be the subject of next section.

\section{Understanding Loss Functions Through Unconstrained Features Model}\label{sec:main-results}
In this section, we study the  effect of different loss functions through the unconstrained features model. We will first present a contrastive property for general loss function $\Lgl$ in \Cref{def:GLoss}. We will then study the global optimality conditions in terms of the learned features and classifiers  as well as geometric properties for \eqref{eq:obj} with such a general loss function $\Lgl$. 

\subsection{A Contrastive Property for the Loss Functions}
In this paper, we aim to provide a unified analysis for different loss functions. Towards that goal, we first present some common properties behind the CE, FL and FL to motivate the discussion. Taking CE as an example, we can lower bound it by
\begin{align}
    \Lce (\vz,\vy_k)
     \ge \log \paren{1 + (K-1)\exp\paren{\frac{\sum_{j\neq k} (z_j - z_k)}{K-1}}  }= \phi_{\textup{CE}}\paren{\sum_{j\neq k} (z_j - z_k)}
\label{eq:CE-contrastive}\end{align}
where $\phi_{\text{CE}}(t)=\log \paren{1 + (K-1)\exp\paren{\frac{ t}{K-1}}}$, and the inequality achieves equality when $z_j = z_{j'}$ for all $j,j'\neq k$. This requirement is reasonable because the commonly used losses treat all the outputs except for the $k$-th output $\vz$ identically. Since $\phi_{\text{CE}}$ is an increasing function, minimizing the CE loss $\Lce (\vz,\vy_k)$ is equivalent to maximizing $(K-1)z_k - \sum_{j\neq K} z_j$, which contrasts the $k$-th output $z_k$ simultaneously to all the other outputs $z_j$ for all $j\neq k$. Thus, we call \eqref{eq:CE-contrastive} as a \emph{contrastive property}. Maximizing $(K-1)z_k - \sum_{j\neq K} z_j$ would lead to a positive (and relatively large) $z_k$ and negative (and relatively small) $z_j$. In particular, within the unit sphere $\norm{\vz}{_2} = 1$,  $(K-1)z_k - \sum_{j\neq K} z_j$ achieves its maximizer when $z_k = \sqrt{\frac{K-1}{K}}$ and $z_j = -\sqrt{\frac{1}{K(K-1)}}$ for all $j\neq k$, which satisfies the requirement $z_j = z_{j'} $ for all $j,j' \neq k$. Thus, $z_k = \sqrt{\frac{K-1}{K}}$ and $z_j = -\sqrt{\frac{1}{K(K-1)}}$ is also the global minimizer for $\phi_{\text{CE}}$ within the unit sphere $\norm{\vz}{_2} = 1$. \revise{As the global minimizer is unique for each class, it encourages intra-class compactness. On the other hand, the minimizers to different classes are maximally distant, promoting inter-class separability. 
}

Motivated by the above discussion, we now introduce the following properties for a general loss function $\Lgl(\vz,\vy_k)$.

\begin{definition}[Contrastive property] We say a loss function $\Lgl (\vz,\vy_k)$ satisfies the contrastive property if 
there exists a function $\phi$ such that $\Lgl (\vz,\vy_k)$ can be lower bounded by
\begin{align}\label{gl-property-1}
    &\Lgl (\vz,\vy_k) \geq \phi\paren{\sum_{j\neq k}(z_j-z_k)}
    \end{align}
where the equality holds only when $z_j=z_j'$ for all $j,j' \neq k$. Moreover,  $\phi(t)$ satisfies
\begin{align}
    t^*&= \arg\min_{t}\phi\paren{t}+c|t| \text{ is unique for any } c>0, \text{and } t^*\leq 0.\label{gl-property-3}
\end{align}
\label{def:GLoss}
\end{definition}
In the appendix, we show that CE, FL and LS all satisfy this property.  The motivation for \eqref{gl-property-1} follows from the above discussion. In particular, \eqref{gl-property-1} achieves equality when all the outputs except for the $k$-th one are identical, which holds for common loss functions since those outputs are treated identically. In \eqref{gl-property-3}, $c$ is a constant related with the weight decay penalty parameters. By \eqref{gl-property-1}, we can find the global minimizer for $\Lgl (\vz,\vy_k)$ by minimizing the right hand side since the equality in \eqref{gl-property-1} is achievable. Thus, the requirement of a unique minimizer \eqref{gl-property-3} ensures a unique minimizer for the regularized $\Lgl (\vz,\vy_k)$. This condition can be easily satisfied. For example,  $\phi_{\text{CE}}(t)$ defined in \eqref{eq:CE-contrastive} for the CE loss is an increasing and strictly convex function and thus has unique minimizer for $\phi_{\text{CE}}(t)+c|t|$.
Along the same line, we require a negative minimizer $t^\star$ to ensure that the minimizer for the regularized $\Lgl (\vz,\vy_k)$ has $k$-th entry being its largest entry, which is required to ensure correct prediction since the largest entry predicts the class membership. Therefore, such a condition is generally satisfied by the common losses. For example,  $\phi_{\text{CE}}(t)$ is an increasing function and thus must have a non-positive minimizer for $\phi_{\text{CE}}(t)+c|t|$. Finally, we note that the MSE loss is not included since it has different form than others and thus the analysis will be different. But as mentioned above, the MSE loss has been studied in \cite{mixon2020neural,han2021neural,tirer2022extended, zhou2022optimization}. 



\subsection{Landscape Analysis for the Unconstrained Features Model}
We now study the global optimality conditions in terms of the learned features and classifiers  as well as geometric properties for the training problem \eqref{eq:obj} with the general loss function $\Lgl$ satisfying the above contrastive property. 

\begin{theorem}[Global Optimality Condition]\label{thm:global-minima}
	Assume that the number of classes $K$ is smaller than the feature dimension $d$, i.e., $K< d$, and the dataset is balanced for each class, $n=n_1=\cdots=n_K$. Then any global minimizer $(\mW^\star, \mH^\star,\vb^\star)$ of $f$ 
 	in \eqref{eq:obj} with a loss function $\Lgl$ satisfying the contrastive property in \Cref{def:GLoss} has following properties:
\begin{align*}
    & \norm{\mb w^\star}{2}\;=\; \norm{\mb w^{\star 1} }{2} \;=\; \norm{\mb w^{\star 2}}{2} \;=\; \cdots \;=\; \norm{\mb w^{\star K} }{2}, \quad \text{and}\quad \mb b^\star = b^\star \mb 1, \\ 
    & \vh_{k,i}^\star \;=\;  \sqrt{ \frac{ \lambda_{\mb W}  }{ \lambda_{\mb H} n } } \vw^{\star k} ,\quad \forall \; k\in[K],\; i\in[n],  \quad \text{and} \quad  \ol{\mb h}_{i}^\star \;:=\; \frac{1}{K} \sum_{j=1}^K \mb h_{j,i}^\star \;=\; \mb 0, \quad \forall \; i \in [n],
\end{align*}
where either $b^\star = 0$ or $\lambdab=0$, and the matrix $\mW^{\star\top} $ is in the form of $K$-simplex ETF structure (see appendix for the formal definition) in the sense that 
\begin{align*}
      \mW^{\star\top} \mW^{\star}\;=\; \norm{\mb w^\star}{2}^2\frac{K}{K-1}  \paren{ \mb I_K - \frac{1}{K} \mb 1_K \mb 1_K^\top }.
\end{align*}
\end{theorem}
Its proof is given in \Cref{app:thm-global-fl}. At a high level, we lower bound the general loss function based on the contrastive property \eqref{gl-property-1} and then check the equality conditions hold for the lower bounds. While similar strategy has been used for CE loss \cite{fang2021layer,ji2021unconstrained,zhu2021geometric},
our proof  is different from previous work in terms of dealing with the nuclear norm and checking the structures of the minimizer per sample and class, enabling the global optimality analysis for general loss functions. \Cref{thm:global-minima} implies that for all the loss functions  (e.g., CE, LS, and FL) satisfying the contrastive property, they share similar global solutions with \NC\ property in the learned features and classifiers. 

While \Cref{thm:global-minima} shows that \NC\ features and classifiers are the only global minimizers to \eqref{eq:obj}, it is not obvious whether local search algorithms (such as gradient descent) can efficiently find these benign global solutions. The reason is that the training problem \eqref{eq:obj} is nonconvex due to the bilinear form between $\mW$ and $\mH$. To address this challenge, we use the recent advances on the geometric analysis for nonconvex optimization \cite{sun2015nonconvex,ge2015escaping,lee2016gradient,zhang2020symmetry} to guarantee that the global solutions of \eqref{eq:obj} can be efficiently achieved by iterative algorithms. Towards that goal, we first present the following general results concerning the global landscape for \eqref{eq:obj}.
\begin{theorem}[Benign Landscape]\label{thm:global-geometry}
Assume that the feature dimension $d$ is larger than the number of classes $K$, i.e., $d>K$. Also assume $\calL(\vz,\vy)$ is a convex function in terms of $\vz$. Then the objective function $f$ in \eqref{eq:obj} is a strict saddle function with no spurious local minimum. That is, any of its critical point is either a global minimizer, or it is a strict saddle point whose Hessian has a strictly negative eigenvalue.
\end{theorem}
This result is similar to \cite[Theorem 3.2]{zhu2021geometric} which studies the particular CE loss. Though the result in \cite{zhu2021geometric} is about the CE loss, \revise{we checked its proof and it only uses convexity and smoothness and thus the result can be applied more generally for any smooth convex loss function $\calL(\cdot,\vy)$}. So we omit the proof of \Cref{thm:global-geometry}. We note that the geometric analysis is also closed related to nonconvex low-rank matrix problems \cite{haeffele2015global,ge2016matrix,bhojanapalli2016global,ge2017no,li2019non,li2019symmetry,chi2019nonconvex} with the Burer-Moneirto factorization approach \cite{burer2003nonlinear} if one views $\mW$ and $\mH$ as two factors of a matrix $\mb Z = \mb W \mb H $. We refer to \cite{zhu2021geometric} for more discussions about the connections and differences.

We now exploit \Cref{thm:global-geometry} for the label smoothing and focal loss. In the supplementary material, we show that LS is a convex function. Thus, the following result establishes global optimization landscape for the training problem \eqref{eq:obj} with such a loss. 
\begin{cor}[Benign Landscape with LS]\label{cor:global-geometry-LS}
Assume that the feature dimension $d$ is larger than the number of classes $K$, i.e., $d>K$.  Then the objective function $f$ in \eqref{eq:obj} with LS loss $ \Lls $ is a strict saddle function with no spurious local minimum. 
\end{cor}

Unlike LS, focal loss $\Lfl (\vz,\vy_k)$ is convex only in a local region rather than the entire space. For example, we can show that $\Lfl (\vz,\vy_k)$ is convex within the region $\Omega = \{\vz\in\R^K: {\exp(z_k)}/{\sum_{j=1}^K \exp(z_j)} \ge 0.21\}$. The set $\Omega$ contains a relative large region including the global minimizer which has the value $\frac{\exp(z_k)}{\sum_{j=1}^K \exp(z_j)}$ approaching 1. Thus, we obtain a benign local landscape for the training problem \eqref{eq:obj} with FL.

\begin{cor}[Benign Landscape with FL]\label{cor:global-geometry-FL} 
Assume that the feature dimension $d$ is larger than the number of classes $K$, i.e., $d>K$.  Then the objective function $f$ in \eqref{eq:obj} with FL loss $ \Lfl $ has a benign local landscape: $f$ is a strict saddle function with no spurious local minimum within the region $\{(\mW,\mH,\vb): \mW\vh_{k,i}+\vb\in\Omega, \ 1 \le k\le K, 1\le i \le n \}$. 
\end{cor}

While \Cref{cor:global-geometry-FL} only provides a local benign landscape for the FL, we observe from experiments that gradient descent with random initialization always converges to a global solution with \NC\ properties for \eqref{eq:obj}. So we expect the training problem in \eqref{eq:obj} with FL loss has benign landscape in a much larger region. One direction is to show $\Lfl (\cdot,\vy_k)$ is locally convex in a much larger region $\Omega$, but we leave the thorough investigation to future work. Noting that the CE and MSE losses are also convex, these results imply that (stochastic) gradient descent with random initialization \cite{ge2015escaping,lee2016gradient} almost surely finds the global solutions of the training problem in \eqref{eq:obj} with different training losses. This together with \Cref{thm:global-minima} implies that for different losses, gradient descent will always learn similar features and classifiers---those that exhibit the \NC\ phenomenon.

\section{Experiments}\label{sec:experiment}
We conduct experiments with practical network architectures on standard image classification datasets to study the effect of different loss functions.
First, \Cref{exp:NC-across-losses} provides results to show that the \NC\ phenomena are not restricted to networks trained via the CE and MSE losses.
Rather, there is a family of loss functions, and for the purpose of illustration we pick FL and LS as two prominent special cases, that exhibit the same \NC\ phenomena. Such results verify our theoretical results in \Cref{sec:main-results}. 
To demonstrate the implication of \NC\ for test performance, we present experimental results in \Cref{exp:same-performance-across-losses} with a varying number of training iterations and a varying width of networks, showing that \emph{all} losses with \NC\ global optimality have similar performance on the test dataset when the network is sufficiently large and trained long enough. 

Before presenting the experiment results, we first introduce our experimental setup, including datasets, network architectures, training procedure, and metrics for measuring \NC\ .

\paragraph{Setup of Loss Function, Network Architecture, Dataset, and Training}
We focus on the CE, FL, LS and MSE loss functions for which we use $\gamma=3$ for FL, $\alpha=0.1$ for LS, and $\kappa=1$ and $\beta=15$ for MSE, \revise{except otherwise specificed}. 
We train a WideResNet50 network \cite{he2016deep} on CIFAR10 and CIFAR100 datasets \cite{krizhevsky2009learning} and \revise{a WideResNet18 network on miniImageNet \cite{vinyals2016matching}} with various widths and number of iterations for image classification using these four different losses.\footnote{\revise{Similar results are expected on other architectures and dataset as \NC\ is observed across a range of architectures and dataset in \cite{papyan2020prevalence}.}} 
To examine the effect of model size, we experiment with four versions of WideResNet, denoted as WideResNet-$X$, where $X \in \{0.25, 0.5, 1, 2\}$ is a multiplier on the width of \revise{its corresponding} standard WideResNet. 
Due to the page limit, we put all results on CIFAR100 and \revise{miniImageNet} in the Appendix. We use standard preprocessing such that images are normalized (channel-wise) by their mean and standard deviation, as well as standard data augmentation. For optimization, we use SGD with momentum $0.9$ and an initial learning rate $0.1$ decayed by a factor of 0.1 at $\frac{3}{7}$ and $\frac{5}{7}$ of the total number of iterations. 
Following \cite{mukhoti2020calibrating}, the norm of gradient is clipped at $2$ which can improve performance for all losses. For CIFAR10 and \revise{miniImageNet}, the weight decay is set to $5\times 10^4$ for all configurations with all losses. For CIFAR100, the weight decay is fine-tuned to achieve best accuracy for every configuration and loss.

\paragraph{Three \NC\ Metrics $\mc {NC}_1$-$\mc {NC}_3$ during Network Training }\label{subsec:nc_metrics_definition} We use the same three metrics $\mc {NC}_1$-$\mc {NC}_3$ for the last-layer features and classifier as in \cite{papyan2020prevalence, zhu2021geometric, zhou2022optimization} to measure the first three \NC\ properties in \Cref{sec:intro}. Before we describe these three metrics, let us denote the global mean ${\mb h}_{G}$ and $k$-th class mean $\ol{\mb h}_k$ of last-layer features $\Brac{ \mb h_{k,i} }  $ as 
\begin{align*}
    \mb h_G \;=\; \frac{1}{nK} \sum_{k=1}^K  \sum_{i=1}^n \vh_{k,i},\quad \ol{\mb h}_k \;=\; \frac{1}{n} \sum_{i=1}^n \vh_{k,i} \;(1\leq k\leq K).
\end{align*}
\textbf{Within-class variability collapse} is measured by $\mc NC_1$ which depicts the relative magnitude of the within-class covariance $\mb \Sigma_W =\frac{1}{nK} \sum_{k=1}^K \sum_{i=1}^n \paren{  \mb h_{k,i}  -  \ol{\mb h}_k } \paren{ \mb h_{k,i} - \ol{\mb h}_k }^\top \in \bb R^{d \times d} $ w.r.t. the between-class covariance $\mb \Sigma_B  = \frac{1}{K} \sum_{k=1}^K \paren{ \ol{\mb h}_k - \mb h_G } \paren{ \ol{\mb h}_k - \mb h_G }^\top \in \bb R^{d \times d} $ of the last-layer features as following:
\begin{align*}
    \mc {NC}_1 &= \frac{1}{K}\trace\parans{\mSigma_ W\mSigma_B^\dagger},
\end{align*}
where
$\mSigma_B^\dagger$ is the pseudo inverse of $\mSigma_B$.

\textbf{Convergence to simplex ETF} is measured by $\mc NC_2$ which reflects the $\ell_2$ distance between the normalized simplex ETF and the normalized $\mW \mW^\top$ as following:
\begin{align*}
\mc {NC}_2\;:=\;  \norm{ \frac{\mW \mW^\top}{\norm{\mW \mW^\top}{F}} \;-\; \frac{1}{\sqrt{K-1}}  \paren{ \mb I_K - \frac{1}{K} \mb 1_K \mb 1_K^\top } }{F},
\end{align*}
where $\mW\in \R^{K\times d}$ is the weight matrix of learned classifier.

\textbf{Convergence to self-duality} is measured by $\mc NC_3$ which calculates the $\ell_2$ distance between the normalized simplex ETF and the normalized  $\mW \ol{\mH}$ as following:
\begin{align*}
 \mc {NC}_3\;:=\; \norm{\frac{\mW \ol{\mH}}{\norm{\mW\ol{\mb H}}{F}}   \;-\; \frac{1}{\sqrt{K-1}}  \paren{ \mb I_K - \frac{1}{K} \mb 1_K \mb 1_K^\top} }{F}.
\end{align*}
where $\ol{\mH}=\begin{bmatrix}
\ol{\mb h}_1 - \mb h_G & \cdots & \ol{\mb h}_K - \mb h_G
\end{bmatrix} \in \bb R^{d \times K}$ is the centered class-mean matrix.

\subsection{Prevalence of \NC\ Across Varying Training Losses}\label{exp:NC-across-losses}
We show that all loss functions lead to \NC\ solutions during the terminal phase of training. The results on CIFAR10 using WideResNet50-2 and different loss functions is provided in Figure \ref{fig:results-nc-epochs-cifar10}.
We consistently observe that all three \NC\ metrics across different losses converge to a small value as training progresses. This supports our theoretical results in \Cref{sec:main-results} that the last-layer features learned under different losses are always maximally linearly separable and perfectly aligned with the linear classifier, and the features and the weight of linear classifier learned by different losses are almost equivalent up to a rotation and a scale of the feature space.
The evolution of three \NC\ metrics across different losses on CIFAR100 is in Appendix \ref{subsec:cifar100-experiments}. 


\begin{figure}[t]
    \centering
    \subfloat[${\mc NC}_1$ (CIFAR10)]{\includegraphics[width=0.30\textwidth]{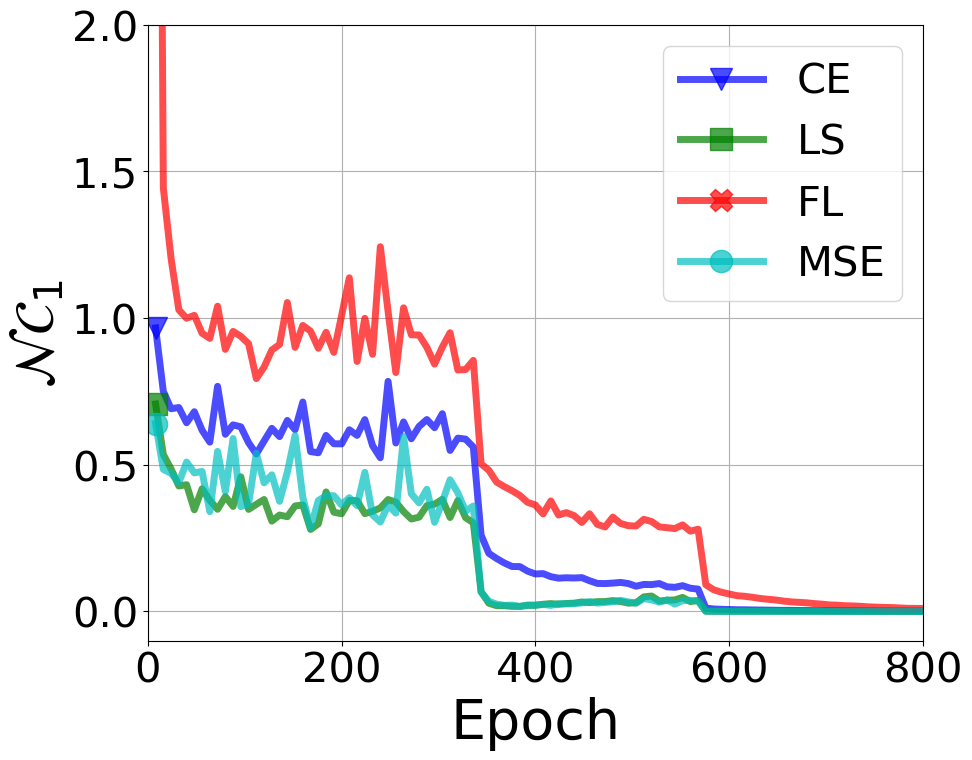}} \
    \subfloat[${\mc NC}_2$ (CIFAR10)]{\includegraphics[width=0.30\textwidth]{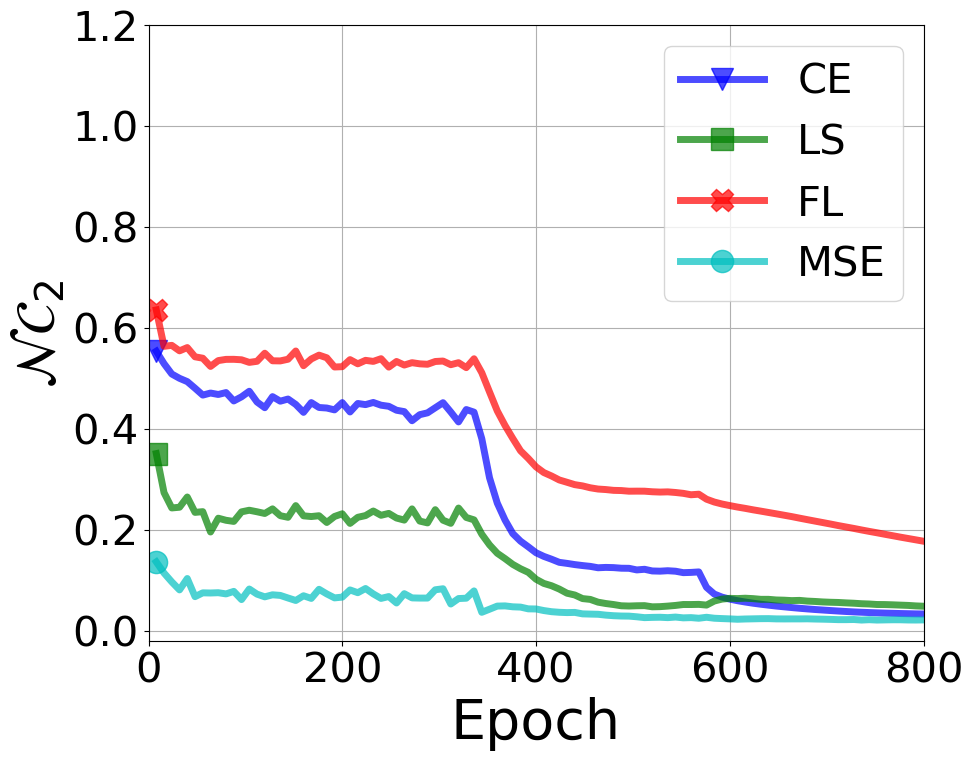}} \
    \subfloat[${\mc NC}_3$ (CIFAR10)]{\includegraphics[width=0.30\textwidth]{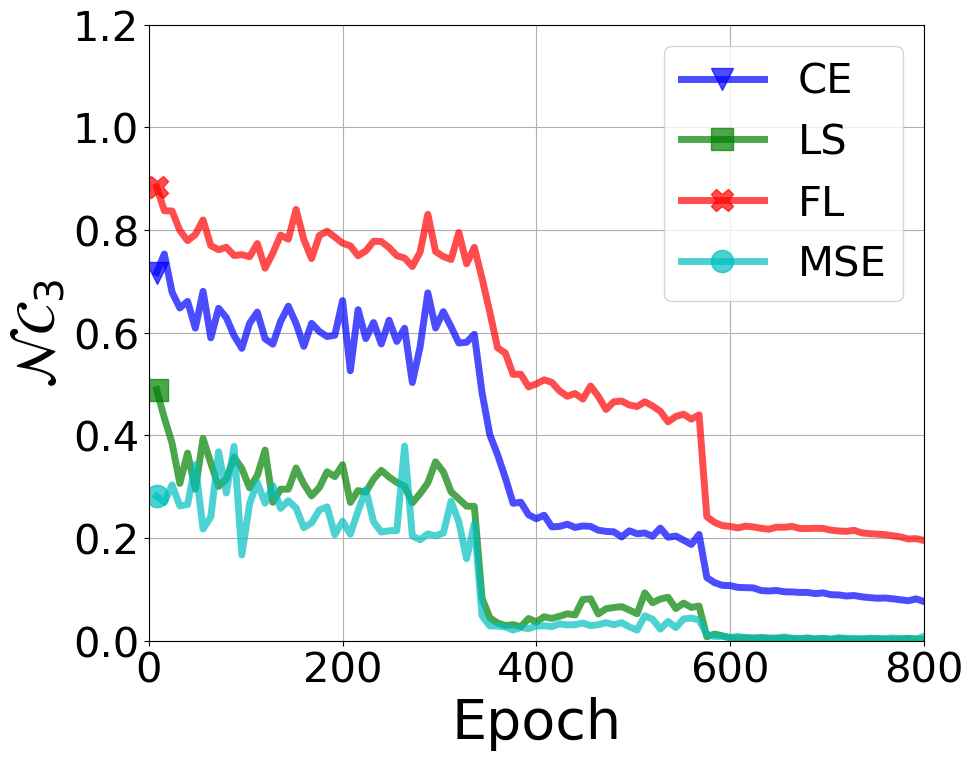}}
    \caption{\textbf{The evolution of \NC\ metrics across different loss functions.} We train the WideResNet50-2 on CIFAR10 dataset for 800 epochs using different loss functions. From left to right: \NC$_1$ (variability collapse), \NC$_2$ (convergence to simplex ETF) and \NC$_3$ (convergence to self-duality).
    }
    \label{fig:results-nc-epochs-cifar10}
\end{figure}

\begin{figure}[t]
    \centering
    \subfloat[Train Acc (CIFAR10)]{\includegraphics[width=0.30\textwidth]{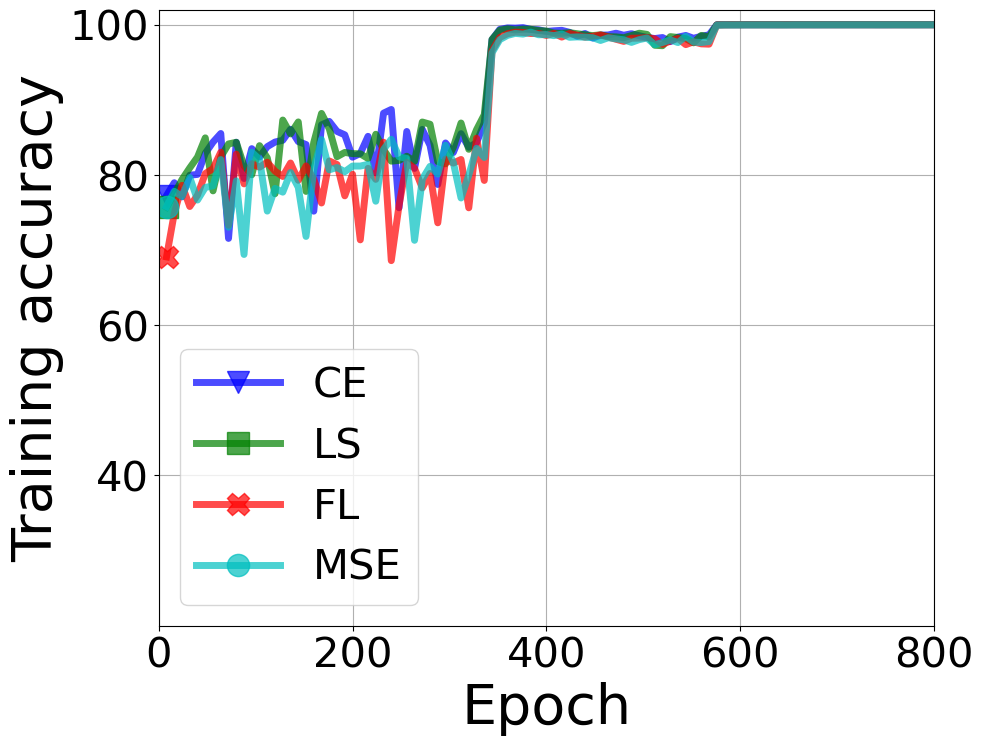}} \
    \subfloat[Val Acc (CIFAR10)]{\includegraphics[width=0.30\textwidth]{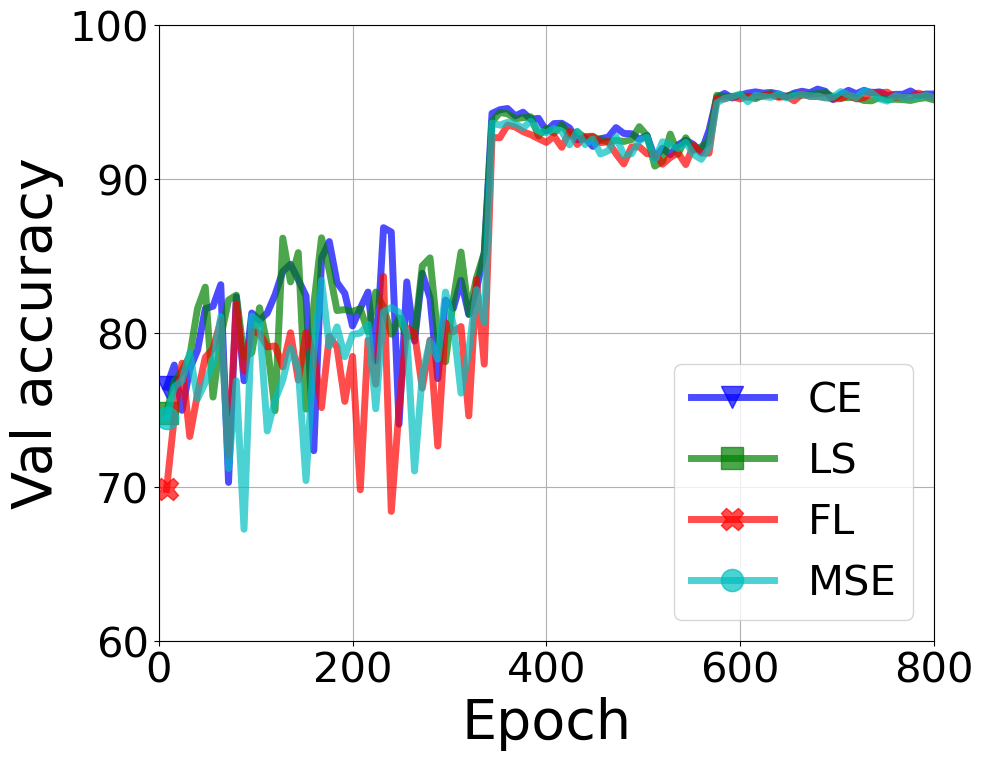}} \
    \subfloat[Test Acc (CIFAR10)]{\includegraphics[width=0.30\textwidth]{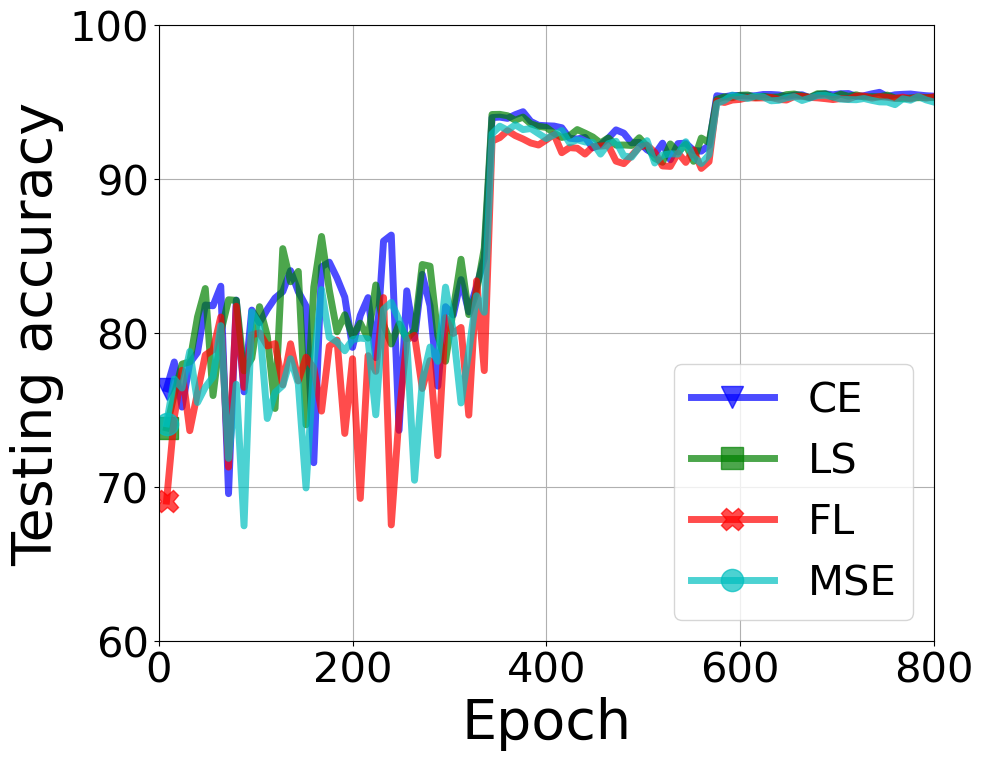}}
    \caption{\textbf{The evolution of performance across different loss functions.} We train the WideResNet50-2 on CIFAR10 dataset for 800 epochs using different loss functions. From left to right: training accuracy, validation accuracy and test accuracy.
    }
    \label{fig:results-acc-epochs-cifar10}
\end{figure}

\subsection{All Losses Lead to Largely Identical Performance}\label{exp:same-performance-across-losses}
We show that all loss functions have largely identical performance once the training procedure converges to the \NC\ global optimality. In Figure \ref{fig:results-acc-epochs-cifar10}, we plot the evolution of the training accuracy, validation accuracy and test accuracy as training progresses, where all losses are optimized on the same WideResNet50-2 architecture and CIFAR10 for 800 epochs. 
To reduce the randomness, we average the results from 3 different random seeds per width-iteration configuration, and the test accuracy is reported based on the model with best accuracy on validation set, where we organize the validation set by holding out 10 percent data from the training set. 
The results consistently show that for all cases the training accuracy converges to one hundred percent (reaching to terminal phase), and the validation accuracy and test accuracy are largely the same, as long as the network is trained longer enough and  converges to the \NC\ global solution. 

While previous work advocates the advantage of some losses over other others, our experiments show that 
\revise{when conditions between dataset and model allow for SGD to find an \NC\ solution, all losses we tested produced indistinguishable results}. In Figure \ref{fig:lossmap-cifar10}, we plot the average test accuracy of different losses under different pairs of width and iterations. 
We consistently observe three phenomenon. First, with a fixed number of iterations, increasing the width of network improves the test accuracy for all losses. This is because the wider networks (more over-parameterized) are more powerful to fit the underlying mapping from input data to the targets. Second, with a fixed width of network, increasing the number of iterations improves the test accuracy for all losses. This is because the longer optimization leads the last-layer features and the linear classifer closer to the \NC\ global solutions. Finally, while there are some unignorable difference between different losses in some width-iteration configurations, the results consistently show that all losses lead to largely identical performance when the network is sufficiently large and trained long enough to achieve a global \NC\ solution (e.g. width=2 and epochs=800). 

\begin{figure}[t]
    \centering
    
    \subfloat[Test (CE)]{\includegraphics[width=0.24\textwidth]{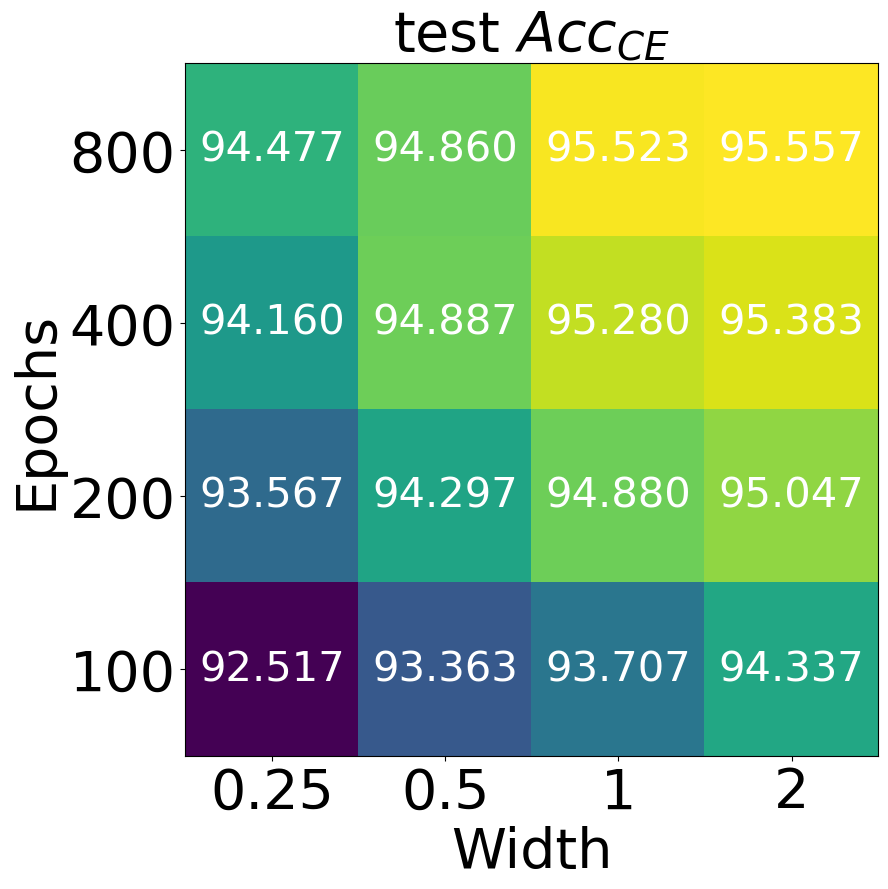}} \
    \subfloat[Test (MSE)]{\includegraphics[width=0.24\textwidth]{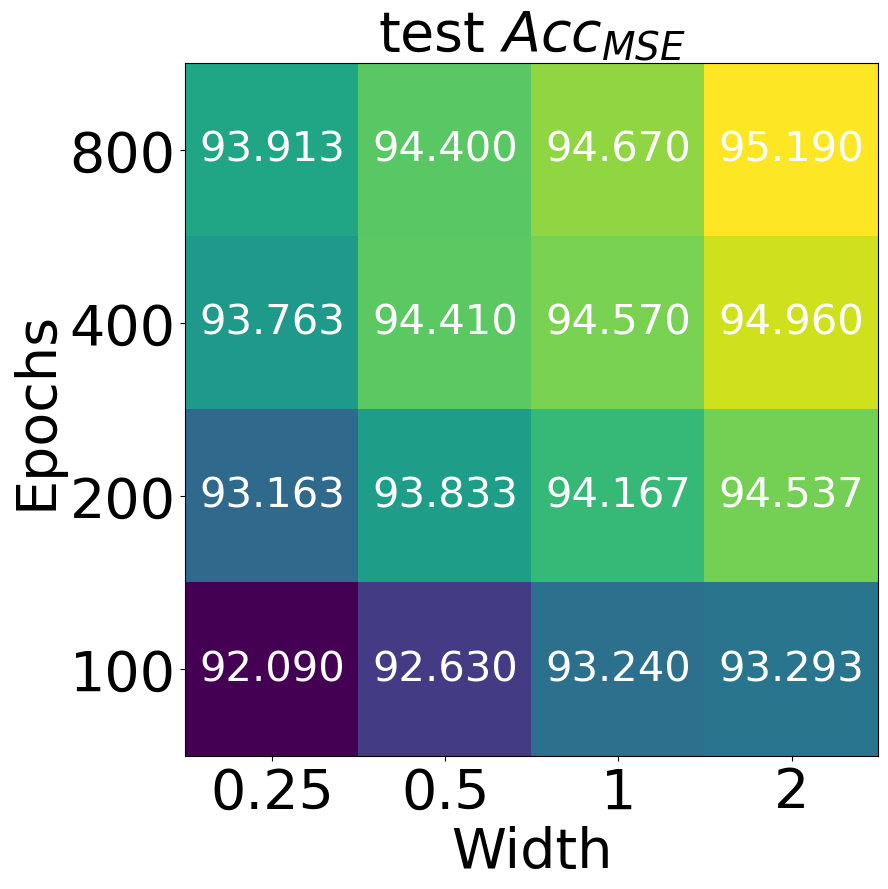}} \
    \subfloat[Test (FL)]{\includegraphics[width=0.24\textwidth]{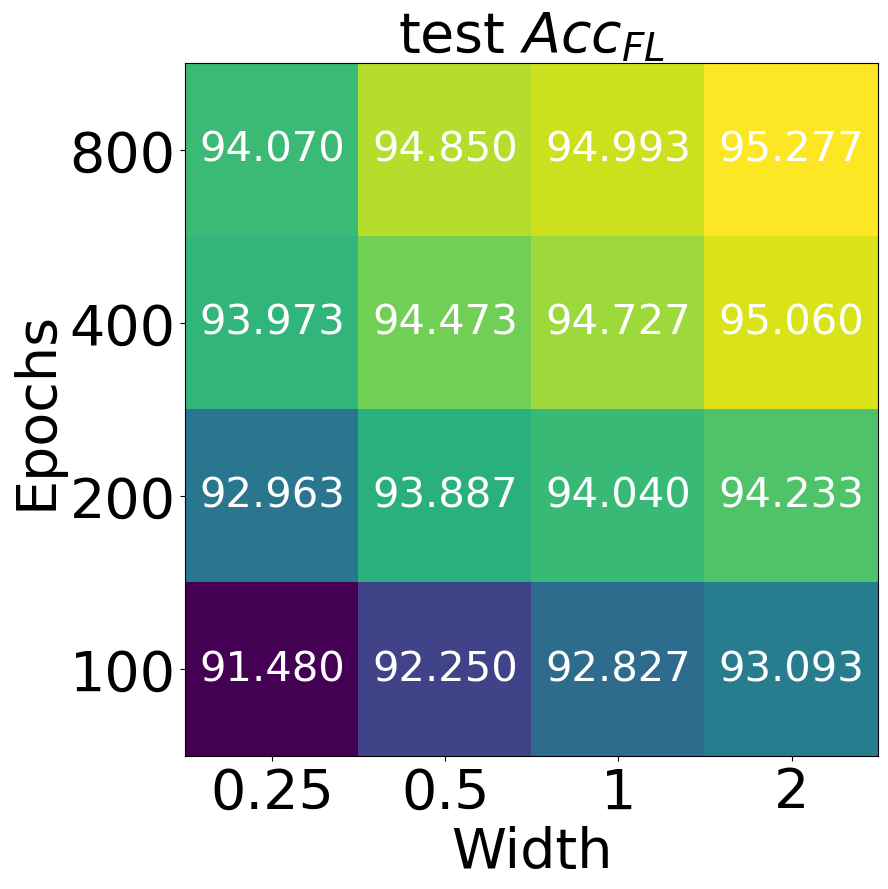}} \
    \subfloat[Test (LS)]{\includegraphics[width=0.24\textwidth]{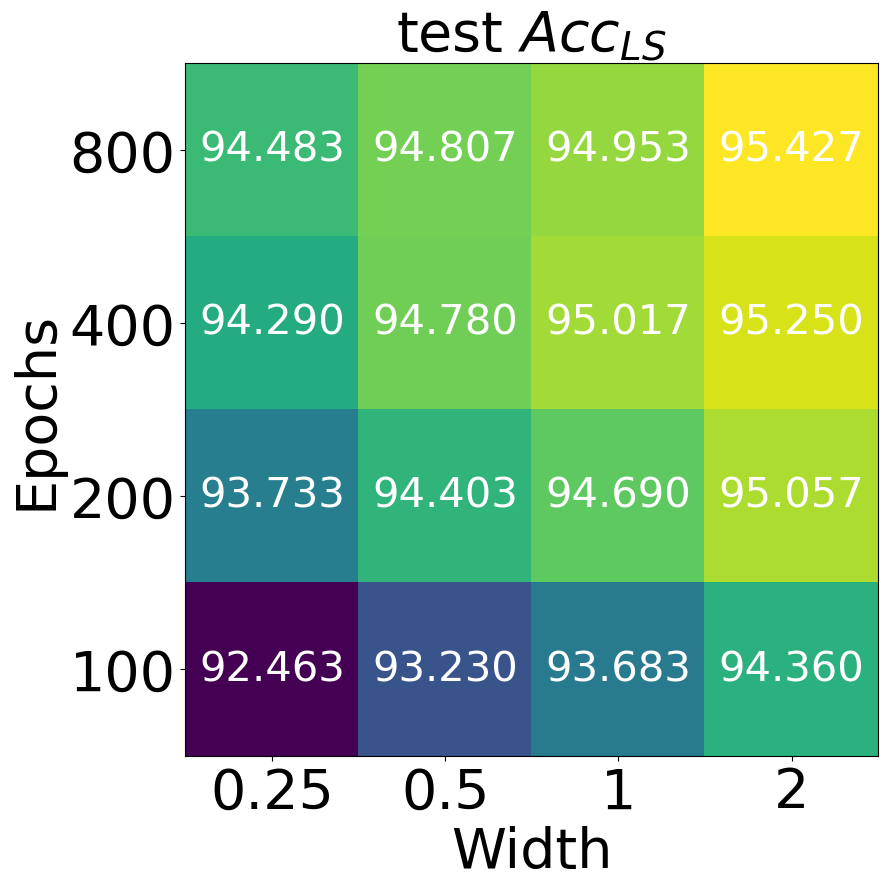}}
    \caption{\textbf{Illustration of test accuracy across different iterations-width settings.} The figure depicts the test accuracy of various iteration-width configurations for different loss functions on CIFAR10. 
    }
    \label{fig:lossmap-cifar10}
\end{figure}

\section{Conclusion}



In this work we provided a theoretical study to extend the scope of \NC, a curious phenomenon associated with last-layer features and classifier weight of a classification network, from networks trained with particular losses (i.e., CE and MSE) to those trained via a broad family of loss functions including the popular LS and FL as special cases. 
Our theory not only establishes \NC\; as the only global solutions, but also shows a benign optimization landscape that explains why \NC\; solutions are easy to obtain in practice. 
Such results readily suggest that all relevant losses (i.e., CE, MSE, LS, and FL) produce entirely equivalent features on the training data. 
Although \NC\; is an optimization phenomenon pertaining \revise{to} training data only, we found through experiments that all relevant losses (i.e., CE, MSE, LS, and FL) lead to very similar test performance as well. 
Such a result may come as a surprise to the common belief that some losses are intrinsically better than the others, and clarify some mystery on how different losses affect the performance. 

The family of loss functions considered in this paper by no means is inclusive of all possible loss functions that lead to \NC.
There are many other popular loss functions, such as center loss \cite{wen2016discriminative}, large-margin softmax (L-Softmax) loss \cite{liu2016large} and many of its variants \cite{liu2017sphereface,wang2018cosface,deng2019arcface}, which are all designed with the intuition of encouraging intra-class compactness and inter-class separability between learned features. 
In addition, many generalized versions of the cross-entropy loss such as those for robust learning under label noise \cite{zhang2018generalized,wang2019symmetric,wei2021optimizing} and long-tail distribution \cite{menon2021long,jitkrittum2022elm} may have similar property as the vanilla cross-entropy loss. 
We conjecture that many of them provably produce \NC\; solutions under unconstrained feature models, while leave a formal justification to future work. 
Beyond losses for classification task, \NC\; may also arise with popular losses used in metric learning \cite{chechik2010large,sohn2016improved} evidenced by recent study \cite{levi2020reducing}. 
This means that the observations from this paper, namely all losses lead to largely the same test performance, may apply for all such losses as well.

\paragraph{Loss functions that do not lead to \NC.} While the study in this paper covers many of the most commonly used loss functions for classification tasks, we note that there are alternative choices in the literature which do not induce \NC\;features.
Many of such losses such as \cite{zhang2017learning,qian2019softtriple,yu2020learning,levi2020reducing}
are particularly designed to discourage variability collapse and learn diverse features, which are shown to benefit model transferability \cite{kornblith2021better} and robustness. 

\section*{Acknowledgment} \revise{JZ and ZZ acknowledge support from NSF grants CCF-2008460 and CCF-2106881. 
XL and QQ  acknowledge support from U-M START \& PODS grants, NSF CAREER CCF 2143904, NSF CCF 2212066, NSF CCF 2212326, and ONR N00014-22-1-2529 grants. SL acknowledges support from NSF NRT-HDR Award 1922658 and Alzheimer’s Association grant AARG- NTF-21-848627. }

{\small
\bibliographystyle{unsrt}
\bibliography{biblio/optimization,biblio/learning}
}

\newpage
\section*{Checklist}

\begin{enumerate}

\item For all authors...
\begin{enumerate}
  \item Do the main claims made in the abstract and introduction accurately reflect the paper's contributions and scope?
    \answerYes
  \item Did you describe the limitations of your work?
    \answerYes{In the abstract, introduction, main results, and conclusion, we explicitly stat that our results are about unconstrained feature model.}
  \item Did you discuss any potential negative societal impacts of your work?
    \answerNA{This paper mainly focuses on understanding the neural collapse phenomena observed in practical neural networks. Based on this understanding, we propose to fix the last layer classifier as a Simplex ETF. So no potential negative societal impact is expected of this work.}
  \item Have you read the ethics review guidelines and ensured that your paper conforms to them?
    \answerYes 
\end{enumerate}

\item If you are including theoretical results...
\begin{enumerate}
  \item Did you state the full set of assumptions of all theoretical results?
    \answerYes{We explicitly mention the assumptions in \Cref{thm:global-minima} and \Cref{thm:global-geometry}.}
	\item Did you include complete proofs of all theoretical results?
   \answerYes{We include all the proofs in the Appendix.}
\end{enumerate}

\item If you ran experiments...
\begin{enumerate}
  \item Did you include the code, data, and instructions needed to reproduce the main experimental results (either in the supplemental material or as a URL)?
    \answerYes{See \Cref{sec:experiment}}
  \item Did you specify all the training details (e.g., data splits, hyperparameters, how they were chosen)?
   \answerYes{See \Cref{sec:experiment}}
	\item Did you report error bars (e.g., with respect to the random seed after running experiments multiple times)?
    \answerYes{All the results are obtained by averaging the resluts from 3 different random seeds. See \Cref{sec:experiment}}
	\item Did you include the total amount of compute and the type of resources used (e.g., type of GPUs, internal cluster, or cloud provider)?
    \answerYes{See \Cref{app:experiments}.}
\end{enumerate}

\item If you are using existing assets (e.g., code, data, models) or curating/releasing new assets...
\begin{enumerate}
  \item If your work uses existing assets, did you cite the creators?
    \answerYes{We cite them in \Cref{app:experiments}.}
  \item Did you mention the license of the assets?
    \answerYes{This is mentioned in \Cref{app:experiments}.}
  \item Did you include any new assets either in the supplemental material or as a URL?
    \answerNA{}
  \item Did you discuss whether and how consent was obtained from people whose data you're using/curating?
    \answerYes{This is discussed in \Cref{app:experiments}.}
  \item Did you discuss whether the data you are using/curating contains personally identifiable information or offensive content?
    \answerNA{The CIFAR datasets do not contain personally identifiable information or offensive content}
\end{enumerate}

\item If you used crowdsourcing or conducted research with human subjects...
\begin{enumerate}
  \item Did you include the full text of instructions given to participants and screenshots, if applicable?
    \answerNA{Our research does not involve with participants and screenshots}
  \item Did you describe any potential participant risks, with links to Institutional Review Board (IRB) approvals, if applicable?
    \answerNA{Our research does not include any potential participant risks, with links to Institutional Review Board (IRB) approvals}
  \item Did you include the estimated hourly wage paid to participants and the total amount spent on participant compensation?
    \answerNA{Our work does not require participants.}
\end{enumerate}

\end{enumerate}


\newpage 
\appendices
\crefalias{section}{appendix}
\paragraph{Organizations and Basic.}
The appendix is organized as follows. We first introduce the basic definitions and inequalities used throughout the appendices. In \Cref{app:experiments}, we provide more details about the datasets, computational resources, and more experiment results on CIFAR10, CIFAR100 and miniImageNet datasets. In \Cref{app:loss-in-GL}, we prove that CE, FL and LS satisfy the contrastive property in \Cref{def:GLoss}. In \Cref{app:thm-global-fl}, we provide a detailed proof for \Cref{thm:global-minima}, showing that the Simplex ETFs are the \emph{only} global minimizers, as long as the loss function satisfies the \Cref{def:GLoss}. Finally, in Appendix \ref{sec:appendix-prf-global-geometry}, we present the whole proof for \Cref{thm:global-geometry} that the FL function is a locally strict saddle function with no spurious local minimizers existing locally and LS function is a globally strict saddle function with no spurious local minimizers existing globally. 

\begin{definition}[$K$-Simplex ETF]\label{def:simplex-ETF}
A standard Simplex ETF is a collection of points in $\bb R^K$ specified by the columns of
\begin{align*}
    \mb M \;=\;  \sqrt{\frac{K}{K-1}}  \paren{ \mb I_K - \frac{1}{K} \mb 1_K \mb 1_K^\top },
\end{align*}
where $\mb I_K \in \bb R^{K \times K}$ is the identity matrix, and $\mb 1_K \in \bb R^K$ is the all ones vector. In the other words, we also have
\begin{align*}
    \mb M^\top \mb M \;=\; \mb M \mb M^\top \;=\; \frac{K}{K-1}  \paren{ \mb I_K - \frac{1}{K} \mb 1_K \mb 1_K^\top }.
\end{align*}

As in \cite{papyan2020prevalence,fang2021layer}, in this paper we consider general Simplex ETF as a collection of points in $\R^d$ specified by the columns of $\sqrt{\frac{K}{K-1}} \mP \paren{ \mb I_K - \frac{1}{K} \mb 1_K \mb 1_K^\top }$, where $\mP\in\R^{d\times K} (d\ge K)$ is an orthonormal matrix, i.e., $\mP^\top \mP = \mb I_K$. 

\end{definition}

\begin{lemma}[Young's Inequality]\label{lem:young-inequality}
Let $p,q$ be positive real numbers satisfying $\frac{1}{p}+\frac{1}{q} = 1$. Then for any $a,b \in \bb R$, we have
\begin{align*}
    \abs{ab} \; \leq\; \frac{ \abs{a}^p }{p}\;+\; \frac{ \abs{b}^q }{q},
\end{align*}
where the equality holds if and only if $\abs{a}^p=\abs{b}^q$. The case $p=q=2$ is just the AM-GM inequality for $a^2,\;b^2$: $\abs{ab}\leq \frac{1}{2}\paren{a^2+b^2}$, where the equality holds if and only if $\abs{a} = \abs{b}$.
\end{lemma}

The following Lemma extends the standard variational form of the nuclear norm.
\begin{lemma} \label{lem:nuclear-norm}
For any fixed $\mW\in\R^{K\times d}$, $\mH_i\in\R^{d\times K}$, $\bar{\mZ}_i=\mW\mH_i\in\R^{K\times K}$ and $\alpha>0$, we have
\begin{align}\label{eqn:nuclear-norm-eq}
    \norm{\bar{\mZ}_i}{*} \;\leq\; \frac{1}{ 2 \sqrt{ \alpha } } \paren{ \norm{\mb W}{F}^2 + \alpha \norm{\mb H_i}{F}^2  }.
\end{align}
Here, $\norm{\bar{\mZ}_i}{*}$ denotes the nuclear norm of $\bar{\mZ}_i$:
\begin{align*}
    \norm{\bar{\mZ}_i}{*} \;:=\; \sum_{k=1}^{ K} \sigma_{k}(\bar{\mZ}_i) = \trace\paren{\mb \Sigma},\quad \text{with}\quad \bar{\mZ}_i \;=\; \mb U \mb \Sigma \mb V^\top,
\end{align*}
where $\Brac{\sigma_{k}}_{k=1}^{ K }$ denotes the singular values of $\bar{\mZ}_i$, and $\bar{\mZ}_i =\mb U \mb \Sigma \mb V^\top$ is the singular value decomposition (SVD) of $\bar{\mZ}_i$.
\end{lemma}

\begin{proof}[Proof of Lemma \ref{lem:nuclear-norm}] 
Let $\bar{\mZ}_i = \mb U \mb \Sigma \mb V^\top$ be the SVD of $\bar{\mZ}_i$. For any $\mb W\mb H_i = \bar{\mZ}_i$, we have
\begin{align*}
    \norm{\bar{\mZ}_i}{*} \;&=\; \trace \paren{ \mb \Sigma } \;=\; \trace \paren{ \mb U^\top \bar{\mZ}_i \mb V } \;=\; \trace\paren{ \mb U^\top \mb W \mb H_i \mb V } \\
    \;&\leq\; \frac{1}{2\sqrt{\alpha}} \norm{\mb U^\top \mb W}{F}^2 + \frac{ \sqrt{\alpha} }{2} \norm{ \mb H_i \mb V }{F}^2 \;\leq\; \frac{1}{2\sqrt{\alpha}} \paren{ \norm{\mb W}{F}^2 + \alpha \norm{\mb H_i}{F}^2 }, 
\end{align*}
where the first inequality utilize the Young's inequality in Lemma \ref{lem:young-inequality} that $\abs{\trace(\mA\mB)}\le \frac{1}{2c}\norm{\mA}{F}^2 + \frac{c}{2}\norm{\mB}{F}^2$ for any $c>0$ and $\mA,\mB$ of appropriate dimensions, 
and the last inequality follows because $\norm{\mb U}{}= 1$ and $\norm{\mb V}{}= 1$. 
Therefore, we have 
\begin{align*}
    \norm{\bar{\mZ}_i}{*} \;\leq\; \frac{1}{ 2 \sqrt{ \alpha } } \paren{ \norm{\mb W}{F}^2 + \alpha \norm{\mb H_i}{F}^2  }.
\end{align*}
We complete the proof.
\end{proof}

\begin{lemma}[Eigenvalues of Diagonal-Plus-Rank-One Matrices]\label{lem:diagonal-plus-rank-one}
Let $\tau<0$, $\vz\in\R^n$, and $\mD$ be an $n\times n$ diagonal matrix with diagonals $d_1,\ldots,d_n$. Let $\lambda_1,\ldots, \lambda_n $ be the  eigenvalues of the diagonal-plus-rank-one matrix $\mD+\tau\vz\vz^\top$.
\begin{itemize}
    \item {Case 1:} If $d_1>d_2>\cdots>d_n$ and $z_i\neq 0$ for all $i=1,\cdots,n$, then the eigenvalues $\{\lambda_i\}$  are equal to the $n$ roots of the rational function \cite{cuppen1980divide, stor2015forward}
    \begin{align*}
    w(\lambda)=1+\tau\vz^\top\paren{\mD-\lambda \mb I}^{-1}\vz=1+\tau\sum_{j=1}^n\frac{z_j^2}{d_j-\lambda},
    \end{align*}
    and the diagonals $\{d_i\}$ strictly separate the eigenvalues as following:
    \begin{align}
     d_1>\lambda_1>d_2>\lambda_2>\cdots>d_n>\lambda_n.\label{lem:inequal}
    \end{align}
    
    \item {Case 2:}If $z_i=0$ for some $i$, then $d_i$ is an eigenvalue of $\mD+\tau\vz\vz^\top$ with  corresponding eigenvector $\ve_i$ since
        \begin{align*}
        (\mb D+\tau\vz\vz^\top)\ve_i=d_i\ve_i+\tau\vz z_i=d_i\ve_i.
    \end{align*}
   
    The remaining $n-1$ eigenvalues of $\mD+\tau\vz\vz^\top$ are equal to the eigenvalues of the smaller matrix $\mD'+\tau\vz'\vz^{'\top}$, where $\mD'\in\R^{(n-1)\times (n-1)}$ and $\vz'\in\R^{n-1}$ are obtained by removing the $i$-th rows and columns from $\mD$ and the $i$-th element from $\vz$, respectively. One can repeat this process if $\vz'$ still has zero element.
  
   
    \item {Case 3:} If there are $m$ mutually equal diagonal elements, say $d_{i+1}=\cdots=d_{i+m}=d$, then for any orthogonal $m\times m$ matrix $\mP$,  $\mD+\tau\vz\vz^\top$ has the same eigenvalues as  
    \[
  \mT \mD \mT^\top + \tau (\mT\vz)(\mT\vz)^\top =  \mD + \tau \wh{\vz}\wh{\vz}^\top, \ \text{where} \ \mT =\begin{bmatrix}
    \mb I_i& &\\
    &\mP&\\
    &&\mb I_{n-i-m}
    \end{bmatrix}, \wh \vz = \mT \wh \vz.
    \]
   We can then choose $\mP$ as a Householder transformation such that  
   \[
   \mP \begin{bmatrix} z_{i+1} & z_{i+2} & \cdots & z_{i+m}\end{bmatrix}^\top =  \begin{bmatrix}0 & 0 & \cdots & \sqrt{\sum_{j=i+1}^{i+m} z_j^2}\end{bmatrix}^\top.
   \]
   Thus, according to Case 2, $d$ is an eigenvalue of $\mD + \tau \wh{\vz}\wh{\vz}^\top$ repeated $m-1$ times and the remaining eigenvalues can be computed by checking the smaller matrix. 
    \end{itemize}
\end{lemma}
Based on \Cref{lem:diagonal-plus-rank-one}, we can prove the following Lemma.

\begin{lemma}\label{lem:Z-structure}
Let $K \ge 3$ and $\mZ=-\paren{\mb I_K-\frac{1}{K}\vone\vone^\top}\text{diag}\paren{\rho_1, \rho_2,\cdots,\rho_K}$ with $|\rho_1|\geq|\rho_2|\geq\cdots\geq|\rho_K|$ and $|\rho_1|>0$.  Also let $\sigma_i \ge 0$ be the $i$-th largest singular value of $\mZ$.  Suppose there exists $k$ with $1\le k\le K-1$ such that 
\e
\sigma_1 = \cdots = \sigma_k = \sigma_{\max}>0 \ \text{and} \
\sigma_{k+1} = \cdots = \sigma_K = 0.
\label{eq:assump-sigs}\ee 
Then $|\rho_1|, \cdots, |\rho_K|$ must satisfy either
\begin{align*}
    |\rho_1|=|\rho_2|=\cdots=|\rho_K|, \quad \text{with} \quad \sigma_{\max} = |\rho_1|,
\end{align*}
or
\begin{align*}
    \rho_2=\cdots=\rho_K=0, \quad  \text{with} \quad \sigma_{\max} = \sqrt{\frac{K-1}{K}}|\rho_1|.
\end{align*}
\end{lemma}

\begin{proof}[Proof of Lemma \ref{lem:Z-structure}] Because
\begin{align*}
    \mZ^\top\mZ&=\text{diag}\paren{\rho_1, \rho_2,\cdots,\rho_K}\paren{\mb I_K-\frac{1}{K}\vone\vone^\top}\text{diag}\paren{\rho_1, \rho_2,\cdots,\rho_K}\\
    &=\text{diag}\paren{\rho_1^2, \rho_2^2,\cdots,\rho_K^2}-\frac{1}{K}\mb\rho\mb\rho^\top
\end{align*}
where $\mb\rho=\begin{bmatrix}\rho_1&\rho_2&\cdots&\rho_K\end{bmatrix}^\top$, $\mZ^\top\mZ$ satisfies the form of Diagonal-Plus-Rank-One in \Cref{lem:diagonal-plus-rank-one} with $\mb D=\text{diag}\paren{\rho_1^2, \rho_2^2,\cdots,\rho_K^2}$, $\vz=\mb\rho$ and $\tau=-\frac{1}{K}$. Let $\lambda_1\ge\lambda_2\ge\cdots\lambda_K\ge0$ denote the $n$ eigenvalues of $\mZ^\top\mZ$. Due to $\vone^\top\mZ=\vzero^\top$, we can have $\lambda_K=0$.
\begin{itemize}
    \item {If $|\rho_1|=|\rho_2|=\cdots=|\rho_K|$:} we have
    \begin{align*}
        \rho_1^2=\lambda_1=\cdots=\lambda_{K-1}=\rho_K^2>\lambda_K=0.
    \end{align*}
    Thus, $\sigma_{\max}=\sqrt{\lambda_1}=|\rho_1|$. 
    \item {If $|\rho_1|>|\rho_2|=\cdots=|\rho_K|=0$:} according to Case 2 in \Cref{lem:diagonal-plus-rank-one}, we have
    \begin{align*}
        \lambda_1 = \paren{1-1/K}\rho_1^2>\rho_2^2=\lambda_2\cdots=\rho_K^2=\lambda_K=0.
    \end{align*}
    Thus, $\sigma_{\max}=\sqrt{\paren{1-1/K}\rho_1^2}=\sqrt{{(K-1)}/{K}}|\rho_1|$.
    \item {If $|\rho_1|>|\rho_2|=\cdots=|\rho_K|\neq0$:} according to Case 3 in \Cref{lem:diagonal-plus-rank-one}, we have
    \begin{align*}
    \lambda_2\cdots=\lambda_{K-1} = \rho_2^2
    \end{align*}
    and the remaining two eigenvalues are the same to those of $\begin{bmatrix}\rho_1^2&\\&\rho_{K}^2\end{bmatrix}+(-\frac{1}{K})\begin{bmatrix}\rho_1\\\sqrt{K-1}\rho_K\end{bmatrix}\begin{bmatrix}\rho_1&\sqrt{K-1}\rho_K\end{bmatrix}$. According to \eqref{lem:inequal} in \Cref{lem:diagonal-plus-rank-one}, we can obtain
    \begin{align*}
        \rho_1^2>\lambda_1>\rho_K^2>\lambda_K=0.
    \end{align*} Combing them together, we can have \begin{align*}
        \rho_1^2>\lambda_1>\rho_2^2=\lambda_2\cdots=\rho_K^2>\lambda_K=0
    \end{align*}thus, $0=\lambda_K<\lambda_2<\lambda_1=\lambda_{\max}$, which violates the assumption \eqref{eq:assump-sigs}.
    \item {If $|\rho_{1}|=\cdots=|\rho_{i}|>|\rho_{i+1}|=\cdots=|\rho_{K}|=0$ and $1<i<K$: }according to the Case 2 and Case 3 in \Cref{lem:diagonal-plus-rank-one}, we can have 
    \begin{align*}
        &\lambda_1=\cdots=\lambda_{i-1}=\rho_1^2\\
        &\lambda_{i+1}=\cdots=\lambda_K=0
    \end{align*}
    and $0<\lambda_i=\rho_1^2-\frac{i}{K}\rho_1^2<\rho_1^2=\lambda_{\max}$, which violates the assumption \eqref{eq:assump-sigs}.
    \item {If $|\rho_{1}|=\cdots=|\rho_{i}|>|\rho_{i+1}|=\cdots=|\rho_{K}|\neq0$ and $1<i<K$: }according to Case 3 in \Cref{lem:diagonal-plus-rank-one}, we have 
    \begin{align*}
        &\lambda_1=\cdots=\lambda_{i-1}=\rho_1^2\\
        &\lambda_{i+1}=\cdots=\lambda_{K-1}=\rho_K^2
    \end{align*}
    and the remaining two eigenvalues are the same to those of $\mD=\begin{bmatrix}\rho_1^2&\\&\rho_{K}^2\end{bmatrix}+(-\frac{1}{K})\begin{bmatrix}\sqrt{i}\rho_1\\\sqrt{K-i}\rho_K\end{bmatrix}\begin{bmatrix}\sqrt{i}\rho_1&\sqrt{K-i}\rho_K\end{bmatrix}$. According to \eqref{lem:inequal} in \Cref{lem:diagonal-plus-rank-one}, we can obtain
    \begin{align*}
        \rho_1^2=\rho_i^2>\lambda_i>\rho_K^2>\lambda_K=0.
    \end{align*}
    Combing them together, we can have
    \begin{align*}
        \rho_1^2=\lambda_1=\cdots=\rho_i^2>\lambda_i>\rho_{i+1}^2=\lambda_{i+1}=\cdots=\rho_K^2>\lambda_K=0
    \end{align*}
    thus, $0=\lambda_K<\lambda_i<\lambda_1=\lambda_{\max}$, which violates the assumption \eqref{eq:assump-sigs}.
    \item {If $|\rho_{1}|>|\rho_{i}|>|\rho_{K}|$ for some $1<i<K$: } Suppose $|\rho_{1}|=\cdots=|\rho_m|$, $|\rho_{i}|=\cdots=|\rho_{i+n-1}|$ and $|\rho_{K-t+1}|=\cdots=|\rho_K|$, where $m<i$, $i+n-1<K-t+1$ and $m,n,t\geq 1$. According to the \eqref{lem:inequal}, Case 2 and Case 3 in \Cref{lem:diagonal-plus-rank-one}, we can find 
    \begin{align*}
        \rho_m^2>\lambda_m>\rho_{i}^2\geq\lambda_{i+n-1}>\rho_K^2\geq\lambda_K=0
    \end{align*}
    thus, $0=\lambda_K<\lambda_{i+n-1}<\lambda_m\leq\lambda_{\max}$, which violates the assumption \eqref{eq:assump-sigs}.
\end{itemize}
We complete the proof.
\end{proof}

\section{Experiments}\label{app:experiments}
In this section, we first describe more details about the datasets and the computational resource used in the paper. Particularly, all CIFAR10, CIFAR100 and \revise{miniImageNet} are publicly available for academic purpose under the MIT license, and we run all experiments on a single RTX3090 GPU with 24GB memory. Moreover, \revise{additional experimental results on CIFAR10, CIFAR100 and miniImageNet are presented in \Cref{subsec:cifar10-experiments},  \Cref{subsec:cifar100-experiments}, and \Cref{subsec:mini-experiments}, respectively.}

\subsection{Additional experimental results on CIFAR10}\label{subsec:cifar10-experiments}
\revise{In \Cref{sec:experiment}, we present the test accuracy for different losses function across various different iteration-width configurations. Moreover, we further show the ${\mc NC}_1$ for different loss functions across different iteration-width configurations , and we reuse the results of test accuracy in \Cref{fig:lossmap-cifar10} for better investigation. The experiment results in \Cref{fig:lossmap-nc1-acc-cifar10} consistently show that the value of ${\mc NC}_1$ of training WideResNet50-0.25 for 100 epochs is around three orders of magnitude larger than it of training WideResNet50-2 for 800 epochs, which indicates that the previous configuration setting is much less collapsed than the latter one. In terms of test accuracy, the maximal difference across different losses for $\text{width}=0.25$ and $\text{epochs}=100$ configuration is $1.037\%$, which is larger than $0.36\%$ for $\text{width}=2$ and $\text{epochs}=800$ configuration. These results support our claim that all losses lead to identical performance, as long as the network has sufficient approximation power and the number of optimization is enough for the convergence to the \NC\ global optimality.}
\begin{figure}[t]
    \centering
    \subfloat[${\mc NC}_1$ (CE)]{\includegraphics[width=0.23\textwidth]{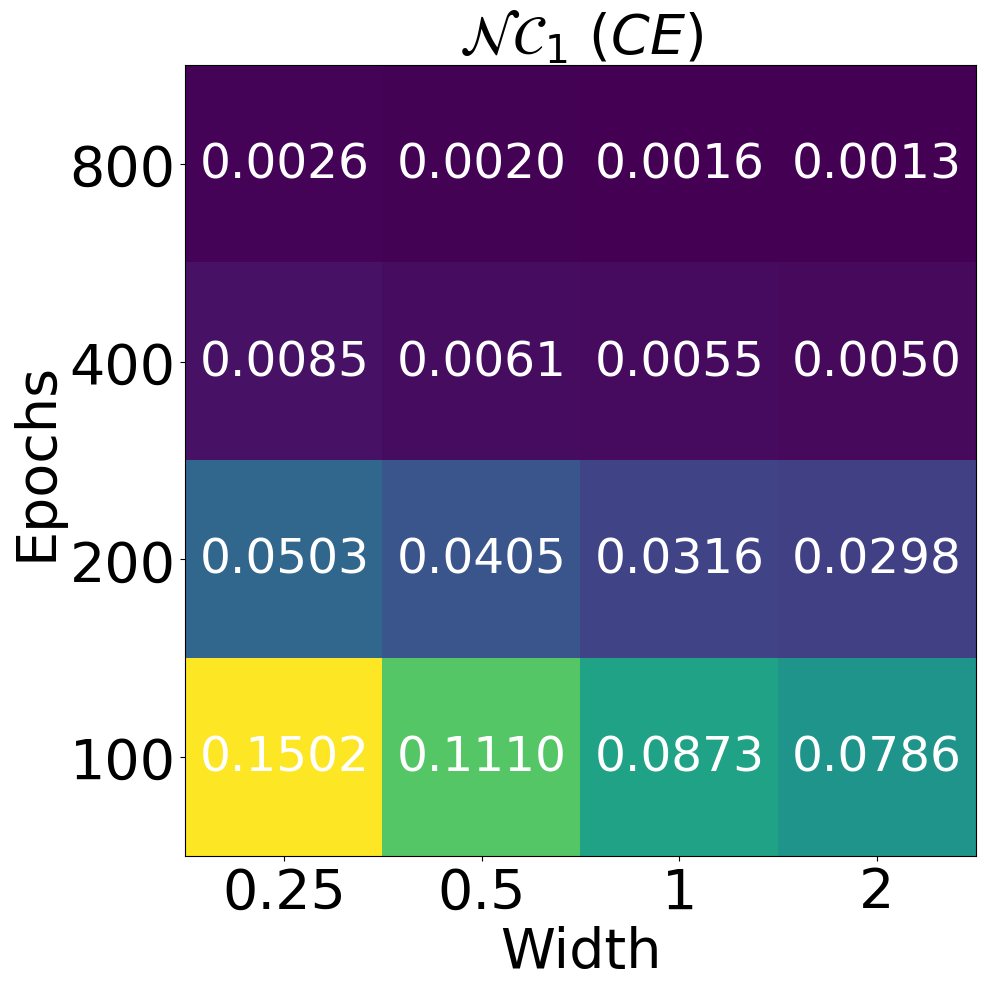}} \
    \subfloat[${\mc NC}_1$ (MSE)]{\includegraphics[width=0.23\textwidth]{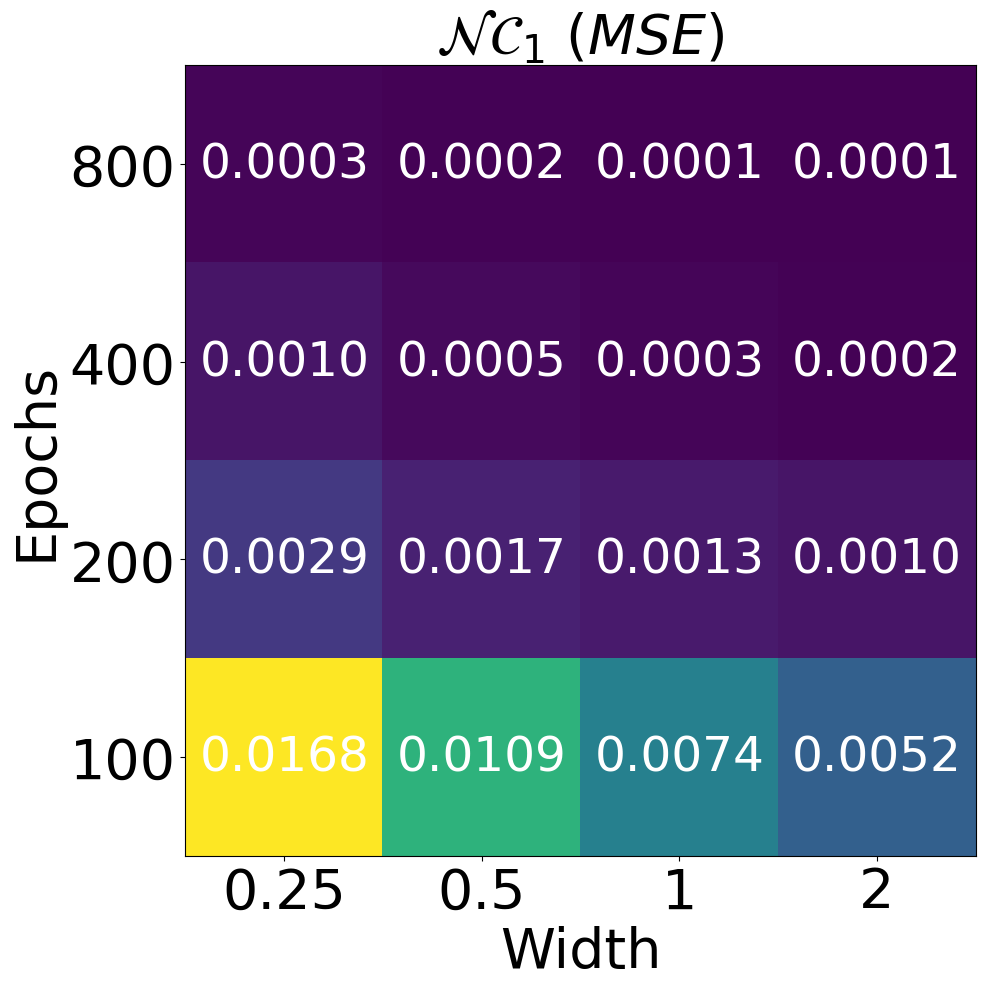}} \
    \subfloat[${\mc NC}_1$ (FL)]{\includegraphics[width=0.23\textwidth]{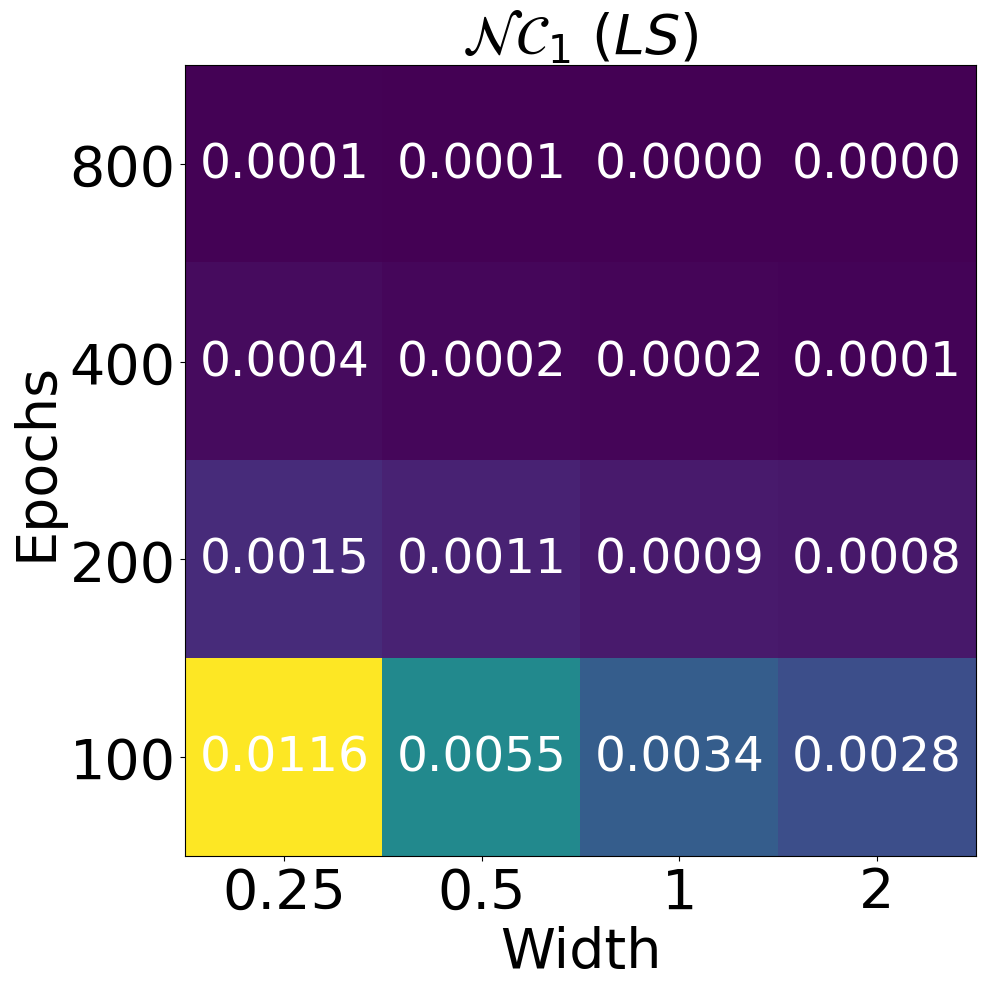}} \
    \subfloat[${\mc NC}_1$ (LS)]{\includegraphics[width=0.23\textwidth]{figures/res50-cifar10-loss-diff/nc1-ls.png}}\\
    
    \subfloat[Test (CE)]{\includegraphics[width=0.23\textwidth]{figures/res50-cifar10-loss-diff/lossmap-ce-test.png}} \
    \subfloat[Test (MSE)]{\includegraphics[width=0.23\textwidth]{figures/res50-cifar10-loss-diff/lossmap-mse-test.png}} \
    \subfloat[Test (FL)]{\includegraphics[width=0.23\textwidth]{figures/res50-cifar10-loss-diff/lossmap-fl-test.png}} \
    \subfloat[Test (LS)]{\includegraphics[width=0.23\textwidth]{figures/res50-cifar10-loss-diff/lossmap-ls-test.png}}
    \caption{\textbf{Illustration of \revise{${\mc NC}_1$ and} test accuracy across different iterations-width configurations.} The figure depicts the \revise{${\mc NC}_1$} and test accuracy of various iteration-width configurations for different loss functions on CIFAR10. 
    }
    \label{fig:lossmap-nc1-acc-cifar10}
\end{figure}

\subsection{Additional experimental results on CIFAR100}\label{subsec:cifar100-experiments}

In this parts, we show the additional results on CIFAR100 dataset. 

\paragraph{Prevalence of \NC\ Across Varying Training Losses}
We show that all loss functions lead to \NC\ solutions during the terminal phase of training on CIFAR100 dataset. The results on CIFAR100 using WideResNet50-2 and different loss functions is provided in Figure \ref{fig:results-nc-epochs-cifar100}.
We consistently observe that all three \NC\ metrics of FL and MSE converge to a small value as training progresses, and metrics of CE and FL still continue to decrease at the last iteration, because CIFAR100 is more difficult than CIFAR10 and requires networks to be optimized longer. The decreasing speed of FL is slowest, which is consistent with our global landscape analysis that FL has benign landscape in the local region near optimality. These results imply that all losses exhibit \NC\ at the end, regardless of the choice of loss functions. 
\begin{figure}[t]
    \centering
    \subfloat[${\mc NC}_1$ (CIFAR100)]{\includegraphics[width=0.30\textwidth]{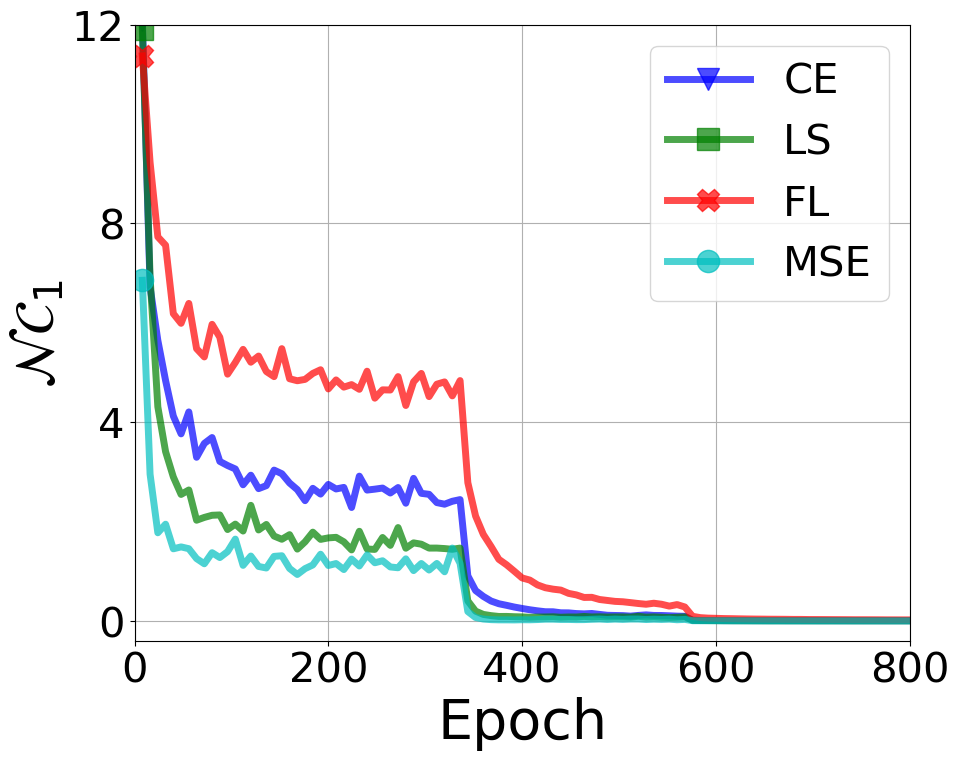}} \
    \subfloat[${\mc NC}_2$ (CIFAR100)]{\includegraphics[width=0.30\textwidth]{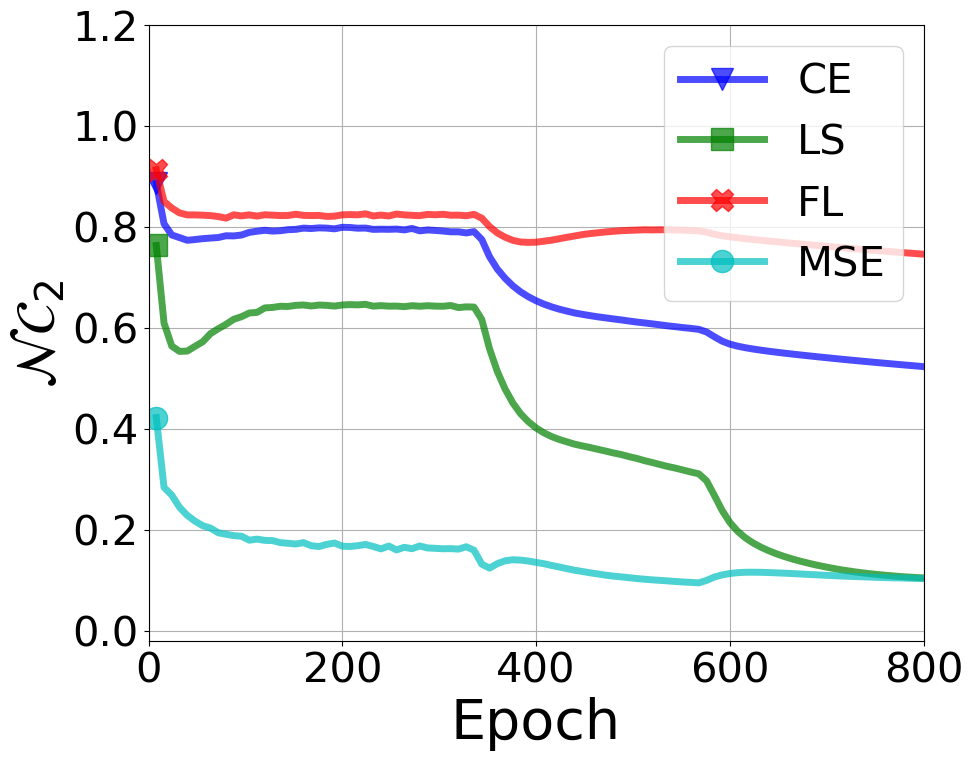}} \
    \subfloat[${\mc NC}_3$ (CIFAR100)]{\includegraphics[width=0.30\textwidth]{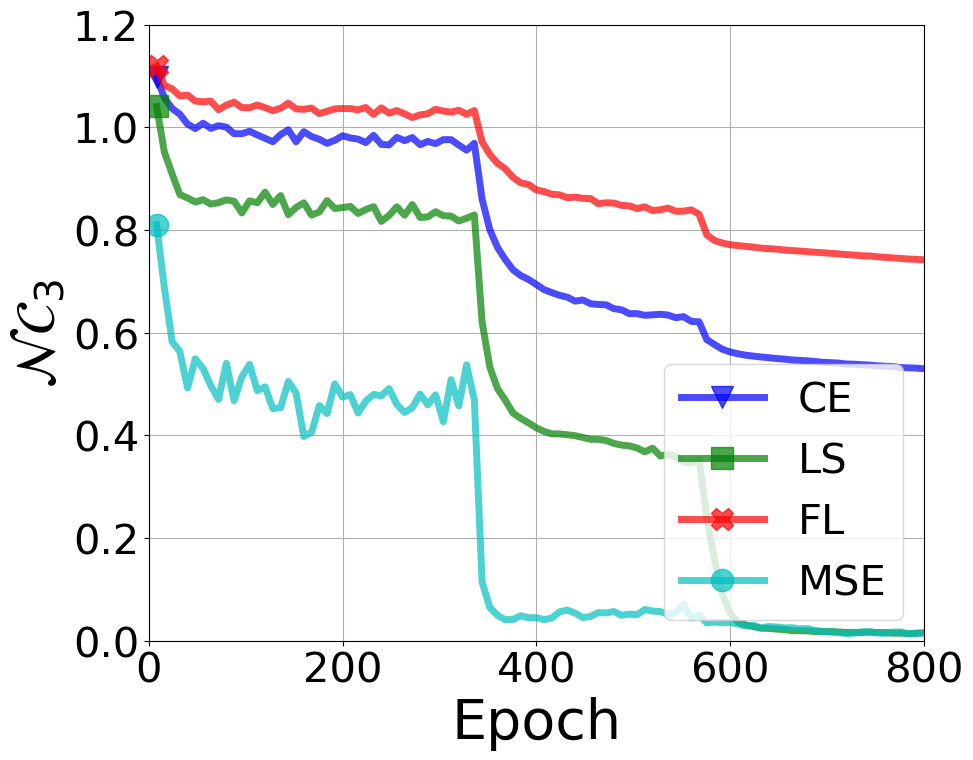}}
    \caption{\textbf{The evolution of \NC\ metrics across different loss functions.} We train the WideResNet50-2 on CIFAR100 dataset for 800 epochs using different loss function. From left to right: $NC_1$ (variability collapse), $NC_2$ (convergence to simplex ETF) and $NC_3$ (convergence to self-duality).
    }
    \label{fig:results-nc-epochs-cifar100}
\end{figure}

\paragraph{All Losses Lead to Largely Identical Performance}
Same as the results on CIFAR10 dataset, the conclusion on CIFAR100 also holds that all loss functions have largely identical performance once the training procedure converges to the \NC\ global optimality. In Figure \ref{fig:results-acc-epochs-cifar100}, we plot the evolution of the training accuracy, validation accuracy and test accuracy with training progressing, where all losses are optimized on the same WideResNet50-2 architecture and CIFAR100 for 800 epochs. 
To reduce the randomness, we average the results from 3 different random seeds per iteration-width configuration, and the test accuracy is reported based on the model with best accuracy on validation set, where we organize the validation set by holding out 10 percent data from the training set. 
The results consistently shows that the training accuracy trained by different losses all converge to one hundred percent (reaching to terminal phase), and the validation accuracy and test accuracy across different losses are largely same, as long as the optimization procedure converges to the \NC\ global solution. In Figure \ref{fig:lossmap-cifar100}, we plot the average \revise{${\mc NC}_1$ and} test accuracy of different losses under different pairs of width and iterations for CIFAR100 dataset. The three phenomenon mentioned in \Cref{exp:same-performance-across-losses} also exist on CIFAR100 in most cases. \revise{Moreover, the values of ${\mc NC}_1$ for $\text{width=0.25} $ and $\text{epochs=100}$ configuration are also around three orders magnitude larger than them for $\text{width=2} $ and $\text{epochs=800}$ configuration and the less collapsed configuration leads to larger difference gap across different loss functions.} While there are some small difference between different losses in $\text{width}=2$ and $\text{epochs}=800$ configurations, We guess that it is because CIFAR100 is much harder than CIFAR10 datasets, and network is not sufficiently large and trained not long enough for all losses to achieve a global solution. 

\begin{figure}[t]
    \centering
    \subfloat[Train Acc (CIFAR100)]{\includegraphics[width=0.30\textwidth]{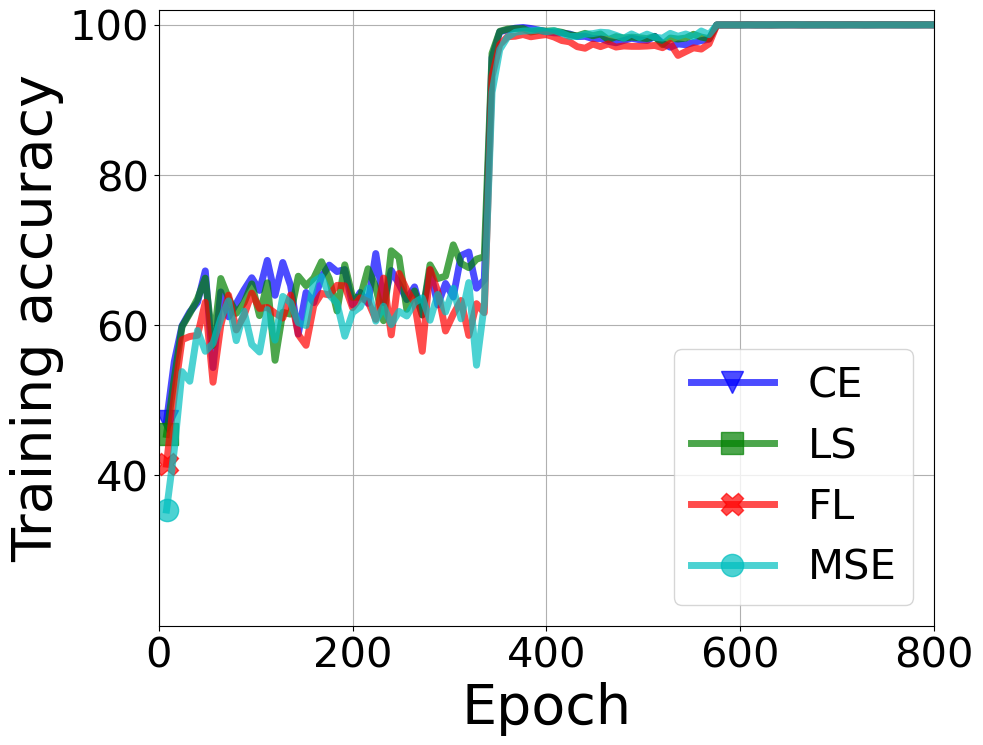}} \
    \subfloat[Val Acc (CIFAR100)]{\includegraphics[width=0.30\textwidth]{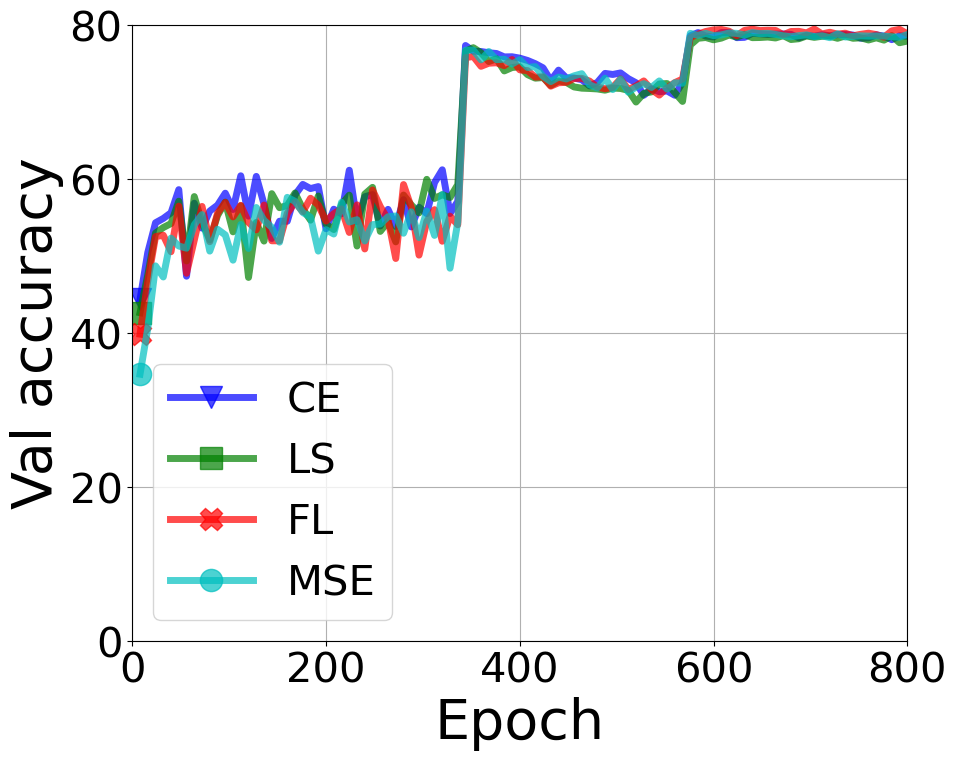}} \
    \subfloat[Test Acc (CIFAR100)]{\includegraphics[width=0.30\textwidth]{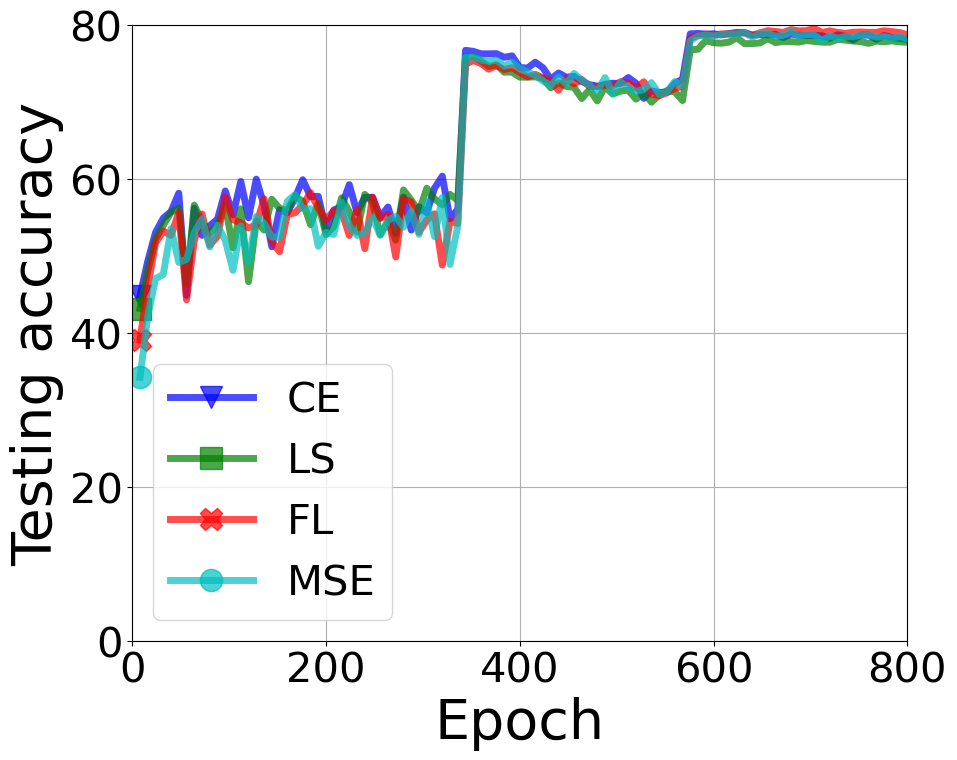}}
    \caption{\textbf{The evolution of performance across different loss functions.} We train the WideResNet50-2 on CIFAR100 dataset for 800 epochs using different loss function. From left to right: training accuracy, validation accuracy and test accuracy.
    }
    \label{fig:results-acc-epochs-cifar100}
\end{figure}

\begin{figure}[t]
    \centering
    \subfloat[${\mc NC}_1$ (CE)]{\includegraphics[width=0.23\textwidth]{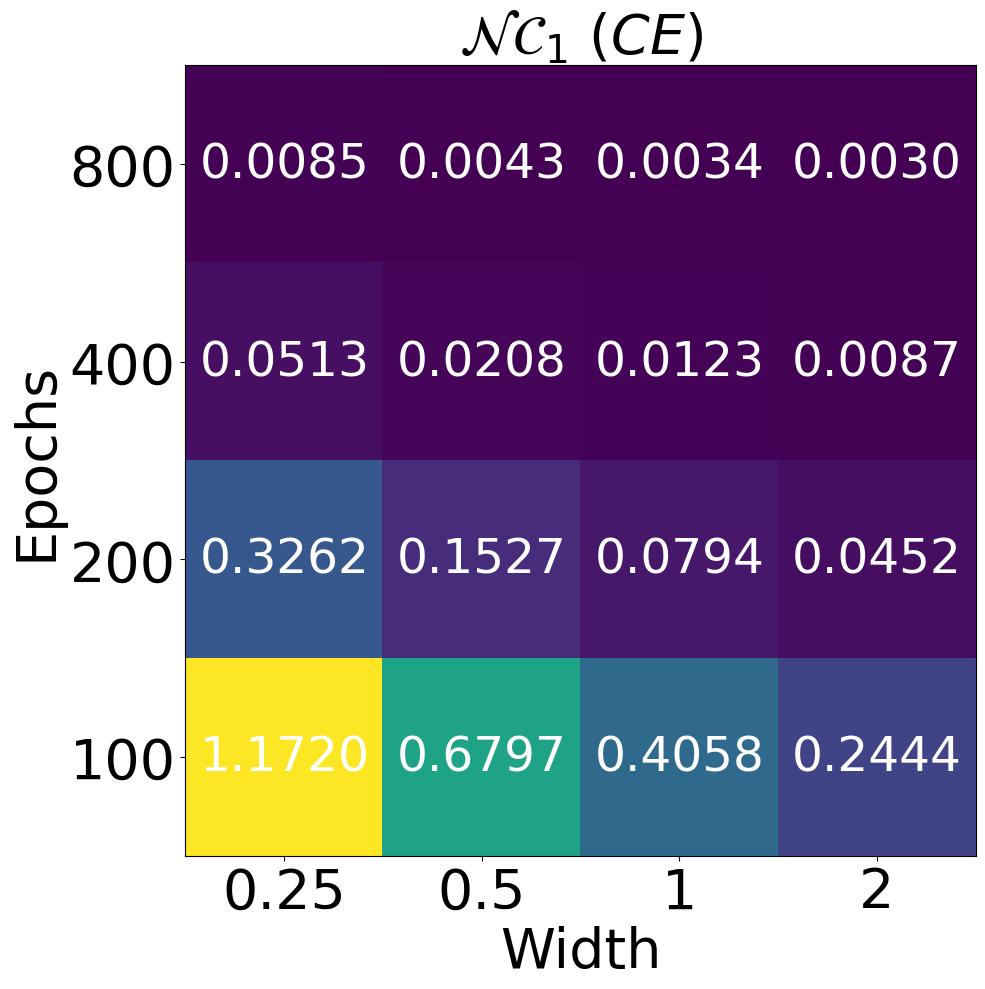}} \
    \subfloat[${\mc NC}_1$ (MSE)]{\includegraphics[width=0.23\textwidth]{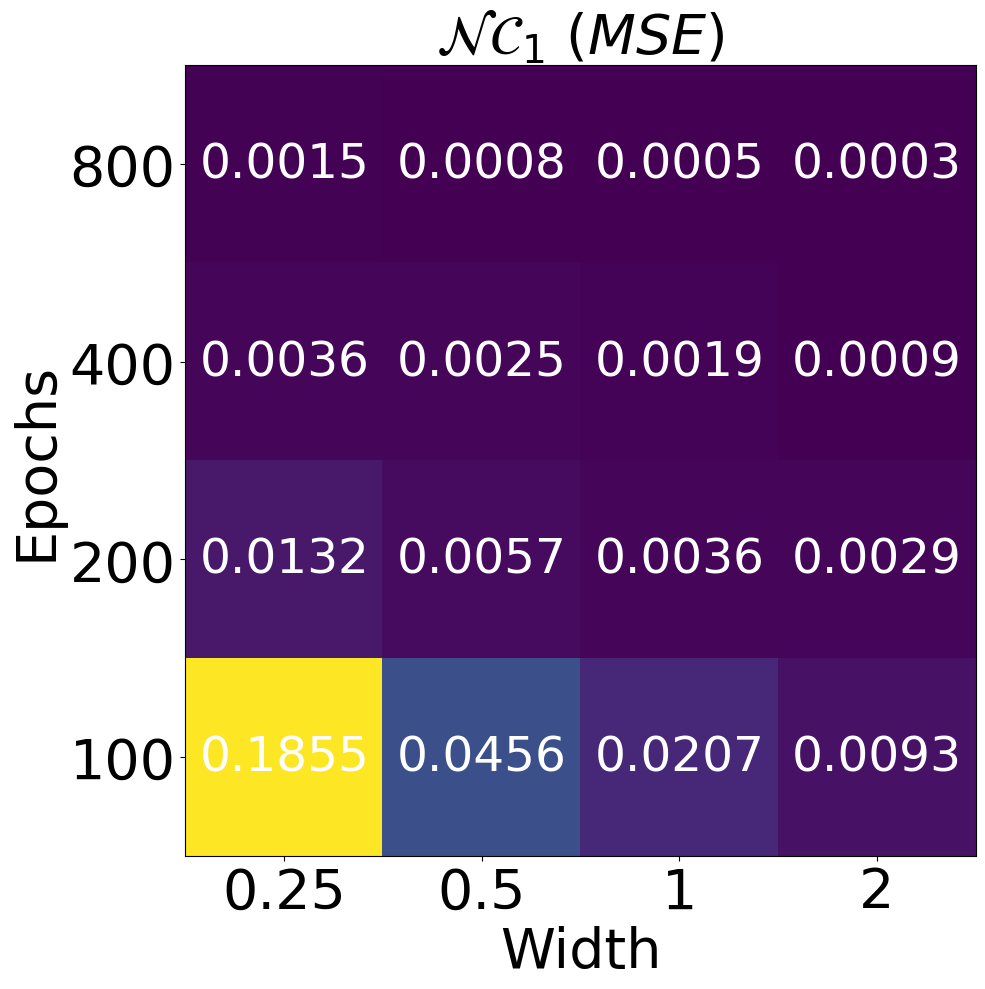}} \
    \subfloat[${\mc NC}_1$ (FL)]{\includegraphics[width=0.23\textwidth]{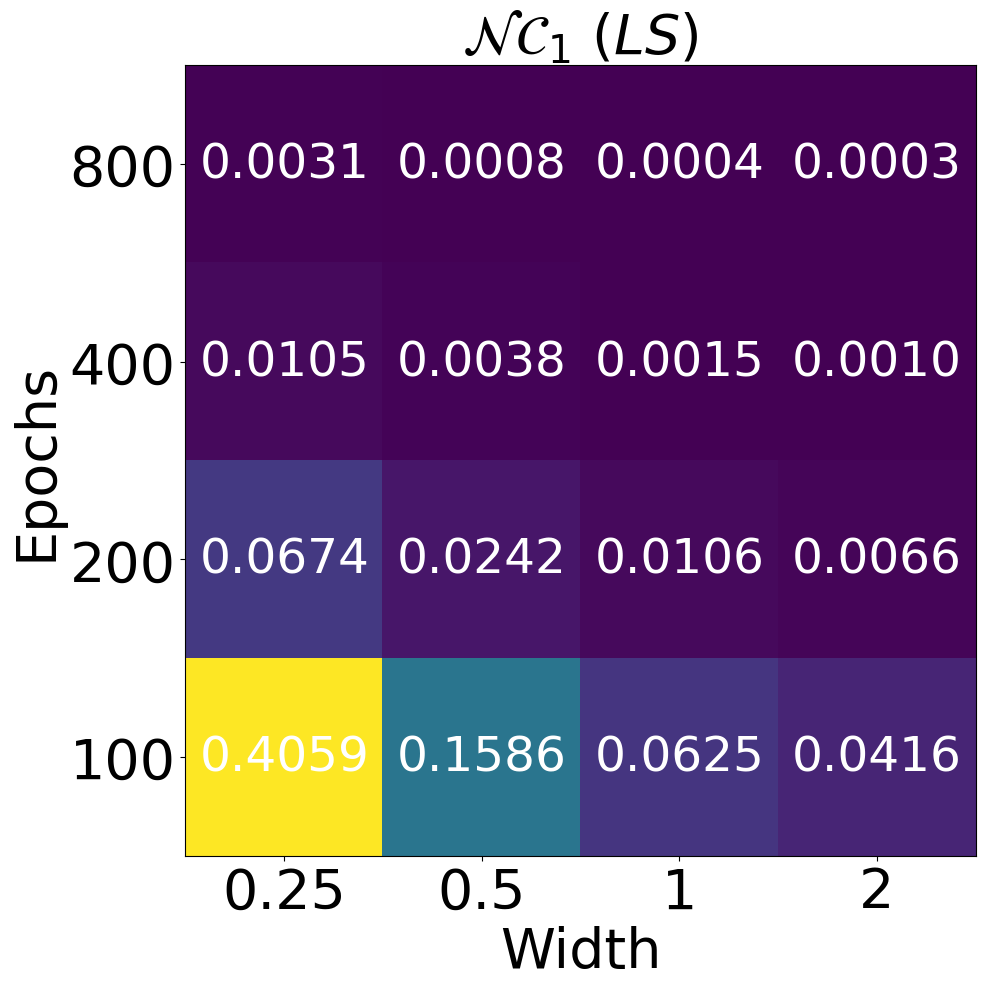}} \
    \subfloat[${\mc NC}_1$ (LS)]{\includegraphics[width=0.23\textwidth]{figures/res50-cifar100-loss-diff/nc1-ls.png}}\\
    
    \subfloat[Test (CE)]{\includegraphics[width=0.23\textwidth]{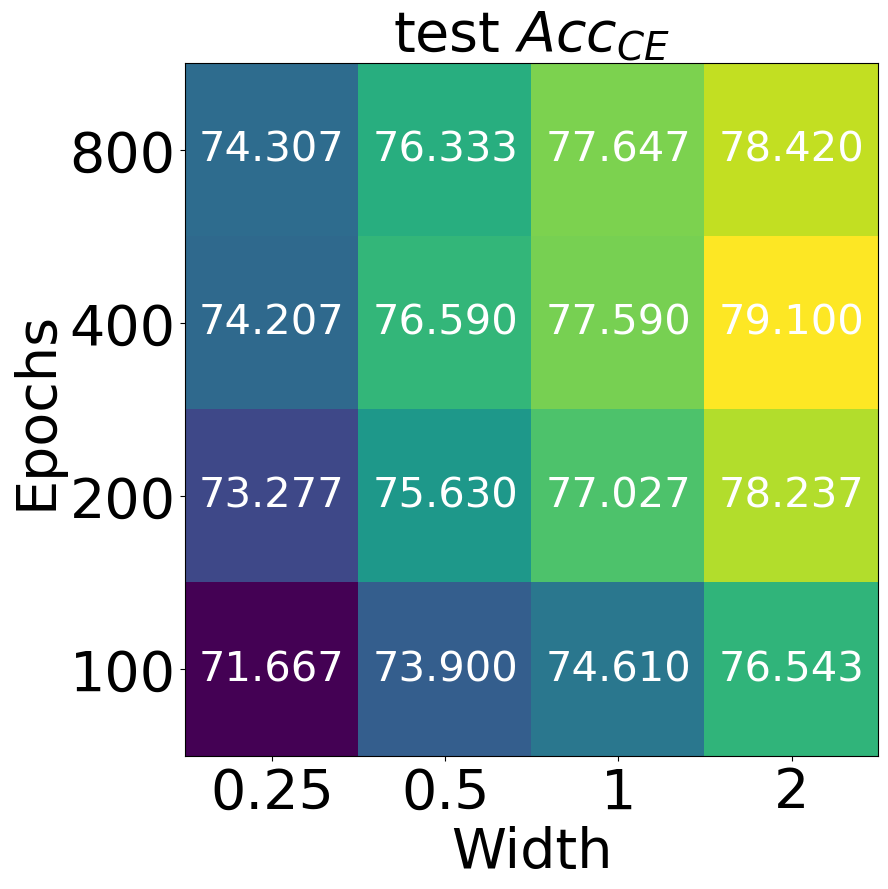}} \
    \subfloat[Test (MSE)]{\includegraphics[width=0.23\textwidth]{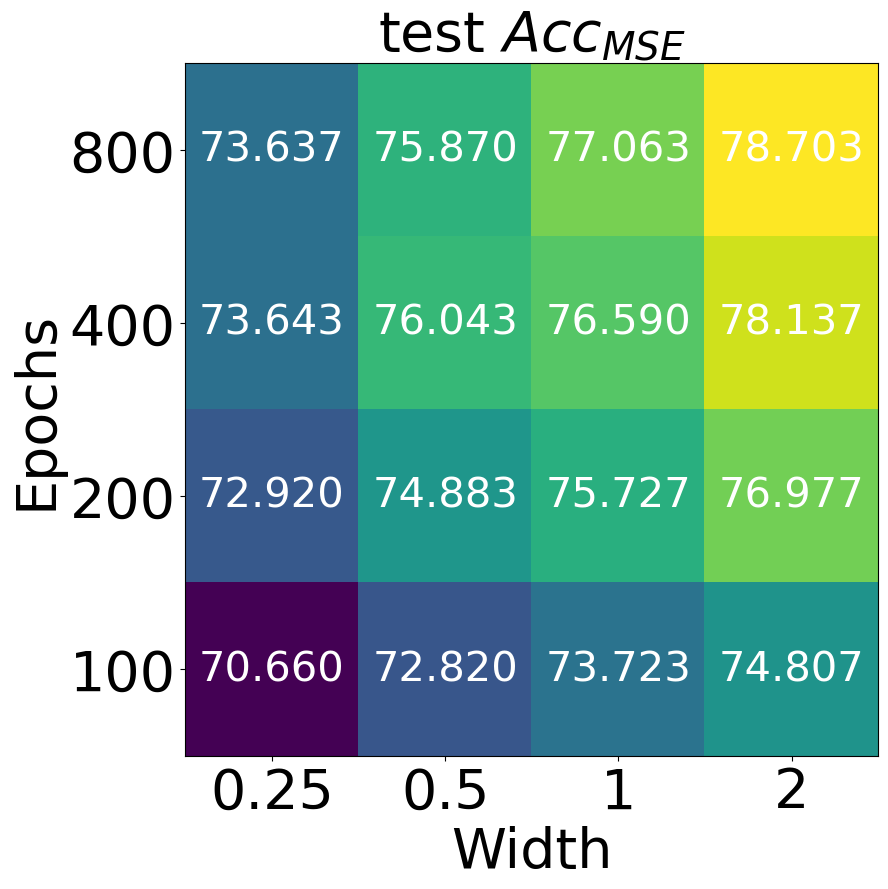}} \
    \subfloat[Test (FL)]{\includegraphics[width=0.23\textwidth]{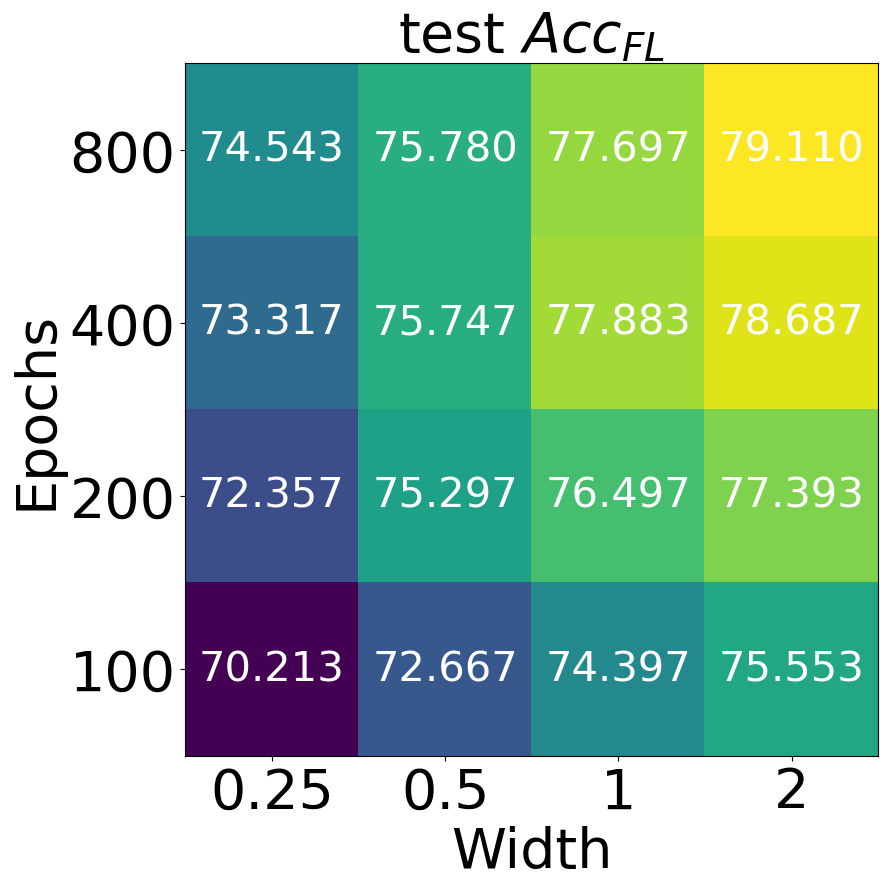}} \
    \subfloat[Test (LS)]{\includegraphics[width=0.23\textwidth]{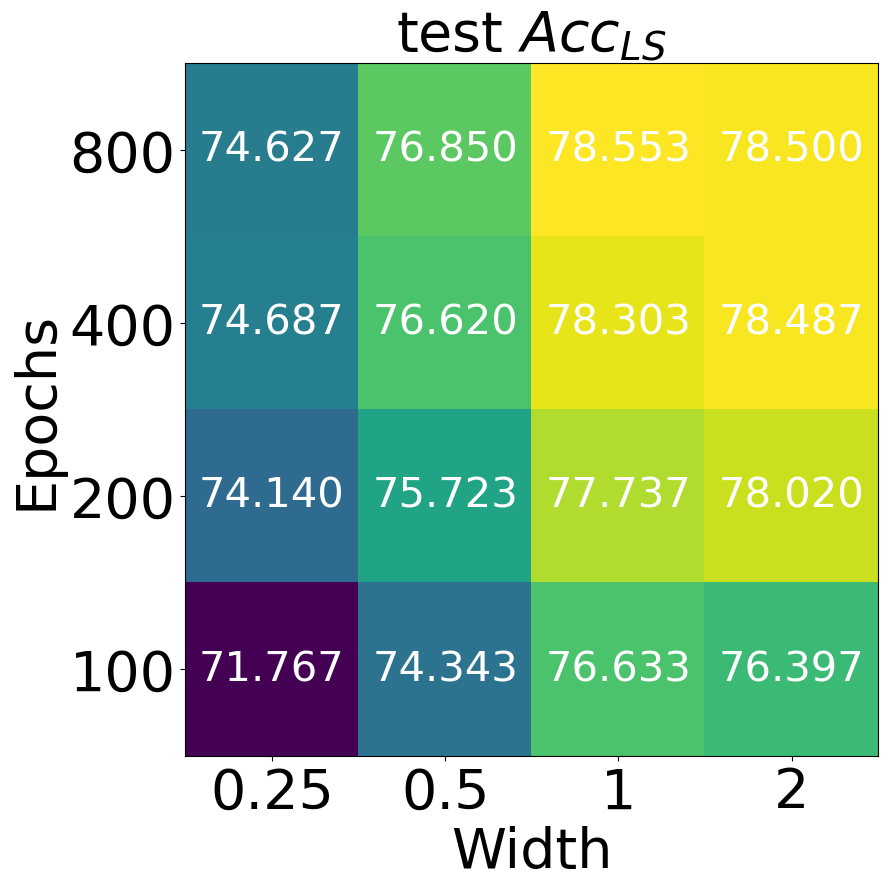}}
    \caption{\textbf{Illustration of \revise{${\mc NC}_1$ and} test accuracy across different iterations-width configurations.} The figure depicts the \revise{${\mc NC}_1$} and test accuracy of various iteration-width configurations for different loss functions on CIFAR100. 
    }
    \label{fig:lossmap-cifar100}
\end{figure}

\subsection{\revise{Additional experimental results on miniImageNet}}\label{subsec:mini-experiments}
\revise{In this parts, we show the additional results on miniImageNet dataset. We trained WideResNet18-0.25 and WideResNet18-2 on miniImageNet for 100 epochs and 800 epochs, respectively. To reduce the randomness, we average the results from 3 different random trials. The ${\mc NC}_1$ and test accuracy of different loss functions are provided in Figure \ref{fig:results-acc-epochs-mini} for comparison. We consistently observe that the ${\mc NC}_1$ metric of all losses converges to a small value as training progress, when the neural network has sufficient approximation power and the training is performed for sufficiently many iterations, such as WideResNet18-2 for 800 epochs. Additionally, the conclusion on miniImageNet also holds that all loss functions have largely identical performance once the training procedure converges to the \NC\ global optimality. Specifically, while the last-iteration test accuracy of training WideResNet18-0.25 for 100 epochs is $0.7195$, $0.6915$, $0.7020$ and $0.7040$, respectively, the last-iteration test accuracy of training WideResNet18-2 for 800 epochs is $0.7930$, $0.7962$, $0.7932$ and $0.8020$ for CE, MSE, FL and LS, respectively. The experiment results on miniImageNet also support our claim that $(i)$ the test performance may be different across different loss functions when the network is not large enough and is optimized with limited number of iterations, but $(ii)$ the test accuracy across different loss are largely identical, once the networks has sufficient capacity and the training is optimized to converge to the \NC\ global solution.}  

\begin{figure}[t]
    \centering
    \subfloat[\revise{${\mc NC}_1$ and Test (CE)}]{\includegraphics[width=0.24\textwidth]{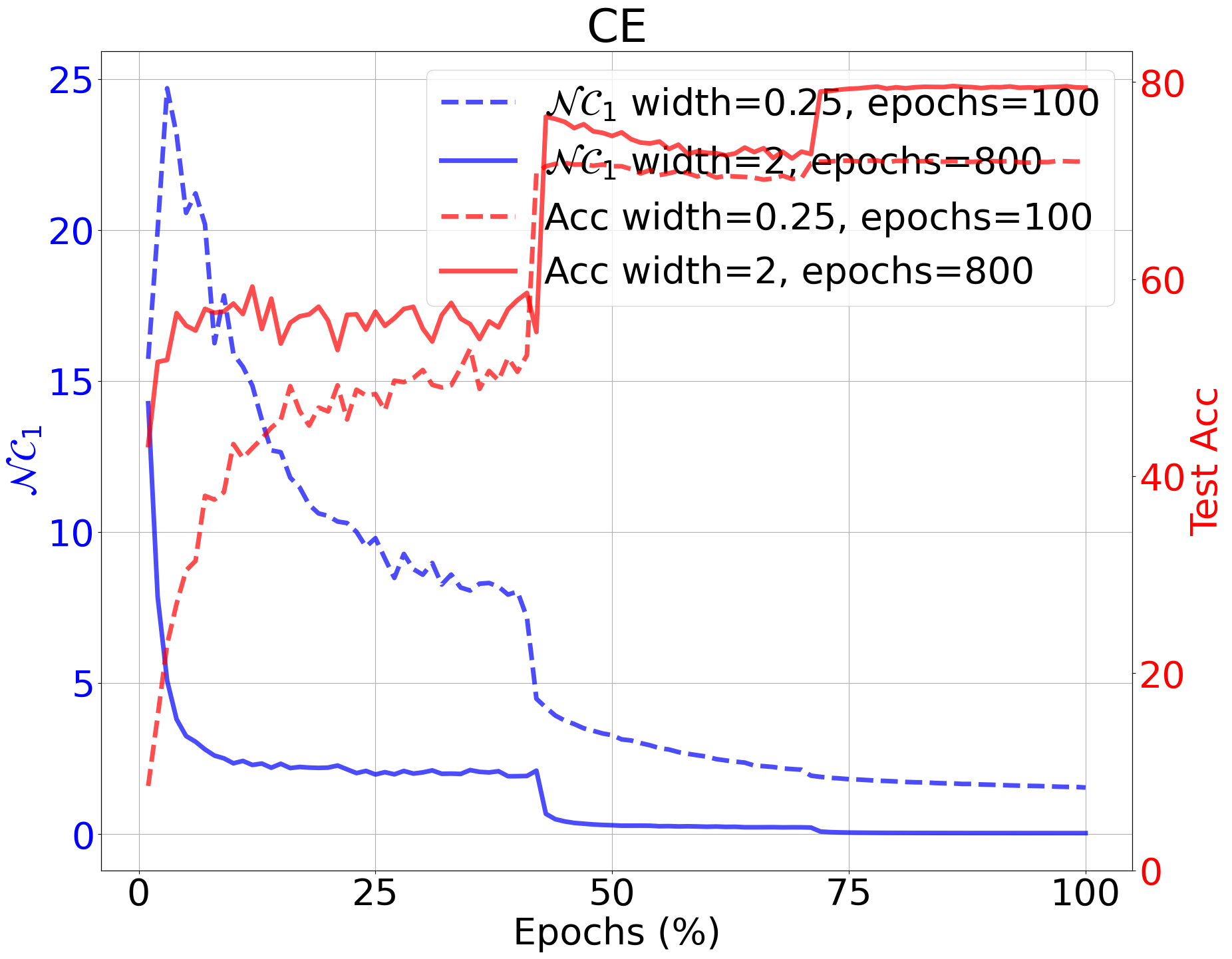}} \
    \subfloat[\revise{${\mc NC}_1$ and Test (MSE)}]{\includegraphics[width=0.24\textwidth]{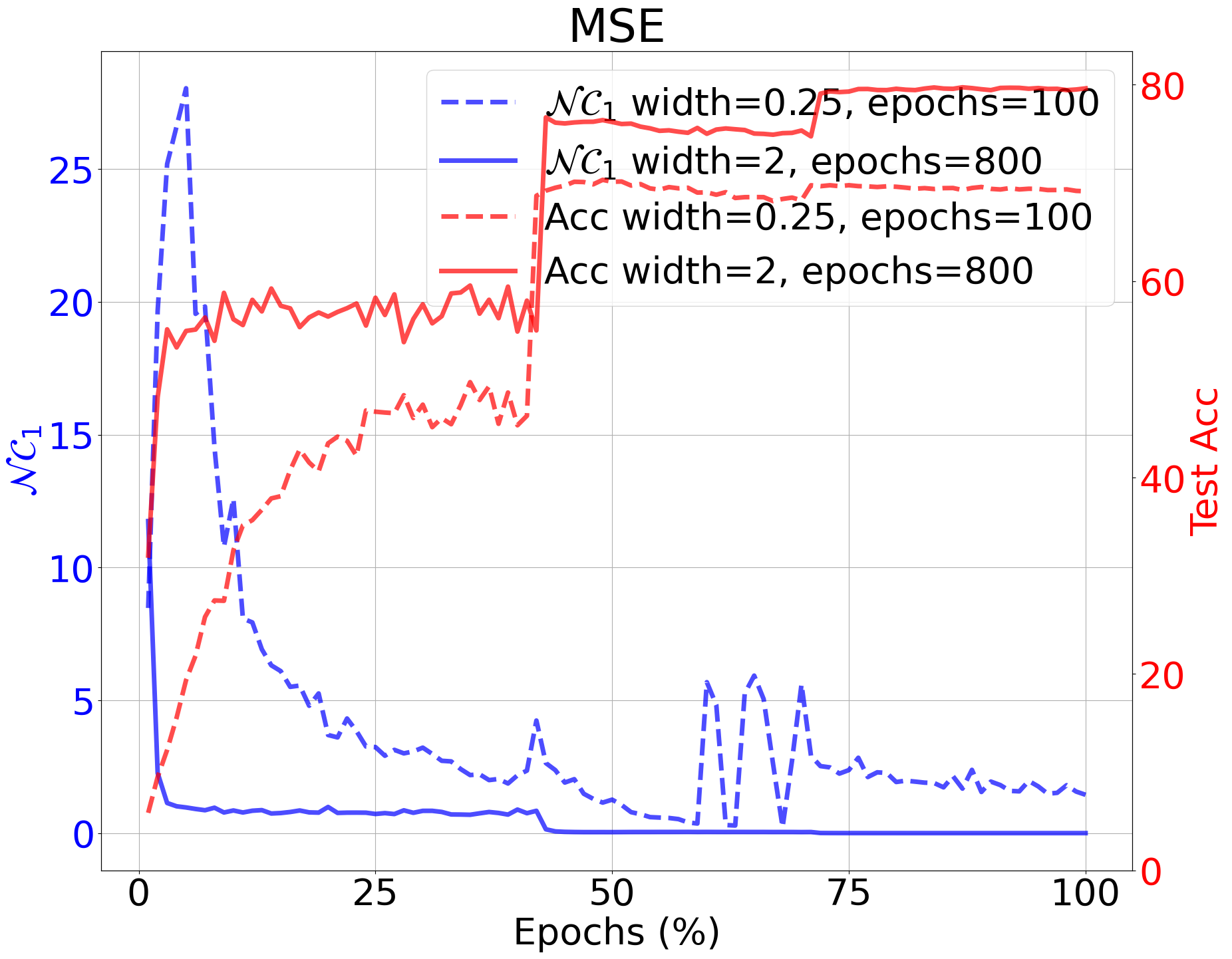}} \
    \subfloat[\revise{${\mc NC}_1$ and Test (FL)}]{\includegraphics[width=0.24\textwidth]{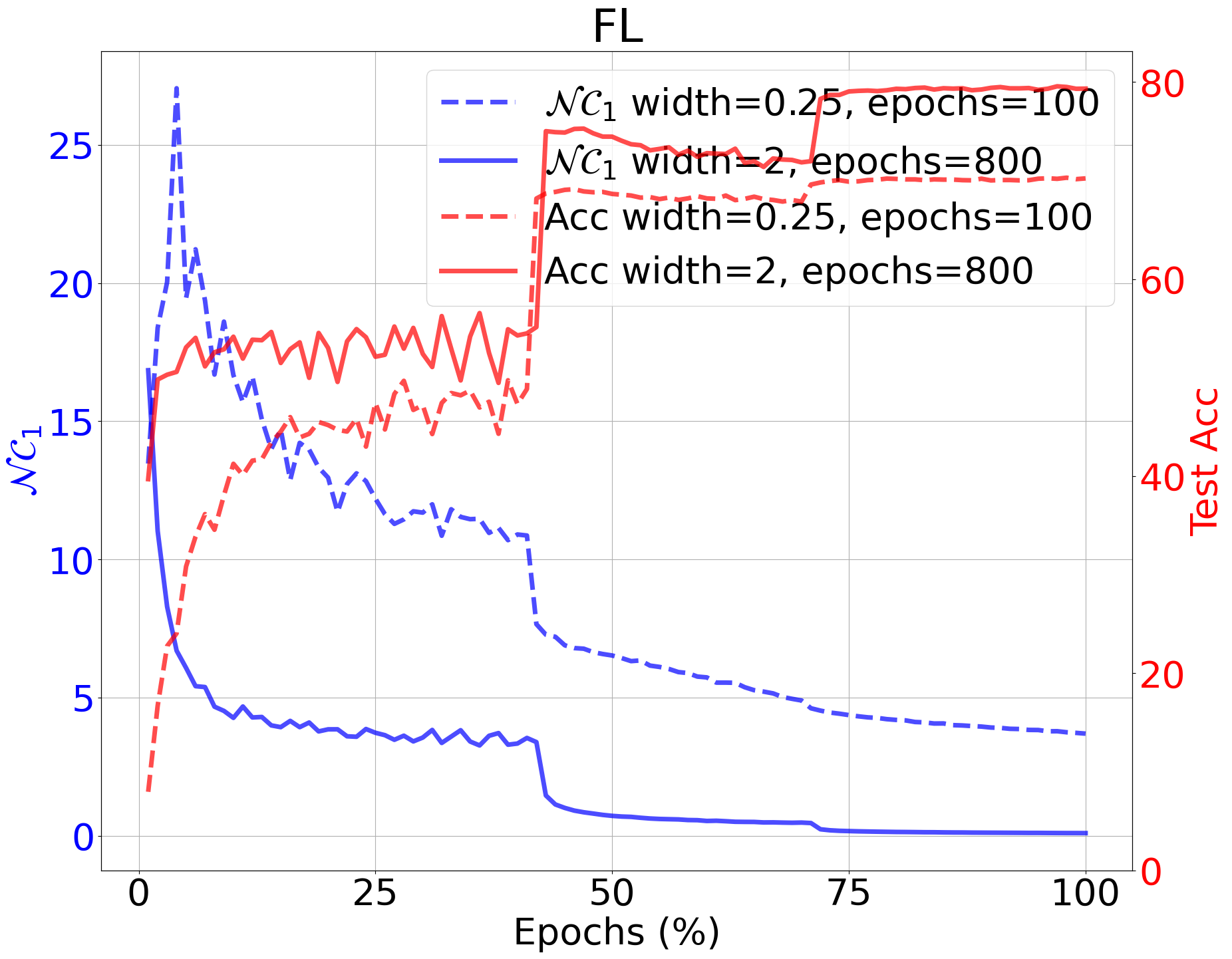}} \
    \subfloat[\revise{${\mc NC}_1$ and Test (LS)}]{\includegraphics[width=0.24\textwidth]{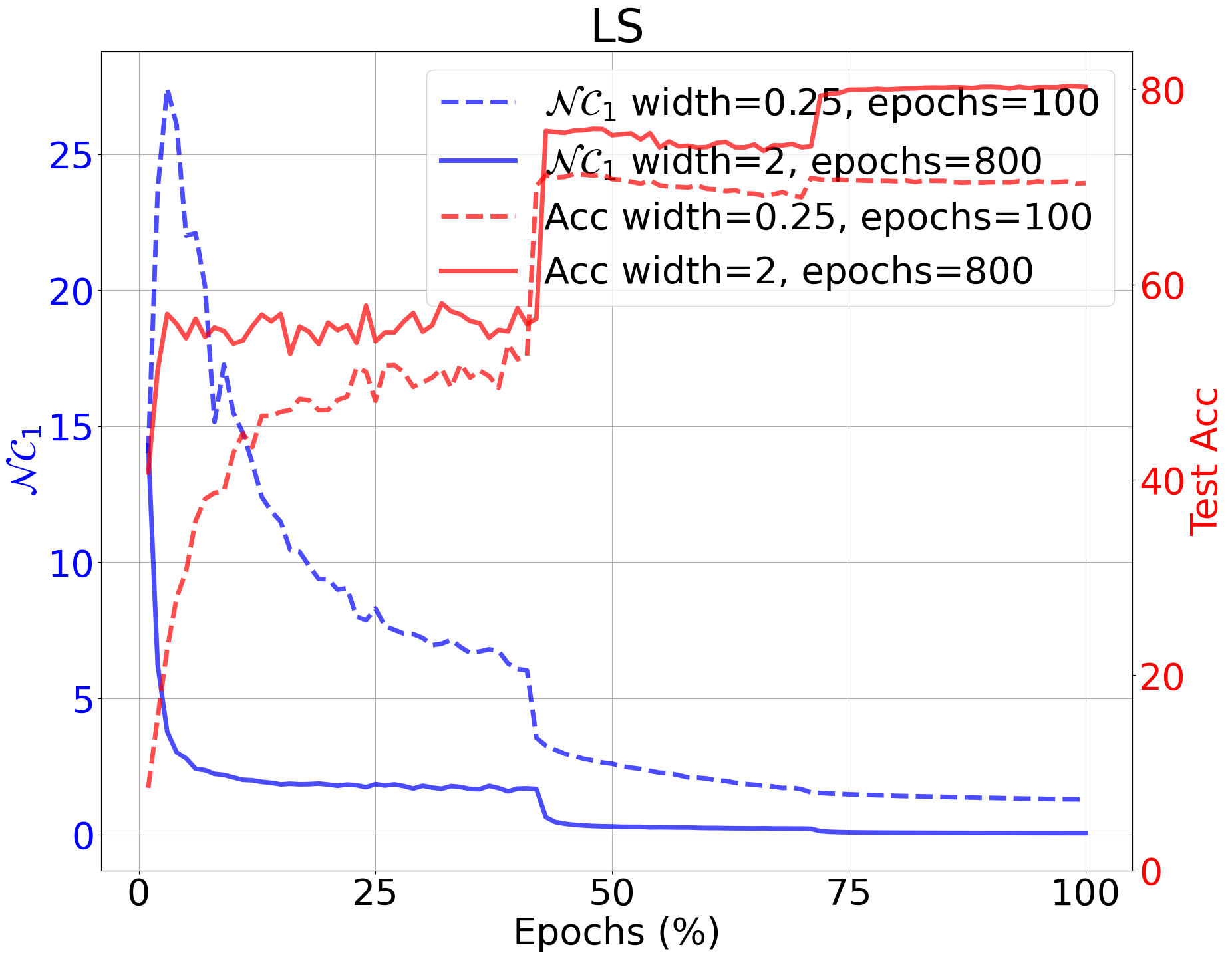}}
    \caption{\textbf{The evolution of \revise{${\mc NC}_1$ and test accuracy} across different loss functions.} We train the \revise{WideResNet18-0.25 for 100 epochs and WideResNet18-2 for 800 epochs on miniImageNet} using different loss functions. 
    }
    \label{fig:results-acc-epochs-mini}
\end{figure}

\section{Proof of CE, FL and LS included in GL}\label{app:loss-in-GL}
In this section, we prove that CE, FL and LS belong to GL in \Cref{app:check-ce-in-gl}, \Cref{app:check-fl-in-gl} and \Cref{app:check-ls-in-gl}, respectively. Before starting the proof for each loss, let us restate the definition of the GL in \Cref{def:GLoss}:

\begin{definition}[Contrastive property] We say a loss function $\Lgl (\vz,\vy_k)$ satisfies the contrastive property if 
there exists a function $\phi$ such that $\Lgl (\vz,\vy_k)$ can be lower bounded by
\begin{align}\label{gl-property-1-app}
    &\Lgl (\vz,\vy_k) \geq \phi\paren{\sum_{j\neq k}(z_j-z_k)}
    \end{align}
where the equality holds only when $z_j=z_j'$ for all $j,j' \neq k$. Moreover,  $\phi(t)$ satisfies
\begin{align}
    t^*&= \arg\min_{t}\phi\paren{t}+c|t| \text{ is unique for any } c>0, \text{and } t^*\leq 0.\label{gl-property-3-app}
\end{align}
\label{def:GLoss-app}
\end{definition}

\subsection{CE is in GL}\label{app:check-ce-in-gl}
In this section, we will show that the CE defined in \eqref{def:ce} belongs to the GL defined in \Cref{def:GLoss-app}. First, let us rewrite the CE definition in GL form as following:
\begin{align*}
    \Lce (\vz,\vy_k) \;&=\;- \log\paren{\frac{\exp(z_{k})}{\sum_{j=1}^K\exp(z_{j})}}
    \;=\; \log\paren{1+\sum_{j\neq k}^K\exp(z_{j}-z_{k})}\\
    \;&\geq\;\log\paren{1+(K-1)\exp{\paren{\frac{z_{j}-z_{k}}{K-1}}}}=\phi_{\text{CE}}\paren{\sum_{j\neq k}(z_j-z_k)}.
\end{align*}
where the inequality is due to the $\log$ is an increasing and function and $\exp$ is a strictly convex function, and it achieves equality only when $z_{j}=z_{j'}$ for all $j,j' \neq k$. Therefore, there exists such a function $\phi_{\text{CE}}$ to lower bound original CE loss $\Lce (\vz,\vy_k)$ as following:
\begin{align*}
    \phi_{\text{CE}}(t) \;&=\;\log\paren{1+(K-1)\exp{\paren{\frac{t}{K-1}}}},
\end{align*}
which satisfies the condition of \eqref{gl-property-1-app}. Next, we will show $\phi_{\text{CE}}(t)$ satisfies the condition \eqref{gl-property-3-app}. The first-order gradient of $\phi_{\text{CE}}(t)$ is following:
\begin{align*}
    \nabla\phi_{\text{CE}}(t) &= \frac{\exp{\paren{\frac{t}{K-1}}}}{1+(K-1)\exp{\paren{\frac{t}{K-1}}}}
\end{align*}
which is an increasing function and greater than $0$ for $t\in\R$. Let denote $\psi_{\text{CE}}(t)=\phi_{\text{CE}}(t)+c|t|$, then 
\begin{itemize}
    \item{\textbf{When $t\geq0$:}} $\nabla\psi_{\text{CE}}(t)=\nabla\phi_{\text{CE}}(t)+c>0$, thus the $\psi_{\text{CE}}(t)$ is an increasing function w.r.t. $t$, and the minimizer is achieved when $t=0$.
    \item{\textbf{When $t\leq0$:}} $\nabla\psi_{\text{CE}}(t)=\nabla\phi_{\text{CE}}(t)-c$, and $\nabla \phi_{\text{CE}}(t)$ is an increasing function, which achieves minimizer when $t=0$ such that $\nabla\phi_{\text{CE}}(t)=\frac{1}{K}$.
    \begin{itemize}
        \item{if $c\geq\frac{1}{K}$}, $\nabla\psi_{\text{CE}}(t)<0$, and $\psi(t)$ is a decreasing function for $t\leq0$, and the minimizer is achieved when $t=0$;
        \item{if $0<c\leq\frac{1}{K}$}, there exist such $t^*$ such that $\nabla\psi_{\text{CE}}(t)=0$. When $t<t^*$, $\phi_{\text{CE}}(t)$ is a decreasing function; and when $t^* < t\leq0$, $\phi_{\text{CE}}(t)$ is an increasing function. Therefore, the minimizer is achieved when $t=t^*<0$
    \end{itemize}
\end{itemize}
Combing them together, we can prove that $\phi_{\text{CE}}$ satisfies the condition of \eqref{gl-property-3-app}.

\subsection{FL is in GL}\label{app:check-fl-in-gl}
In this section, we will show that the FL defined in \eqref{def:fl} belongs to the GL defined in \Cref{def:GLoss-app}. let us rewrite the FL definition in GL form as following:
\begin{align*}
    \Lfl (\vz,\vy_k) \;&=\;-\paren{1-\frac{\exp(z_{k})}{\sum_{j=1}^K\exp(z_{j})}}^\gamma \log\paren{\frac{\exp(z_{k})}{\sum_{j=1}^K\exp(z_{j})}}\\
    \;&=\; \paren{1-\frac{\exp(z_{k})}{\sum_{j=1}^K\exp(z_{j})}}^\gamma \log\paren{\sum_{j=1}^K\exp(z_{j}-z_{k})}\\
    \;&=\;\paren{1-\frac{1}{1+\sum_{j\neq k}^K\exp(z_{j}-z_{k})}}^\gamma \log\paren{1+\sum_{j\neq k}^K\exp(z_{j}-z_{k})}\\
    \;&=\; \eta\paren{1+\sum_{j\neq k}^K\exp(z_{j}-z_{k})}
\end{align*}
where the function $\eta(t)=(1-\frac{1}{t})^\gamma\log\paren{t}$ is an increasing function for $t\geq 1$ because
\begin{align*}
    \nabla \eta(t) &= \gamma(\frac{1}{t^2})(1-\frac{1}{t})^{\gamma-1}\log(t)+\frac{1}{t}(1-\frac{1}{t})^{\gamma}>0
\end{align*}
Thus, we can find the lower bound function by
\begin{align*}
    \Lfl (\vz,\vy_k) \;&\geq\; \eta\paren{1+(K-1)\exp\paren{\sum_{j\neq k}^K\frac{z_{j}-z_{k}}{K-1}}}\\
    \;&=\; \eta\paren{\xi\paren{\sum_{j\neq k}^K(z_{j}-z_{k})}}\\
    \;&=\;\phi_{\text{FL}}\paren{\sum_{j\neq k}^K(z_{j}-z_{k})}
\end{align*}
where $\phi_{\text{FL}}(t)=\eta\paren{\xi\paren{t}}$ and $\xi{\paren{t}}=1+(K-1)\exp{\frac{t}{K-1}}\in [1,K]$, which satisfies the condition of \eqref{gl-property-1-app}. Next, we will show $\phi_{\text{FL}}(t)$ satisfies the condition \eqref{gl-property-3-app}. The first-order gradient of $\phi_{\text{FL}}(t)$ is following:
\begin{align*}
    &\nabla_t \psi_{\text{FL}}(t) \;=\;\nabla_t \paren{\phi_{\text{FL}}(t)+c|t|} \;=\;\nabla_{\xi\paren{t}}\eta\paren{\xi\paren{t}}\nabla_t{\xi\paren{t}}+c\frac{t}{|t|}\\
    \;=\;&\paren{\gamma\paren{\frac{1}{\xi\paren{t}}}^2\paren{1-\frac{1}{\xi\paren{t}}}^{\gamma-1}\log\paren{\xi\paren{t}}+\frac{1}{\xi\paren{t}}\paren{1-\frac{1}{\xi\paren{t}}}^{\gamma}}\paren{\exp{\paren{\frac{t}{K-1}}}}+c\frac{t}{|t|}\\
    \;=\;&\paren{\gamma\paren{\frac{1}{\xi\paren{t}}}^2\paren{1-\frac{1}{\xi\paren{t}}}^{\gamma-1}\log\paren{\xi\paren{t}}+\frac{1}{\xi\paren{t}}\paren{1-\frac{1}{\xi\paren{t}}}^{\gamma}}\paren{\frac{\xi\paren{t}-1}{K-1}}+c\frac{t}{|t|}\\
    \;=\;&\frac{1}{K-1}\underbrace{\frac{\paren{\xi\paren{t}-1}^\gamma}{{\xi\paren{t}}^{\gamma+1}}\paren{\xi\paren{t}-1+\gamma\log\paren{\xi\paren{t}}}}_{\varsigma(\xi(t))\geq 0}+c\frac{t}{|t|}
\end{align*}
Similarly, by chain rule, the second-order derivation is:
\begin{align*}
    &\nabla^2_t \psi(t) = \nabla^2_t \phi(t)=\nabla_{\xi(t)} \varsigma\paren{\xi(t)} \nabla_t (t)\\
    =& (\gamma+1)\frac{1}{\paren{\xi(t)}^2}(1-\frac{1}{\xi(t)})^\gamma\\&-\frac{\gamma}{\paren{\xi(t)}^2}(1-\frac{1}{\xi(t)})^\gamma\paren{\log(\xi(t))-\gamma\frac{\log\paren{\xi(t)}}{\xi(t)-1}-\gamma}\paren{\frac{1}{(K-1)^2}(\xi(t)-1)}\\
    =& \frac{1}{(K-1)^2}\frac{\gamma(\xi(t)-1)^{\gamma+1}}{\paren{\xi(t)}^{\gamma+2}}\paren{\underbrace{-\log(\xi(t))+\gamma\frac{\log(\xi(t))}{\xi(t)-1}+\gamma+\frac{\gamma+1}{\gamma}}_{\vartheta(\xi(t))}}
\end{align*}

\begin{itemize}
    \item {When $t\geq 0$: }$\nabla_t \psi_{\text{FL}}(t)=\frac{1}{K-1}\xi_(t)+c\geq 0$, thus the $\psi_{\text{CE}}(t)$ is an increasing function w.r.t. $t$, and the minimizer is achieved when $x=0$.
    \item{When $t\leq 0$: }$\nabla_t \psi_{\text{FL}}(t)=\frac{1}{K-1}\xi(t)-c\geq 0$. Moreover, we can find $\vartheta(\xi(t))$ is a decreasing function w.r.t. $\xi(t)$ and $\xi(t)$ is an increasing function w.r.t. $t$, therefore, $\vartheta(\xi(t))$ is a decreasing function w.r.t. $t$. 
    \begin{itemize}
        \item {If $\vartheta(\xi(0))=\vartheta(K)\geq 0$}, then $\nabla_{x}^2 \psi(x)>0$ for $x\leq 0$, which means that $\nabla_{x} \xi(t)$ is an increasing function. Because $\varsigma(\xi(-\infty))=\varsigma(1)=0$, here we need to consider two cases(Please refer to Figure \ref{fig:fl_case_1}):
        \begin{itemize}
            \item {if $ {\varsigma(\xi(0)} = \varsigma(K)\leq c(K-1)$}, then $\nabla_t \psi_{FL}(t)\geq 0$, that is, $\psi_{FL}(t)$ is a decreasing function. Therefore, the global minimizer is achieved when $x=0$ (the blue curve in Figure \ref{fig:fl_case_1}).
            \item {if $ {\varsigma(\xi(0)} = \varsigma(K)\geq c(K-1)$}, so $\psi_{FL
            }(x)$ will first decrease and then increase. Therefore the global minimizer is unique (the red curve in Figure \ref{fig:fl_case_1}).
        \end{itemize}
    \item{If $\vartheta(\xi(0))=\vartheta(K)< 0$}, then for $t\in[-\infty, t']$, $\nabla_t \psi_{FL} (x)$ is an increasing function w.r.t. $t$; for $t\in[t', 0)$, $\nabla_t \Phi_{FL} (t)$ is a decreasing function w.r.t. $t$. Here we need to consider three cases(please refer to Figure \ref{fig:fl_case_2}):
        \begin{itemize}
            \item {if ${\varsigma(\xi(t'))}\leq c(K-1)$}, then $\nabla_t \psi_{FL}(t)\leq 0$, that is, $\psi_{FL}(t)$ is a decreasing function. Therefore, the global minimizer is achieved when $x=0$ (the green curve in Figure \ref{fig:fl_case_2}).
    
            \item {if ${\varsigma(\xi(0))}={\varsigma(K)}\geq c(K-1)$}, so $\psi_{FL}(x)$ will first decrease and then increase. Therefore the global minimizer is unique (the red curve in Figure \ref{fig:fl_case_2}).
    
            \item {if ${\varsigma(\xi(t'))}\geq c(K-1)$ and ${\varsigma(\xi(0))}={\varsigma(K)}\leq c(K-1)$}, then $\nabla_t \psi_{FL}(t)=0$ has two solutions $t_1$ and $t_2$. For $t \in [-\infty,t_1]$, $\psi_{FL}(t)$ is an decreasing function w.r.t. $t$; for $t \in [t_1, t_2]$, $\Phi_{FL}(t)$ is an increasing function w.r.t. $t$; and for $t \in [t_2, 0)$, $\psi_{FL}(t)$ is a decreasing function w.r.t. $t$. The unique minimizer is achieved when either $t=0$ or $t=t_1$, as long as $\psi_{FL}(0)\neq \psi_{FL}(t_1)$. As for the minor case $\psi_{FL}(0)= \psi_{FL}(t_1)$, it requires carefully chosen penalized parameters, which can be omitted (the blue curve in Figure \ref{fig:fl_case_2}).
        \end{itemize}
    \end{itemize}
\end{itemize}

In conclusion, for focal loss, $\psi_{FL}(t)$ has a unique minimum in terms of $t\leq 0$, which satisfies the condition of \eqref{gl-property-3-app}.

\begin{figure}[t]
    \centering
    \subfloat[$\varsigma(\xi(t))$ w.r.t. $\xi(t)$]{\includegraphics[width=0.4\textwidth]{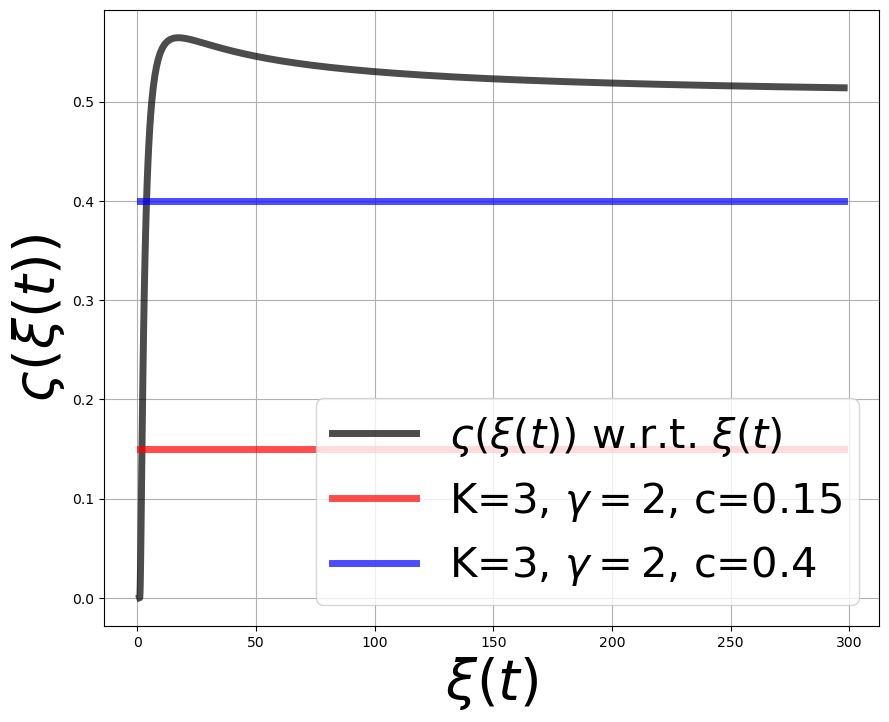}} \
    \subfloat[$\psi(t)$ w.r.t. $t$]{\includegraphics[width=0.4\textwidth]{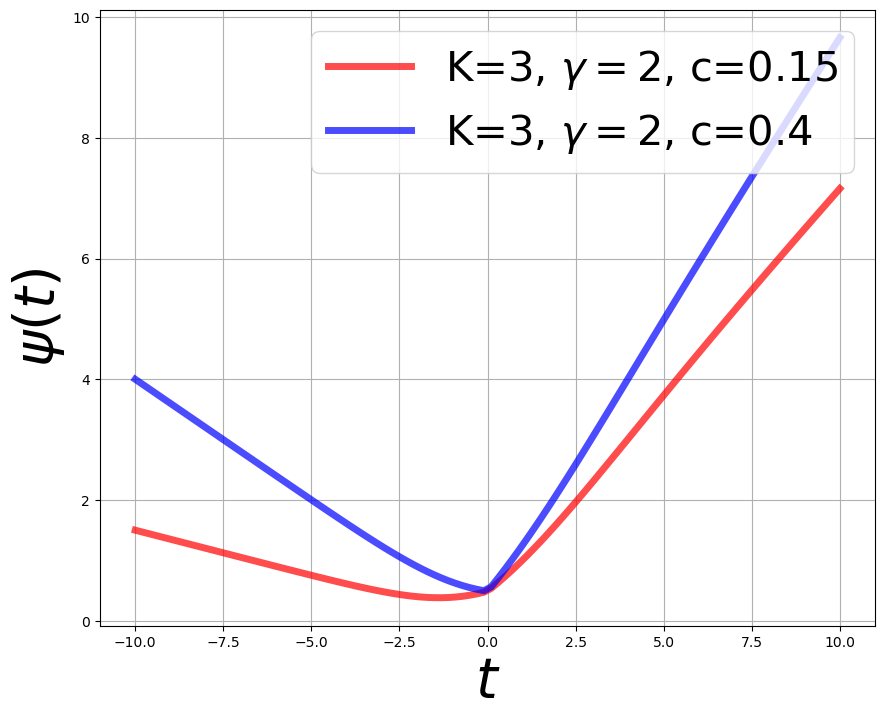}} \
    \caption{\textbf{Illustration of the case of $\vartheta(\xi(0))\geq 0$, where $c=-K\sqrt{n\lambda_{\mb W} \lambda_{\mb H}}$}.
    }
    \label{fig:fl_case_1}
\end{figure}

\begin{figure}[t]
    \centering
    \subfloat[$\varsigma(\xi(t))$ w.r.t. $\xi(t)$]{\includegraphics[width=0.4\textwidth]{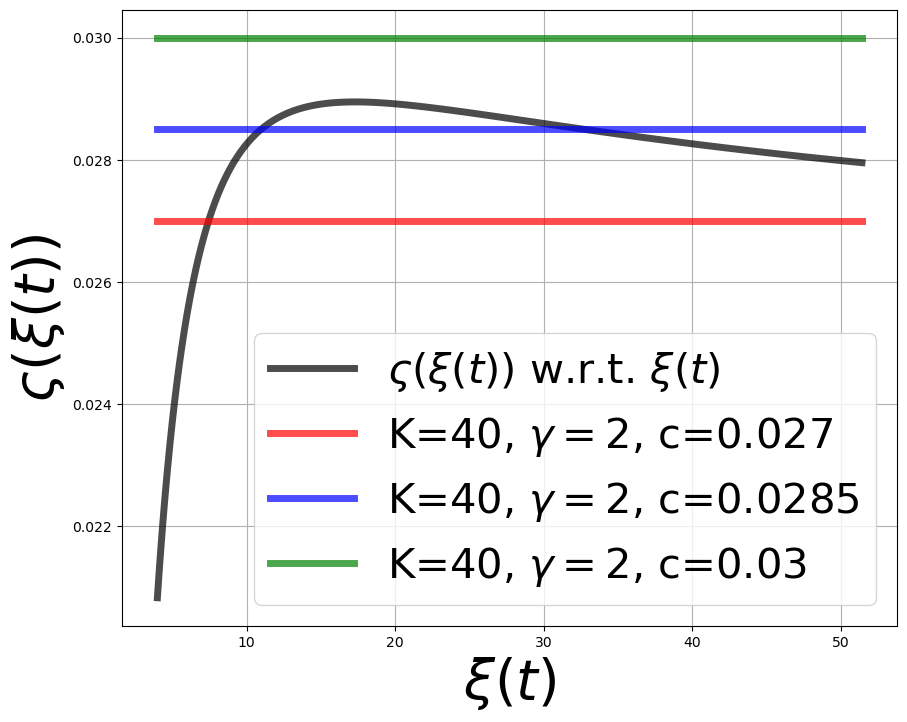}} \
    \subfloat[$\psi(t)$ w.r.t. $t$]{\includegraphics[width=0.4\textwidth]{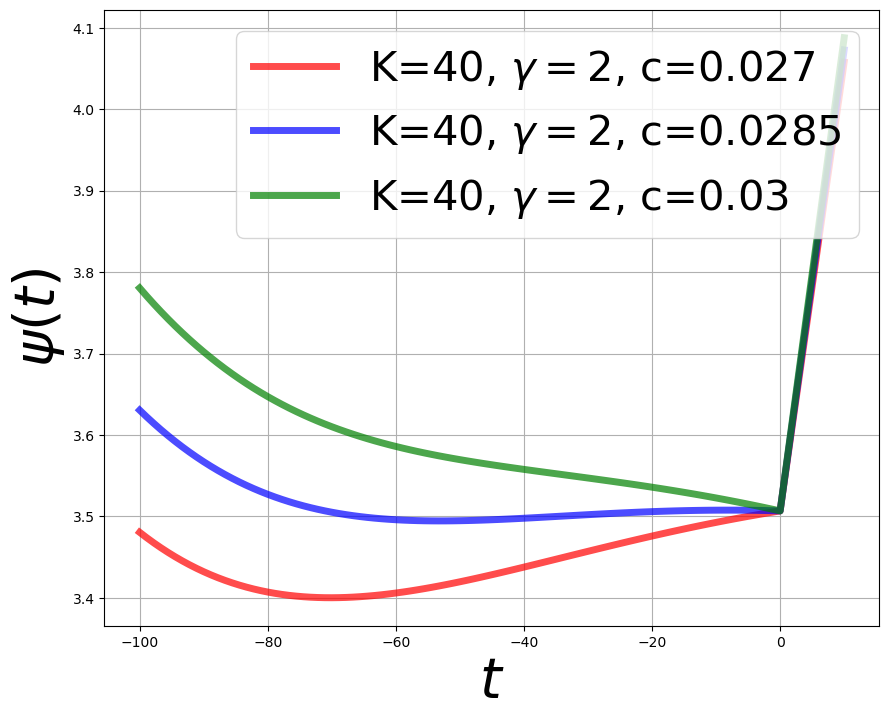}} \
    \caption{\textbf{Illustration of the case of $\vartheta(\xi(0))< 0$, where $c=K\sqrt{n\lambda_{\mb W} \lambda_{\mb H}}$}.
    }
    \label{fig:fl_case_2}
\end{figure}

\subsection{LS is in GL}\label{app:check-ls-in-gl}
    In this section, we will show that the LS defined in \eqref{def:ls} belongs to the GL defined in \Cref{def:GLoss-app}. First, let us rewrite the LS definition in GL form as following:
\begin{align*}
    \Lls (\vz,\vy_k) \;&=\;-\paren{1-\frac{(K-1)\alpha}{K}} \log\paren{\frac{\exp(z_{k})}{\sum_{j=1}^K\exp(z_{j})}}
    -\frac{\alpha}{K}\sum_{\ell\neq k}^K\log\paren{\frac{\exp(z_{\ell})}{\sum_{j=1}^K\exp(z_{j})}}\\
    \;&=\;\paren{1-\frac{(K-1)\alpha}{K}} \log\paren{\frac{\sum_{j=1}^K\exp(z_{j})}{\exp(z_{k})}}
    +\frac{\alpha}{K}\sum_{\ell\neq k}^K\log\paren{\frac{\sum_{j=1}^K\exp(z_{j})}{\exp(z_{\ell})}}\\
    \;&=\;\paren{1-\frac{(K-1)\alpha}{K}} \log\paren{\sum_{j=1}^K\exp(z_{j}-z_{k}))}
    +\frac{\alpha}{K}\sum_{\ell\neq k}^K\log\paren{\frac{\sum_{j=1}^K\exp(z_{j}-z_{k})}{\exp(z_{\ell}-z_{k})}}\\
    \;&=\;\log\paren{\sum_{j=1}^K\exp(z_{j}-z_{k}))}
    -\frac{\alpha}{K}\sum_{\ell\neq k}^K(z_{\ell}-z_{k})\\
    \;&\geq\;\log\paren{1+(K-1)\exp{\paren{\frac{z_{j}-z_{k}}{K-1}}}}-\frac{\alpha}{K}\sum_{\ell\neq k}^K(z_{\ell}-z_{k})
\end{align*}
where the inequality is due to the $\log$ is an increasing and function and $\exp$ is a strictly convex function, and it achieves equality only when $z_{j}=z_{j'}$ for all $j,j' \neq k$. Therefore, there exists such a function $\phi_{\text{LS}}$ to lower bound original LS loss $\Lls (\vz,\vy_k)$ as following:
\begin{align*}
    \phi_{\text{LS}}(t) \;&=\;\log\paren{1+(K-1)\exp{\paren{\frac{t}{K-1}}}}-\frac{\alpha}{K}t,
\end{align*}
which satisfies the condition of \eqref{gl-property-1-app}. Next, we will show $\phi_{\text{LS}}(t)$ satisfies the condition \eqref{gl-property-3-app}. The first-order gradient of $\phi_{\text{LS}}(t)$ is following:
\begin{align*}
    \nabla\phi_{\text{LS}}(t) &= \frac{\exp{\paren{\frac{t}{K-1}}}}{1+(K-1)\exp{\paren{\frac{t}{K-1}}}}-\frac{\alpha}{K}
\end{align*}
Let denote $\psi_{\text{LS}}(t)=\phi_{\text{LS}}(t)+c|t|$, then 
\begin{itemize}
    \item{\textbf{When $t\geq0$:}} $\nabla\psi_{\text{LS}}(t)=\nabla\phi_{\text{LS}}(t)+c>0$ due to $\nabla\phi_{\text{LS}}(t)\geq0 \quad \text{for} \quad t>0$, thus the $\psi_{\text{LS}}(t)$ is an increasing function w.r.t. $t$, and the minimizer is achieved when $x=0$.
    \item{\textbf{When $t\leq0$:}} $\nabla\psi_{\text{LS}}(t)=\nabla\phi_{\text{LS}}(t)-c$, and $\nabla \phi_{\text{LS}}(t)$ is an increasing function, which achieves minimizer when $t=0$ such that $\phi_{\text{LS}}(t)=\frac{1-\alpha}{K}$.
    \begin{itemize}
        \item{if $c\geq\frac{1-\alpha}{K}$}, $\nabla\psi_{\text{LS}}(t)<0$, and $\psi(t)$ is a decreasing function for $t\leq0$, and the minimizer is achieved when $t=0$;
        \item{if $0<c\leq\frac{1-\alpha}{K}$}, there exist such $t^*$ such that $\nabla\psi_{\text{LS}}(t)=0$. When $t<t^*$, $\phi_{\text{LS}}(t)$ is a decreasing function; and when $t^* < t\leq0$, $\phi_{\text{LS}}(t)$ is an increasing function. Therefore, the minimizer is achieved when $t=t^*<0$
    \end{itemize}
\end{itemize}
Combing them together, we can prove that $\phi_{\text{LS}}$ satisfies the condition of \eqref{gl-property-1-app}.



\section{Proof of \Cref{thm:global-minima} for GL}\label{app:thm-global-fl}

In this part of appendices, we prove \Cref{thm:global-minima} in \Cref{sec:main-results} that we restate as follows.

\begin{theorem}[Global Optimality Condition of GL]\label{thm:global-minima-app-gl}
	Assume that the number of classes $K$ is smaller than feature dimension $d$, i.e., $K< d$, and the dataset is balanced for each class, $n=n_1=\cdots=n_K$. Then any global minimizer $(\mW^\star, \mH^\star,\vb^\star)$ of 
	\begin{align}\label{eq:obj-app-gl}
     \min_{\mb W , \mb H,\mb b  } \; f(\mb W,\mb H,\mb b) \;&:=\; g(\mW\mH + \vb\vone^\top) \;+\; \frac{\lambda_{\mb W} }{2} \norm{\mb W}{F}^2 + \frac{\lambda_{\mb H} }{2} \norm{\mb H}{F}^2 + \frac{\lambda_{\mb b} }{2} \norm{\mb b}{2}^2, \nonumber\\
\end{align}
with 
\begin{align}\label{eqn:g-lgl-app}
    &g(\mW\mH + \vb\vone^\top) :=\sum_{i=1}^n g(\mW\mH_i + \vb\vone^\top) := \frac{1}{N}\sum_{k=1}^K\sum_{i=1}^n \mc L (\mW\vh_{k,i} + \vb,\vy_k);\\
    &\mc L (\mW\vh_{k,i} + \vb,\vy_k)=\mc L(\vz_{k,i},\vy_k) \text{ satisfying the the \textbf{Contrastive property} in \Cref{def:GLoss-app}};
\end{align}

obeys the following 
\begin{align*}
    & \norm{\mb w^\star}{2}\;=\; \norm{\mb w^{\star 1} }{2} \;=\; \norm{\mb w^{\star 2}}{2} \;=\; \cdots \;=\; \norm{\mb w^{\star K} }{2}, \quad \text{and}\quad \mb b^\star = b^\star \mb 1, \\ 
    & \vh_{k,i}^\star \;=\;  \sqrt{ \frac{ \lambda_{\mb W}  }{ \lambda_{\mb H} n } } \vw^{\star k} ,\quad \forall \; k\in[K],\; i\in[n],  \quad \text{and} \quad  \ol{\mb h}_{i}^\star \;:=\; \frac{1}{K} \sum_{j=1}^K \mb h_{j,i}^\star \;=\; \mb 0, \quad \forall \; i \in [n],
\end{align*}
where either $b^\star = 0$ or $\lambdab=0$, and the matrix $  \mW^{\star\top} $ is in the form of $K$-simplex ETF structure defined in Definition \ref{def:simplex-ETF} in the sense that 
\begin{align*}
      \mW^{\star\top} \mW^{\star}\;=\; \norm{\mb w^\star}{2}^2\frac{K}{K-1}  \paren{ \mb I_K - \frac{1}{K} \mb 1_K \mb 1_K^\top }.
\end{align*}
\end{theorem}

\subsection{Main Proof}
At a high level, we lower bound the general loss function based on the contrastive property \eqref{gl-property-1-app}, then check the equality conditions hold for the lower bounds and these equality conditions ensure that the global solutions $(\mW^\star,\mH^\star,\vb^\star)$ are in the form as shown in \Cref{thm:global-minima-app-gl}.

\vspace{0.1in}

\begin{proof}[Proof of \Cref{thm:global-minima-app-gl}] 
First by  \Cref{lem:critical-balance-gl}, \Cref{lem:critical-balance-gl-per-group} and \Cref{lem:isotropic-bias-gl}, we know that any critical point $(\mb W,\mb H,\mb b)$ of $f$ in \eqref{eq:obj-app-gl} satisfies 
\begin{align*}
    & \mW^\top\mW = \frac{\lambda_{\mH}}{\lambda_{\mW}}\mH\mH^\top;\\
    & \lambdaH \mH_i =-\mW^\top \nabla_{\mb Z_i = \mb W\mb H_i}\; g(\mW\mH_i + \vb\vone^\top);\\
    &\vb = -\frac{\nabla g(\mW\mH + \vb\vone^\top)}{\lambda_{\vb}} \vone. 
\end{align*}
For the rest of the proof, let $\mG_i=\nabla_{\mb Z_i = \mb W\mb H_i}\; g(\mW\mH_i + \vb\vone^\top)$ and $\tau= -\frac{\nabla g(\mW\mH + \vb\vone^\top)}{\lambda_{\vb}}$ to simplify the notations, and thus $\norm{\mH}{F}^2 = \frac{\lambda_{\mH}}{\lambda_{\mW}}\norm{\mW}{F}^2$, $\lambdaH \mH_i=-\mW^\top\mG_i$ and $\vb = \tau \vone$. 

We will first provide a lower bound for the general loss term $g(\mW\mH + \vb\vone^\top)$ according to the \Cref{def:GLoss-app}, and then show that the lower bound is attained if and only if the parameters are in the form described in \Cref{thm:global-minima-app-gl}. By Lemma \ref{lem:lower-bound-g-gl}, we have
\begin{align*}
    f(\mb W,\mb H,\mb b)\;&=\;  g(\mb W \mb H + \mb b \mb 1^\top) \;+\; \frac{\lambda_{\mb W} }{2} \norm{\mb W}{F}^2 + \frac{\lambda_{\mb H} }{2} \norm{\mb H}{F}^2 + \frac{\lambda_{\mb b} }{2} \norm{\mb b}{2}^2 \\
 \;&\geq\; \phi\paren{\rho^\star} + K\sqrt{n\lambda_{\mb W} \lambda_{\mb H} } |\rho^\star|
\end{align*}
where $\phi$ is lower bound function satisfying the \Cref{def:GLoss-app}, $\rho^\star=\arg\min_\rho\phi\paren{\rho} + K\sqrt{n\lambda_{\mb W} \lambda_{\mb H} } |\rho|\leq 0$. 
Furthermore, by \Cref{lem:lower-bound-g-gl}, we know that $\bar{\mZ}_i^\star=\mW^\star\mH^\star_i=-\rho^\star\paren{\mb I_K - \frac{1}{K} \mb 1_K \mb 1_K^\top }$, which satisfies the $K$-simplex ETF structure defined in Definition \ref{def:simplex-ETF}. In \Cref{lem:lower-bound-equality-cond-gl}, we show the any minimizer $(\mW^\star, \mH^\star, \vb^\star)$ of $f(\mb W,\mb H,\mb b)$ has following properties via check the equality conditions hold for the lower bounds in \Cref{lem:lower-bound-g-gl}:
\begin{itemize}
    \item[(a)] $\norm{\vw ^\star}{2} \;=\; \norm{\vw^{\star 1}}{2} \;=\; \norm{\vw^{\star 2}}{2} \;=\; \cdots \;=\; \norm{\vw^{\star K}}{2}$;
    \item[(b)] $\mb b^\star = b^\star \mb 1$, where either $b^\star = 0$ or $\lambdab = 0$;
    \item[(c)] $\ol{\mb h}_{i}^\star \;:=\; \frac{1}{K} \sum_{j=1}^K \mb h_{j,i}^\star \;=\; \mb 0, \quad \forall \; i \in [n]$, and $\sqrt{ \frac{ \lambda_{\mb W}  }{ \lambda_{\mb H} n } } \vw^{k\star}\;=\; \vh_{k,i}^\star,\quad \forall \; k\in[K],\; i\in[n]$;
    \item[(d)] $\mb W \mb W^\top \;=\; \norm{\vw^\star}{2}^2\frac{K-1}{K} \paren{ \mb I_K - \frac{1 }{K} \mb 1_K \mb 1_K^\top }$;
\end{itemize}
The proof is complete.
\end{proof}

\subsection{Supporting Lemmas}

We first characterize the following balance property between $\mW$ and $\mH$ for any critical point  $(\mW,\mH,\vb)$ of our loss function:

\begin{lemma}\label{lem:critical-balance-gl} 
Let $\rho = \norm{\mb W }{F}^2$. Any critical point $(\mW,\mH,\vb)$ of \eqref{eq:obj-app-gl} obeys
\begin{align}\label{eq:critical-balance-gl}
    \mW^\top\mW \;=\; \frac{\lambda_{\mH}}{\lambda_{\mW}}\mH\mH^\top\quad \text{and}\quad \rho \;=\; \norm{\mb W}{F}^2 \;=\; \frac{\lambda_{\mH}}{\lambda_{\mW}} \norm{\mb H}{F}^2.
\end{align}

\end{lemma}

\begin{proof}[Proof of Lemma \ref{lem:critical-balance-gl}]
By definition, any critical point $(\mW,\mH,\vb)$ of \eqref{eq:obj-app-gl} satisfies the following:
\begin{align}
    	\nabla_{\mW}f(\mW,\mH,\vb) \;&=\; \nabla_{\mb Z = \mb W\mb H}\; g(\mW\mH + \vb \vone^\top)\mH^\top + \lambdaW \mW \;=\; \vzero,\label{eqn:W-crtical-gl} \\
	\nabla_{\mH}f(\mW,\mH,\vb) \;&=\;  \mW^\top \nabla_{\mb Z = \mb W\mb H}\; g(\mW\mH + \vb\vone^\top) + \lambdaH \mH \;=\; \vzero. \label{eqn:H-crtical-gl}
\end{align}
Left multiply the first equation by $\mb W^\top$ on both sides and then right multiply second equation by $\mH^\top$ on both sides, it gives
\begin{align*}
    \mb W^\top \nabla_{\mb Z = \mb W\mb H}\; g(\mW\mH + \vb\vone^\top)\mH^\top \;&=\; - \lambdaW \mb W^\top \mW , \\
     \mb W^\top \nabla_{\mb Z = \mb W\mb H}\; g(\mW\mH + \vb\vone^\top)\mH^\top  \;&=\; - \lambdaH \mb H^\top \mH.
\end{align*}
Therefore, combining the equations above, we obtain
\begin{align*}
    \lambdaW \mW^\top \mW \;=\; \lambdaH \mH \mH^\top.
\end{align*}
Moreover, we have
\begin{align*}
    \rho\;=\; \norm{\mb W}{F}^2 \;=\; \trace \paren{ \mb W^\top \mb W } \;=\;  \frac{ \lambda_{\mb H} }{ \lambda_{\mb W} } \trace \paren{ \mb H \mb H^\top } \;=\; \frac{ \lambda_{\mb H} }{ \lambda_{\mb W} } \trace \paren{ \mb H^\top\mb H  } \;=\; \frac{ \lambda_{\mb H} }{ \lambda_{\mb W} } \norm{\mb H}{F}^2,
\end{align*}
as desired.
\end{proof}

Next, we characterize the following relationship per group between $\mW$ and $\mH_i$ for $i\in [n]$ for any critical $(\mW,\mH,\vb)$ of \eqref{eq:obj-app-gl} satisfies the following:

\begin{lemma}\label{lem:critical-balance-gl-per-group} 
Let $\mG_i = \nabla_{\mb Z_i = \mb W\mb H_i}\; g(\mW\mH_i + \vb\vone^\top)$. Any critical point $(\mW,\mH,\vb)$ of \eqref{eq:obj-app-gl} obeys
\begin{align}\label{eqn:Hi-crtical-gl}
    \mW^\top\mG_i \;=\; -\lambdaH \mH_i.
\end{align}

\end{lemma}

\begin{proof}[Proof of Lemma \ref{lem:critical-balance-gl}]
By definition, any critical point $(\mW,\mH,\vb)$ of \eqref{eq:obj-app-gl} satisfies the following:
\begin{align}
	\nabla_{\mH_i}f(\mW,\mH,\vb) \;&=\;  \mW^\top \nabla_{\mb Z_i = \mb W\mb H_i}\; g(\mW\mH_i + \vb\vone^\top) + \lambdaH \mH_i \;=\; \vzero;\\
	\mW^\top\mG_i \;&=\; -\lambdaH \mH_i.
\end{align}
as desired.
\end{proof}

We then characterize the following isotropic property of $\vb$ for any critical point  $(\mW,\mH,\vb)$ of our loss function:

\begin{lemma}\label{lem:isotropic-bias-gl} 
Let $\tau = -\frac{\nabla g(\mW\mH + \vb\vone^\top)}{\lambda_{\vb}}$. Any critical point $(\mW,\mH,\vb)$ of \eqref{eq:obj-app-gl} obeys
\begin{align}\label{eq:isotropic-bias-gl}
    \vb \;=\; \tau \vone.
\end{align}

\end{lemma}

\begin{proof}[Proof of Lemma \ref{lem:isotropic-bias-gl}]
By definition, any critical point $(\mW,\mH,\vb)$ of \eqref{eq:obj-app-gl} satisfies the following:
\begin{align}
    	\nabla_{\vb}f(\mW,\mH,\vb) \;&=\; \nabla\; g(\mW\mH + \vb \vone^\top)\vone + \lambdab \vb \;=\; \vzero,\label{eqn:b-crtical-gl} \nonumber\\
	    \vb \;&=\; -\frac{\nabla g(\mW\mH + \vb\vone^\top)}{\lambda_{\vb}} \vone = \tau \vone
\end{align}
as desired.
\end{proof}

\begin{lemma}\label{lem:lower-bound-g-gl}
Let $\mb W = \begin{bmatrix} (\mb w^1)^\top \\ \vdots \\ (\mb w^K)^\top 
\end{bmatrix}\in \bb R^{K \times d}$, $\mb H = \begin{bmatrix}
    \mb H_1 & \mb H_2 & \cdots & \mb H_n
    \end{bmatrix}\in \bb R^{d \times N}$, $\mb H_i= \begin{bmatrix}\mb h_{1,i} & \cdots & \mb h_{K,i} \end{bmatrix} \in \bb R^{ d \times K}$,  $\bar{\mb Z} = \mW\mH\in \bb R^{ d \times N}$, $N= nK$, and $\vb=\tau\vone$. Given  $g(\mW\mH + \vb\vone^\top_K)$ defined in \eqref{eqn:g-lgl-app}, for any critical point $(\mW,\mH,\vb)$ of \eqref{eq:obj-app-gl}, it satisfies
\begin{align}\label{eqn:lower-bound-g-gl}
   f(\mb W,\mb H,\mb b) &\geq  \phi\paren{\rho^\star} + (K-1)\sqrt{n\lambda_{\mb W} \lambda_{\mb H} } |\rho^\star|\\
   \bar{\mZ}^\star &= -\rho^\star\paren{\mb I_K - \frac{1}{K} \mb 1_K \mb 1_K^\top }{\mb I_K^n}
\end{align}
where $\phi$ is lower bound function satisfying the \Cref{def:GLoss-app}, $\rho^\star=\arg\min_\rho\phi\paren{\rho} + K\sqrt{n\lambda_{\mb W} \lambda_{\mb H} } |\rho|$, and $\bar{\mb Z}^\star = \mW^\star\mH^\star$.
\end{lemma}

\begin{proof}[Proof of Lemma \ref{lem:lower-bound-g-gl}]
With $\bar{\mb Z}_i = \mW\mH_i$, and $\norm{\bar{\mb Z}_i}{2} = \sigma_i^{\max}$, we have the following lower bound for $ f(\mb W,\mb H,\mb b)$ as 
\begin{equation}
\begin{split}
    f(\mb W,\mb H,\mb b) &= g(\mW\mH + \vb\vone^\top) + \frac{\lambda_{\mb W} }{2} \norm{\mb W}{F}^2 + \frac{\lambda_{\mb H} }{2} \norm{\mb H}{F}^2 + \frac{\lambda_{\mb b} }{2} \norm{\mb b}{2}^2    \\
    \;&=\;\sum_{i=1}^n \paren{g(\mW\mH_i + \vb\vone^\top) + \frac{\lambda_{\mb W} }{2n} \norm{\mb W}{F}^2 + \frac{\lambda_{\mb H} }{2} \norm{\mb H_i}{F}^2 }+ \frac{\lambda_{\mb b} }{2} \norm{\mb b}{2}^2    \\
    \;&\geq\; \sum_{i=1}^n \paren{ g(\bar{\mb Z}_i + \vb\vone^\top) + \sqrt{\lambda_{\mb W} \lambda_{\mb H}/n}\norm{\bar{\mb Z}_i}{*}} + \frac{\lambda_{\mb b} }{2}\norm{\mb b}{2}^2\\
    \;&\geq\;  \sum_{i=1}^n \paren{ g(\bar{\mb Z}_i + \vb\vone^\top) +\sqrt{\lambda_{\mb W} \lambda_{\mb H}/n}\frac{\norm{\bar{\mb Z}}{F}^2}{\norm{\bar{\mb Z}_i}{2}} }+ \frac{\lambda_{\mb b} }{2} \norm{\mb b}{2}^2 \\ 
    \;&=\;  \sum_{i=1}^n \paren{ g(\bar{\mb Z}_i + \vb \vone^\top) + \frac{\sqrt{\lambda_{\mb W} \lambda_{\mb H}/n}}{\sigma_i^{\max}} \norm{\bar{\mb Z}_i}{F}^2} + \frac{\lambda_{\mb b} }{2} \norm{\mb b}{2}^2,\nonumber 
\end{split}
\end{equation}
where the first inequality  is from \Cref{lem:nuclear-norm}, and the second inequality becomes equality only when $\bar{\mZ}_i\neq\vzero$ and 
\begin{equation}\begin{split}
    \forall \; k,  \sigma_k(\bar{\mZ}_i) &= \sigma_i^{\max} \text{ or } 0  \\
    \exists \; k,  \sigma_k(\bar{\mZ}_i) &\neq 0
    \end{split}\label{lem: lower-bound-norm-exchange-euqality-gl}
\end{equation}
where $\sigma_k(\bar{\mZ}_i)$ is the $k$-th singular value of $\bar{\mZ}_i$. While we only consider $\bar{\mZ}_i\neq\vzero$, we will show the $\bar{\mZ}_i=\vzero$ can be included in an uniform form as following proof. We can further bound $f(\mb W,\mb H,\mb b)$ by 
\begin{align}
    &f(\mb W,\mb H,\mb b) \;\geq\; \sum_{i=1}^n \paren{ g(\bar{\mb Z}_i + \vb \vone^\top) + \frac{\sqrt{\lambda_{\mb W} \lambda_{\mb H}/n}}{\sigma_i^{\max}} \norm{\bar{\mb Z}_i}{F}^2} + \frac{\lambda_{\mb b} }{2} \norm{\mb b}{2}^2, \nonumber\\
    \;\geq\;& \frac{1}{N}\sum_{k=1}^K\sum_{i=1}^n\phi\paren{\sum_{j\neq k}\paren{\bar{z}_{k,i,j}-\bar{z}_{k,i,k}+\underbrace{b_j-b_k}_{=0}}}+\sum_{i=1}^n\frac{\sqrt{\lambda_{\mb W} \lambda_{\mb H}/n}}{\sigma_i^{\max}} \norm{\bar{\mb Z}_i}{F}^2 + \frac{\lambda_{\mb b} }{2} \norm{\mb b}{2}^2, \nonumber\\
    \;=\; & \frac{1}{N}\sum_{i=1}^n\sum_{k=1}^K\paren{ \phi\paren{\sum_{j\neq k}^K\paren{\bar{z}_{k,i,j}-\bar{z}_{k,i,k}}} + \frac{K\sqrt{n\lambda_{\mb W} \lambda_{\mb H} }}{\sigma_i^{\max}} \norm{\mb \bar{z}_{k,i}}{2}^2} + \frac{\lambda_{\mb b} }{2} \norm{\mb b}{2}^2, 
\label{lem:decouplable-per-sample-per-class}
\end{align}
where the first inequality is from the first condition \eqref{gl-property-1-app} of loss function $\mc L$ and the equality achieves only when $\bar{z}_{k,i,j}=\bar{z}_{k,i,j'}$ for $j\neq k, j'\neq k$, and $b_j-b_k=0$ is due to \Cref{lem:isotropic-bias-gl}. If we denote by $\rho_{k,i}=\sum_{j\neq k}^K\paren{\bar{z}_{k,i,j}-\bar{z}_{k,i,k}}/(K-1)$,  then 
\begin{align*}
    \norm{\bar{\vz}_{k,i}}{2}^2&=\sum_{j\neq k}\bar{z}^2_{k,i,j}+\bar{z}^2_{k,i,k}\\
    &\geq (K-1)\paren{\sum_{j\neq k}\frac{\bar{z}_{k,i,j}}{K-1}}^2+\bar{z}_{k,i,k}\\
    &= (K-1)\paren{\sum_{j\neq k}\frac{\bar{z}_{k,i,j}-\bar{z}_{k,i,k}}{K-1}+\bar{z}_{k,i,k}}^2+\bar{z}_{k,i,k}\\
    &= (K-1)\paren{\rho_{k,i}+\bar{z}_{k,i,k}}^2+\bar{z}_{k,i,k}\\
    &\geq \frac{K-1}{K}\rho_{k,i}^2
\end{align*}
where the first inequality achieves equality only when $\bar{z}_{k,i,j}=\bar{z}_{k,i,j'}$ for $j\neq k, j'\neq k$, and the last line achieves equality only when $\bar{z}_{k,i,k}=-\frac{K-1}{K}\rho_{k,i}$, thus $\bar{z}_{k,i,j}=\frac{1}{K}\rho_{k,i}$ for $j\neq k$. Denoting $\mb \rho_i = \begin{bmatrix}
     \rho_{i,1} &  \rho_{i,2} &\cdots& \rho_{i,K} 
\end{bmatrix}$ and $\text{diag}(\mb \rho_i)$ is a diagonal matrix using $\mb \rho_i$ as diagonal entries, and supposing $|\rho_1|\geq|\rho_2|>\cdots>|\rho_K|$,  we can express $\bar{\mZ}_i$ as:
\begin{align}
    \bar{\mZ}_i&=-(\mb I_K - \frac{1}{K}\vone_K\vone^\top_K) \text{diag}(\mb \rho_i),
\end{align} and we can extend the expression of \eqref{lem:decouplable-per-sample-per-class} as following
\begin{align}\label{lem:after-replace-by-rho}
    f(\mb W,\mb H,\mb b) \;\geq\; &\frac{1}{N}\sum_{i=1}^n\sum_{k=1}^K \paren{\underbrace{\phi\paren{\rho_{k,i}} + \frac{(K-1)\sqrt{n\lambda_{\mb W} \lambda_{\mb H} }}{\sigma_i^{\max}}\rho_{k,i}^2}_{\psi(\rho_{k,i})}} +\frac{\lambdab}{2}\norm{\vb}{2}
\end{align}
which is decouplable if we treat the $i$-th samples per class as a group, thus we only consider the $i$-th samples per class. In the next part, denote $\rho^\star=\arg\min_{\rho} \phi\paren{\rho} + (K-1)\sqrt{n\lambda_{\mb W} \lambda_{\mb H} }|\rho|$. 

When $K\geq3$, according to the $\mZ=-(\mb I_K - \frac{1}{K}\vone_K\vone^\top_K) \text{diag}(\mb \rho_i)$, the condition of \eqref{lem: lower-bound-norm-exchange-euqality-gl} and \Cref{lem:Z-structure}, we know $\mZ$ has only two possible forms corresponding to two different objective value of $\sum_{k=1}^K\psi(\rho_{k})$ such that 
\begin{itemize}
    \item {$|\rho_{1}|=|\rho_2|=\cdots=|\rho_K|$:} we can have $\sigma_{\max}=|\rho_1|$ and 
    \begin{align*}
        \sum_{k=1}^K\psi(\rho_{k}) &= \sum_{k=1}^K\paren{\phi\paren{\rho_{k}} + \frac{(K-1)\sqrt{n\lambda_{\mb W} \lambda_{\mb H} }}{\sigma^{\max}}\rho_k^2}\\
        &=\sum_{k=1}^K\paren{\phi\paren{\rho_{k}} + (K-1)\sqrt{n\lambda_{\mb W} \lambda_{\mb H} }|\rho_{k}|}\\
        &\geq K\paren{\phi\paren{\rho^\star} + (K-1)\sqrt{n\lambda_{\mb W} \lambda_{\mb H} }|\rho^\star|}
    \end{align*}
    where the last line holds equality only when $|\rho_{1}|=|\rho_2|=\cdots=|\rho_K|=\rho^*$.
    \item {$|\rho_{2}|=\cdots=|\rho_K|=0$:} we can have $\sigma_{\max}=\sqrt{(K-1)/K}|\rho_1|$ and 
    \begin{align*}
        \sum_{k=1}^K\psi(\rho_{k}) &= \phi\paren{\rho_{1}} + (K-1)\sqrt{n\lambda_{\mb W} \lambda_{\mb H} }\sqrt{\frac{K}{K-1}}|\rho_1|+(K-1)\phi\paren{0}\\
        &=\phi\paren{\rho_{1}} + (K-1)\sqrt{n\lambda_{\mb W} \lambda_{\mb H} }|\rho_1|\\
        &\;+(K-1)\sqrt{n\lambda_{\mb W} \lambda_{\mb H} }\paren{\sqrt{\frac{K}{K-1}}-1}|\rho_1|+(K-1)\phi\paren{0}\\
        &\geq K\paren{\phi\paren{\rho^\star} + (K-1)\sqrt{n\lambda_{\mb W} \lambda_{\mb H} }|\rho^\star|}
    \end{align*}
    where the last line holds equality only when $|\rho_{1}|=\cdots=|\rho_{K}|=|\rho^\star|=0$.
\end{itemize}
When $K=2$, according to the \Cref{lem:diagonal-plus-rank-one}, we can calculate $\sigma_{\max}=\sqrt{\frac{\rho_1^2+\rho_2^2}{2}}$, then 
\begin{align*}
    \sum_{k=1}^2\psi(\rho_{i})&=\phi\paren{\rho_{1}}+\phi\paren{\rho_{2}}+\frac{(K-1)\sqrt{n\lambda_{\mb W} \lambda_{\mb H} }}{\sigma^{\max}}\paren{\rho_1^2+\rho_2^2}\\
    &=\phi\paren{\rho_{1}}+\phi\paren{\rho_{2}}+(K-1)\sqrt{n\lambda_{\mb W} \lambda_{\mb H} }\sqrt{2(\rho_1^2+\rho_2^2)}\\
    &= \phi\paren{\rho_{1}}+(K-1)\sqrt{n\lambda_{\mb W} \lambda_{\mb H} }|\rho_1|+\phi\paren{\rho_{2}}+(K-1)\sqrt{n\lambda_{\mb W} \lambda_{\mb H} }|\rho_2|\\
    &\;+(K-1)\sqrt{n\lambda_{\mb W} \lambda_{\mb H} }\paren{\sqrt{2(\rho_1^2+\rho_2^2)}-|\rho_1|-|\rho_2|}\\
    &\geq 2\paren{\phi\paren{\rho^\star} + (K-1)\sqrt{n\lambda_{\mb W} \lambda_{\mb H} }|\rho^\star|}
\end{align*}
where the last line holds equality only when $|\rho_{1}|=|\rho_2|=|\rho^\star|$.

Combining them together, for $K\geq 2$, we can further extend the expression of \eqref{lem:after-replace-by-rho} as following
\begin{align}
    f(\mb W,\mb H,\mb b) \;\geq\; &\frac{1}{N}\sum_{i=1}^n\sum_{k=1}^K \paren{\phi\paren{\rho_{k,i}} + \frac{(K-1)\sqrt{n\lambda_{\mb W} \lambda_{\mb H} }}{\sigma_i^{\max}}\rho_{k,i}^2} +\frac{\lambdab}{2}\norm{\vb}{2}\nonumber\\
    \;\geq\; &\frac{1}{N}\sum_{i=1}^n K\paren{\phi\paren{\rho^\star} + (K-1)\sqrt{n\lambda_{\mb W} \lambda_{\mb H} }|\rho^\star|} +\frac{\lambdab}{2}\norm{\vb}{2}\nonumber\\
    \;\geq\;&\phi\paren{\rho^\star} + (K-1)\sqrt{n\lambda_{\mb W} \lambda_{\mb H} } |\rho^\star|\label{eqn:b-zeros}
\end{align}
where the last equation is achieved when $\vb=\mb 0$ or $\lambdab=0$. According to the condition \eqref{gl-property-3-app} of loss function $\mc L$ that the minimizer $\rho^\star$ of $\phi(\rho)+c|\rho|$ is unique for any $c>0$, and by denoting $\mb I_K^n=\begin{bmatrix}
     \mb I_K &\cdots& \mb I_K
\end{bmatrix}\in\mathbb{R}^{K\times nK}$, we  have
\begin{align}
    \bar{\mZ}_i^\star &= -\rho^\star\paren{\mb I_K - \frac{1}{K} \mb 1_K \mb 1_K^\top }\label{eqn:Z-strcuture-per-group}\\
    \bar{\mZ}^\star &= -\rho^\star\paren{\mb I_K - \frac{1}{K} \mb 1_K \mb 1_K^\top }{\mb I_K^n}\label{eqn:Z-strcuture}
\end{align}
as desired.
\end{proof}

Next, we show that the lower bound in \eqref{eqn:lower-bound-g-gl} is attained if and only if $(\mb W, \mb H,\mb b)$ satisfies the following conditions.  
\begin{lemma}\label{lem:lower-bound-equality-cond-gl}
Under the same assumptions of Lemma \ref{lem:lower-bound-g-gl},
the lower bound in \eqref{eqn:lower-bound-g-gl} is attained for any minimizer $(\mb W^\star,\mb H^\star,\mb b^\star)$ of \eqref{eq:obj-app-gl}
if and only if the following hold
\begin{align*}
    & \norm{\mb w^\star}{2}\;=\; \norm{\mb w^{\star 1} }{2} \;=\; \norm{\mb w^{\star 2}}{2} \;=\; \cdots \;=\; \norm{\mb w^{\star K} }{2}, \quad \text{and}\quad \mb b^\star = b^\star \mb 1, \\ 
    & \vh_{k,i}^\star \;=\;  \sqrt{ \frac{ \lambda_{\mb W}  }{ \lambda_{\mb H} n } } \vw^{\star k} ,\quad \forall \; k\in[K],\; i\in[n],  \quad \text{and} \quad  \ol{\mb h}_{i}^\star \;:=\; \frac{1}{K} \sum_{j=1}^K \mb h_{j,i}^\star \;=\; \mb 0, \quad \forall \; i \in [n],
\end{align*}
where either $b^\star = 0$ or $\lambdab=0$, and the matrix $\mW^{\star\top} $ is in the form of $K$-simplex ETF structure (see appendix for the formal definition) in the sense that 
\begin{align*}
      \mW^{\star\top} \mW^{\star}\;=\; \norm{\mb w^\star}{2}^2\frac{K}{K-1}  \paren{ \mb I_K - \frac{1}{K} \mb 1_K \mb 1_K^\top }.
\end{align*}

\end{lemma}
The proof of \Cref{lem:lower-bound-equality-cond-gl} utilizes the Lemma \Cref{lem:critical-balance-gl}, \Cref{lem:critical-balance-gl-per-group} and  \Cref{lem:isotropic-bias-gl}, and the conditions \eqref{eqn:b-zeros} and the structure of $\bar{\mb{Z}}^\star$ \eqref{eqn:Z-strcuture} during the proof of Lemma \ref{lem:lower-bound-g-gl}. 

\begin{proof}[Proof of Lemma \ref{lem:lower-bound-equality-cond-gl}]
From the \eqref{eqn:Z-strcuture}, we know that $\bar{\mb{Z}}^\star_1=\bar{\mb{Z}}^\star_2=\cdots=\bar{\mb{Z}}^\star_n$ and then $\mG_i^\star=\nabla_{\bar{\mb{Z}}^\star_i=\mW^\star\mH^\star_i} g(\mW^\star\mH^\star_i+\vb\vone^\top)$ is equivalent for $i\in[n]$. Let denote $\mG^\star= \mG^\star_1=\mG^\star_2=\cdots=\mG^\star_n$, the \eqref{eqn:Hi-crtical-gl} in Lemma \ref{lem:critical-balance-gl-per-group} can be expressed as:
\begin{align*}
    {\mW^\star}^\top \mG^\star &=- \lambdaH \mH_i^\star
\end{align*}
Therefore, $\Tilde{\mH}^\star=\mH_1^\star=\mH_2^\star=\cdots=\mH_n^\star$, which means the last-layer features from different classes are collapsed to their corresponding class-mean $\vh_{k,1}^\star=\vh_{k,2}^\star=\cdots=\vh_{k,n}^\star$, for $k\in[K]$. Furthermore, $\mH^\star\mH^{\star\top}=n\Tilde{\mH}^\star\Tilde{\mH}^{\star\top}$, combining this with \eqref{eq:critical-balance-gl} in Lemma \ref{lem:critical-balance-gl}, we know that 
\begin{align*}
     \lambdaW \mW^{\star\top} \mW^\star \;=\; \lambdaH \mH^\star \mH^{\star\top}=n\lambdaH \Tilde{\mH}^\star\Tilde{\mH}^{\star\top}
\end{align*}
By denoting $\mW^\star={\mU_{\mW}} {\mSigma_{\mW}} {\mV^\top_{\mW}}$ and ${\tilde{\mH}^\star}={\mU_{\tilde{\mH}}} {\mSigma_{\tilde{\mH}}} {\mV^\top_{\tilde{\mH}}}$, where $\mU_{\mW}$, $\mSigma_{\mW}$, $\mV^\top_{\mW}$ are the left singular vector matrix, singular value matrix, and right singular vector matrix of $\mW^\star$, respectively; and $\mU_{\tilde{\mH}^\star}$, $\mSigma_{\tilde{\mH}^\star}$, $\mV_{{\tilde{\mH}}^\star}^\top$ are the left singular vector matrix, singular value matrix, and right singular vector matrix of $\tilde{\mH}$, respectively, we can get
\begin{align*}
    \mV^\top_{\mW}&=\mU_{\tilde{\mH}}\\
    \mSigma_{\mW} &= \sqrt{\frac{n\lambdaH}{\lambdaW}}\mSigma_{\tilde{\mH}}
\end{align*}
Therefore, $\mZ_i^\star=\mW^\star\Tilde{\mH}^\star=\sqrt{\frac{\lambdaW}{n\lambdaH}}\mU_{\mW}\mb{\Sigma}^2_{\mW}\mV^\top_{\Tilde{\mH}}$. According to the $\mZ_i=-\rho^\star(\mb I_K-\frac{1}{K}\vone_K\vone_K^\top)$ in \eqref{eqn:Z-strcuture-per-group} and $\rho^\star \leq 0$, which is symmetric, thus, $\mU_{\mW}=\mV_{\Tilde{\mH}}$, $\mW^\star=\sqrt{\frac{n\lambdaH}{\lambdaW}}\Tilde{\mH}^{\star\top}$, that is, $\vw^{\star k} \;=\; \sqrt{\frac{n\lambdaH}{\lambdaW}}\vh^\star_{k,i},\quad \forall \; k\in[K],\; i\in[n]$ and
\begin{align*}
    \mZ_i^\star &= \sqrt{\frac{\lambdaW}{n\lambdaH}}\mW^\star\mW^{\star\top} = \sqrt{\frac{\lambdaW}{n\lambdaH}}\mW^\star\mW^\star\\
    &= -\rho^\star(\mb I_K-\frac{1}{K}\vone_K\vone_K^\top) = -\rho^\star(\mb I_K-\frac{1}{K}\vone_K\vone_K^\top)(\mb I_K-\frac{1}{K}\vone_K\vone_K^\top)\\
    \mW^\star &= (\frac{\rho^{\star 2}n\lambdaH}{\lambdaW})^\frac{1}{4}(\mb I_K-\frac{1}{K}\vone_K\vone_K^\top)\\
    \Tilde{\mH}^\star &= (\frac{\rho^{\star 2}\lambdaW}{\lambdaH})^\frac{1}{4}(\mb I_K-\frac{1}{K}\vone_K\vone_K^\top)
\end{align*}
Therefore,
\begin{align*}
    & \norm{\mb w^{\star 1}}{2} \;=\; \norm{\mb w^{\star 2}}{2} \;=\; \cdots \;=\; \norm{\mb w^{\star K}}{2}\\
    &\ol{\mb h}_{i}^\star \;:=\; \frac{1}{K} \sum_{j=1}^K \mb h_{j,i}^\star \;=\; \mb 0, \quad \forall \; i \in [n]
\end{align*}
where $\ol{\mb h}_{i}^\star=\sum_{k=1}^K(\vh_{k,i}^\star)$ and according to the condition of \eqref{eqn:b-zeros} and  Lemma \ref{lem:isotropic-bias-gl}, $\vb^\star=\mb 0$ or $\lambdab=0$.
\end{proof}

\section{Proof of \Cref{cor:global-geometry-LS} and \Cref{cor:global-geometry-FL}}
\label{sec:appendix-prf-global-geometry}

Following \Cref{thm:global-geometry}, we only need to prove convexity for label smoothing and local convexity for focal loss.

For any output (logit) $\vz \in\R^{K}$, define
\[
\vp = \sigma(\vz)\in \R^{K}, \ \text{where} \ p_i = \frac{\exp(z_i)}{\sum_{j=1}^K \exp(z_j)}.
\]
Let $\vy^{\text{smooth}}\in\R^K$ be the label vector with $0\le y^{\text{smooth}}_i\le 1$ and $\sum_i y^{\text{smooth}}_i = 1$. The three loss functions can be written as 
\[
f(\vz) = \sum_{i=1}^K y^{\text{smooth}}_i \xi(p_i).
\]

Some useful properties:

\[
\partial_{z_i}\xi(p_k) = \begin{cases}
\xi'(p_k) (p_k - p_k^2), & i = k, \\ - \xi'(p_k)p_k p_i , & i\neq k,
\end{cases} \ \Longrightarrow \ \nabla_{\vz}\xi(p_k) = \xi'(p_k) p_k (\ve_k - \vp)
\]

\[
\partial_{z_i}p_k = \begin{cases}
p_k - p_k^2, & i = k, \\ -p_k p_i , & i\neq k,
\end{cases}
\ \Longrightarrow \ \nabla_{\vz}\vp = \nabla_{\vz}\sigma(\vz) = \diag(\vp) - \vp \vp^\top
\]
Therefore, the gradient and Hessian of $f(\vz)$ are given by
\begin{align}
\nabla f(\vz) &= \sum_{i=1}^K y^{\text{smooth}}_i \nabla_{\vz} \xi(p_i) = \sum_{i=1}^K y^{\text{smooth}}_i \underbrace{ \xi'(p_i) p_i}_{\eta(p_i)} (\vone_i - \vp) \label{eq:grad-z}\\
\nabla^2 f(\vz) & = \nabla (\nabla f(\vz)) = \sum_{i=1}^K y^{\text{smooth}}_i \parans{ \eta'(p_i) p_i \underbrace{ (\vone_i - \vp) (\vone_i - \vp)^\top }_{\vzero} - \eta(p_i)\underbrace{ \parans{ \diag(\vp) - \vp \vp^\top } }_{\succeq \vzero}
}\nonumber
\end{align}

Thus, $\nabla^2 f(\vz)$ is PSD when $\eta(p_i) \le 0$ and $\eta'(p_i) \ge 0$ for all $i$, i.e., 
\e
\xi'(p_i) \le 0, \quad \xi''(p_i)p_i + \xi'(p_i) \ge 0.
\ee

Now we consider the following cases:
\begin{itemize}
    \item {\bf CE loss} with $\vy^{\text{smooth}} = \ve_k$ and $\xi(t) = -\log(t) $. In this case, $\xi'(p_i) = -\frac{1}{p_i}$ and $\eta(p_i) = \xi'(p_i) p_i = -1 $, and thus
    \begin{align*}
\nabla^2 f(\vz) =     \diag(\vp) - \vp \vp^\top \succeq \vzero,
\end{align*}
where the inequality can be obtained by the Gershgorin circle theorem.
 \item {\bf Label smoothing} with $\vy^{\text{smooth}} = (1-\alpha)\ve_k + \frac{\alpha}{K}\vone$ and $\xi(t) = -\log(t) $. In this case, $\xi'(p_i) = -\frac{1}{p_i}$ and $\eta(p_i) = \xi'(p_i) p_i = -1 $, and thus
    \begin{align*}
\nabla^2 f(\vz)  = \sum_{i=1}^K y^{\text{smooth}}_i \parans{    \diag(\vp) - \vp \vp^\top} = \diag(\vp) - \vp \vp^\top \succeq \vzero
\end{align*}
since $\sum_{i=1}^K y^{\text{smooth}}_i = 1$.

\item {\bf Focal loss} with $\vy^{\text{smooth}} = \ve_k$ and $\xi(t) = -(1-t)^\beta \log(t) $. In this case, 
\begin{align*}
\xi'(p_i) & = \beta(1-p_i)^{\beta-1}\log(p_i)-\frac{(1-p_i)^\beta}{p_i}, \\
\eta(p_i) & = \xi'(p_i) p_i = \beta p_i (1-p_i)^{\beta-1}\log(p_i) - (1-p_i)^\beta \le 0,  \ \forall \ \beta\ge 0,  p_i \in [0,1],\\
\eta'(p_i) & = \beta (1-p_i)^{\beta-1}\log(p_i) - \beta (\beta - 1)  p_i (1 - p_i)^{\beta -2}\log(p_i) + \beta (1-p_i)^{\beta -1} + \beta (1-p_i)^{\beta-1}\\
& = \beta (1-p_i)^{\beta-2} \parans{(1-\beta p_i) \log(p_i) + 2(1-p_i)}\\
& \ge \beta (1-p_i)^{\beta-2} \parans{ \log(p_i) + 2(1-p_i)}.
\end{align*}
Thus, $\eta'(p_i) \ge 0$ whenever $0.21\le p_i\le 1$.
The Hessian becomes
\begin{align*}
\nabla^2 f(\vz) & =  \eta'(p_k) p_k \underbrace{(\ve_k - \vp) (\ve_k - \vp)^\top}_{\succeq \vzero} - \eta(p_k) \underbrace{\parans{ \diag(\vp) - \vp \vp^\top }}_{\succeq \vzero} 
\end{align*}
which is PSD when $0.21\le p_k\le 1$.

\end{itemize}

\end{document}